\documentclass{article}

 \usepackage[preprint]{neurips_2026}

\usepackage[utf8]{inputenc} 
\usepackage[T1]{fontenc}    
\usepackage{hyperref}       
\usepackage{url}            
\usepackage{booktabs}       
\usepackage{amsfonts}       
\usepackage{nicefrac}       
\usepackage{microtype}      
\usepackage{xcolor}         

\title{Pure Exploration Beyond Reward Feedback: \\ The Role of Post-Action Context}

\author{%
    Mohammad Shahverdikondori \thanks{Authors contributed equally.}   \\
    EPFL \\
    \texttt{mohammad.shahverdikondori@epfl.ch} \\
    \And
    Amir Mohammad Abouei \footnotemark[1]\\
    EPFL \\
    \texttt{amir.abouei@epfl.ch} \\
    \And
    Alireza Rezaeimoghadam \\
    Sharif University of Technology\\
    \texttt{alireza.moghadam@sharif.edu} \\
    \And
    Negar Kiyavash \\
    EPFL \\
    \texttt{negar.kiyavash@epfl.ch} \\
}

\usepackage[utf8]{inputenc} 
\usepackage[T1]{fontenc}    
\usepackage{hyperref}       
\usepackage{url}            
\usepackage{booktabs}       
\usepackage{amsfonts}       
\usepackage{nicefrac}       
\usepackage{microtype}      
\usepackage{xcolor}   

\usepackage{microtype}
\usepackage{graphicx}
\usepackage{subcaption}

\usepackage{booktabs}
\usepackage{enumitem}

\usepackage{mathtools}
\DeclarePairedDelimiter\norm{\lVert}{\rVert}

\usepackage{algorithm}
\usepackage{algorithmic}

\usepackage{tikz}
\usetikzlibrary{positioning}
\usetikzlibrary{shapes, arrows, arrows.meta, positioning}

\usepackage{amsmath}
\usepackage{amsthm}
\usepackage{amssymb}
\usepackage{mathtools}

\usepackage{thmtools,thm-restate}
\usepackage{multirow}

\usepackage{wrapfig}

\usepackage{mwe}
\usepackage{tikz} 
\usetikzlibrary{arrows.meta,shapes} 
\usepackage{pgf}
\usepackage{bbm}
\usepackage{extarrows}

\newcommand{\M}[0]{\mathcal{M}}
\newcommand{\A}[0]{\mathcal{A}}
\newcommand{\V}[0]{\mathcal{V}}

\newcommand{\ocal}[0]{\mathcal{O}}
\newcommand{\E}[0]{\mathcal{E}}
\newcommand{\I}[0]{\mathcal{I}}
\newcommand{\Is}[0]{\mathcal{I}^s}
\newcommand{\pr}[0]{P}
\newcommand{\C}[0]{\mathcal{C}}
\newcommand{\ch}[0]{\text{conv}}
\newcommand{\N}[0]{\mathcal{N}}
\newcommand{\hcal}[0]{\mathcal{H}}
\newcommand{\lcal}[0]{\mathcal{L}}

\newcommand{\ub}{\mathbf{u}}
\newcommand{\wb}{\mathbf{w}}
\newcommand{\pb}{\mathbf{p}}
\newcommand{\Nb}{\mathbf{N}}

\newcommand{\pol}[0]{\pb}

\newcommand{\bmu}[0]{\boldsymbol{\mu}}
\newcommand{\muh}[0]{\hat{\mu}}
\newcommand{\bmuh}[0]{\hat{\bmu}}
\newcommand{\muht}[0]{\hat{\bmu}(t)}
\newcommand{\lamht}[0]{\hat{\Lambda}(t)}

\newcommand{\delhit}[1]{\hat{\Delta}_{#1}(t)}

\newcommand{\logdel}[0]{\ln(1 / \delta)}  
\newcommand{\logDel}[0]{\ln(\frac{1} {\delta})}  
\newcommand{\sumij}[0]{\sum_{\substack{i \in [n] \\ j \in [k]}}}
\newcommand{\cdelt}[0]{\hat{c}_t(\delta)}
\newcommand{\fiwmu}[0]{f_i(\mathbf{w}, \bmu, \A)}
\newcommand{\istarmu}[0]{i^*(\bmu)}
\newcommand{\istarmut}[0]{i^*(\muht)}
\newcommand{\wstar}[1]{\mathbf{w}^*(#1, \A)}
\newcommand{\wzstar}[1]{\mathbf{w}_z^*(#1, \A)}
\newcommand{\wepsstar}[1]{\mathbf{w}^*_{#1}(\bmu, \A)}
\newcommand{\infnorm}[1]{\left\Vert #1 \right\Vert _{\infty}}
\newcommand{\lsnorm}[1]{\left\Vert #1 \right\Vert _{2}}

\newcommand{\dls}[2]{dist_{L^2}\left(#1, #2\right)}
\newcommand{\eteps}[0]{\E_T(\epsilon)}
\newcommand{\etepspr}[0]{\E'_T(\epsilon)}
\newcommand{\cstareps}[0]{C^*_{\epsilon}(\bmu, \A)}
\newcommand{\taudel}[0]{\tau_{\delta}}
\newcommand{\amin}[0]{a_{\min}}
\newcommand{\expec}[1]{\mathbb{E}[#1]}

\newcommand{\numaction}[0]{N^{X}_{i}(t)}
\newcommand{\numcontext}[0]{N^{Z}_{j}(t)}
\newcommand{\numac}[0]{N_{j, i}(t)}

\newcommand{\argmin}{\operatornamewithlimits{argmin}}
\newcommand{\argmax}{\operatornamewithlimits{argmax}}

\newtheorem{theorem}{Theorem}[section]
\newtheorem{proposition}[theorem]{Proposition}
\newtheorem{lemma}[theorem]{Lemma}

\newtheorem{definition}[theorem]{Definition}

\usepackage[textsize=tiny]{todonotes}

\begin{document}

\maketitle

\begin{abstract}
We introduce the problem of best arm identification (BAI) with post-action context, a new BAI problem in a stochastic multi-armed bandit environment and the fixed-confidence setting. The problem addresses the scenarios in which the learner receives a \emph{post-action context} in addition to the reward after playing each action. This post-action context provides additional information that can significantly facilitate the decision process. We analyze two different types of the post-action context: (i) \textit{separator}, where the reward depends solely on the context, and (ii) \textit{non-separator}, where the reward depends on both the action and the context. 
For both cases, we derive instance-dependent lower bounds on the sample complexity and propose algorithms that asymptotically achieve the optimal sample complexity.
For the separator setting, we propose a novel sampling rule called \textit{G-tracking}, which uses the geometry of the context space to directly track the contexts rather than the actions.
For the non-separator setting, we do so by demonstrating that the Track-and-Stop algorithm can be extended to this setting. 
Moreover, in both settings, we theoretically and empirically show that algorithms that ignore the post-action context are sub-optimal.
Finally, our empirical results showcase the advantage of our approaches compared to the state of the art.
\end{abstract}

\section{Introduction}
Multi-armed bandit (MAB) refers to a class of sequential decision-making problems, where a learner selects actions (arms) in order to maximize a reward. 
MAB has widespread applications in various domains, such as clinical trials \citep{william_r__thompson_1933}, dynamic pricing \citep{kleinberg2003value, besbes2009dynamic}, recommender systems \citep{li2010contextual}, and resource allocation \citep{gai2012combinatorial}. Depending on the learner’s goal and constraints, different objectives may be pursued. For example, if the learner's goal is to minimize cumulative regret, they must balance the exploration-exploitation trade-off \citep{auer2002using, garivier2011kl}. Alternatively, in pure exploration problems, the learner aims to identify a desired object rather than maximize cumulative reward during learning. In the \emph{Best Arm Identification (BAI)} setting, a canonical pure exploration problem, the learner seeks the arm with the highest expected reward and must minimize the sample complexity, i.e., the number of rounds needed to identify this arm \citep{lb-tsitsiklis-mannor2004sample, SR-audibert2010best, track-stop-garivier2016optimal}. We focus on BAI with \textit{fixed-confidence}, where the error probability $\delta$ for identifying the best arm is fixed, and the objective is to minimize the sample complexity \citep{track-stop-garivier2016optimal, SR-audibert2010best, kaufmann2020contributions, confidence-jamieson2014best, lb-tsitsiklis-mannor2004sample}.


While the classic MAB model is suitable for a broad range of applications, additional side information about the environment can lead to more efficient algorithms. In the bandit literature, various types of side information have been considered. For example, \emph{causal bandits} \citep{lattimore2016causal, causal-cucb-lu2020regret} or \emph{linear bandits} \citep{auer2002using, abbasi2011improved} impose specific structures on the actions, and \emph{contextual bandits} \citep{tewari2017ads, langford2007epoch, russac2021b} allow the learner to observe a context before choosing an action. In this work, we study a form of side information known as \emph{mediator feedback} or \emph{post-action context}, which has been considered in the setting of cumulative regret minimization by several recent studies \citep{mann2019learning, metelli2021policy, causal-benign-bilodeau2022adaptively, esposito2023delayed, mediator-cum-eldowa2024information, causal-pareto-liu2024causal}. Specifically, after choosing an action in each round, the learner receives intermediate feedback from the environment along with the reward. This post-action context can substantially accelerate the process of identifying the best arm. We consider this new problem in the fixed-confidence setting, aiming to show how leveraging post-action context can reduce the sample complexity for different environments.

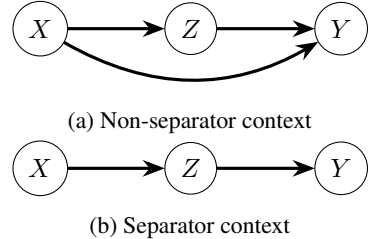
\begin{wrapfigure}{r}{0.4\textwidth}
        \begin{subfigure}[b]{0.45\textwidth}
            \centering
            \begin{tikzpicture}           
                \tikzset{line width=2pt, outer sep=0pt,
                ell/.style={draw,fill=white, inner sep=2pt,
                line width=2pt},
                };
                    \node[name=X, shape=circle, draw] at (-3,0){$X$};
                    \node[name=Z, shape=circle, draw] at (-1,0){$Z$};
                    \node[name=Y, shape=circle, draw] at (1,0){$Y$};
                
                \begin{scope}[>={Stealth[black]},
                              every edge/.style={draw=black,very thick}]
                    \path[->] (X) edge (Z);
                    \path[->] (X) edge[bend left=40] (Y);
                    \path[->] (Z) edge (Y);
                \end{scope}
            \end{tikzpicture}
            \caption{Non-separator context}
            \label{fig: nonsep}
        \end{subfigure}
        \begin{subfigure}[b]{0.45\textwidth}
            \centering
            \begin{tikzpicture}
                \tikzset{line width=2pt, outer sep=0pt,
                ell/.style={draw,fill=white, inner sep=2pt,
                line width=2pt},
                };
                    \node[name=X, shape=circle, draw] at (-3,0){$X$};
                    \node[name=Z, shape=circle, draw] at (-1,0){$Z$};
                    \node[name=Y, shape=circle, draw] at (1,0){$Y$};
                
                \begin{scope}[>={Stealth[black]},
                              every edge/.style={draw=black,very thick}]
                    \path[->] (X) edge (Z);
                    \path[->] (Z) edge (Y);
                \end{scope}
            \end{tikzpicture}
            \caption{Separator context}
            \label{fig: sep}
        \end{subfigure}
        \caption{Two possible structures for the post-action context.}
        \label{fig: causal_graph}
\end{wrapfigure}

Figure~\ref{fig: causal_graph} illustrates two structures pertaining to post-action context, which may also be interpreted as causal graphs. In both structures, variables $X$, $Z$, and $Y$ represent the action, the post-action context, and the reward, respectively. In Figure~\ref{fig: sep}, once $Z$ is observed, the distribution of $Y$ depends only on $Z$, making $X$ and $Y$ conditionally independent given $Z$. We call this the \textit{separator} context. By contrast, in Figure~\ref{fig: nonsep}, the context $Z$ is informative but not a sufficient statistic of $X$ for determining $Y$. In other words, the distribution of $Y$ depends on both $X$ and $Z$. We refer to this as the \textit{non-separator} context. Although the separator context can be viewed as a special case of the non-separator, an algorithm optimized for the latter may not be optimal for the former. This is because a separator context provides the important side information that $X$ and $Y$ are conditionally independent given $Z$. We analyze both settings in this work.

\textbf{Motivating Example.}  
Consider a healthcare scenario for diabetes management where $X$ denotes a chosen treatment regimen (e.g., medication or dosage), $Z$ is an intermediate biomarker such as short-term blood sugar levels, and $Y$ is the clinical target of improvement (e.g., glucose stability and range). In each round, a clinician assigns a treatment $X$ to a patient, measures $Z$ after a short interval, and later observes $Y$. Tracking this intermediate observation may provide additional information about the alternative treatments. The distribution $\pr( Z \vert X) $ can often be reliably estimated from electronic health records or prior clinical studies, making it valuable prior knowledge. By incorporating $Z$, clinicians can adapt treatment strategies more efficiently, reducing the number of trials needed to identify an optimal regimen. This scenario may fall into either non-separator or separator settings. In both cases, considering $Z$ enhances learning compared to relying solely on target $Y$.

\textbf{Comparison to closely related work.}
\cite{causal-cucb-lu2020regret, mediator-cum-eldowa2024information} and \cite{lattimore2016causal} study similar settings with separating contexts, focusing on regret minimization and simple regret optimization, respectively.
Several works study best-arm identification (BAI) in contextual bandits \citep{context1-li2022instance, context2-kato2021role}, including the proportional mode of \cite{russac2021b}. Although these problems are related to ours, there is a fundamental distinction: in contextual bandits, the context is generated before the action is selected, and the action is chosen based on the observed context. \cite{russac2021b} also considers more general scenarios, including active, agnostic, and oblivious modes. For example, in the agnostic mode, the context is independent of the action but is observed only after an action is chosen. However, none of these works captures the structure of our problem, where the context is generated as a consequence of the chosen action.
Finally, \cite{carlsson2024pure} considers a similar problem; however, in their setting, the learner can choose $Z$ in each round, whereas in ours the learner cannot directly choose $Z$. For a more detailed discussion of related work, see Appendix~\ref{apd: related-work}.

The presence of post-action context can fundamentally alter the optimal strategy for selecting arms. In classic BAI problems, the optimal algorithms proposed by \cite{track-stop-garivier2016optimal} focus more time on pulling the arms with the higher observed rewards, which is natural when no additional side information is available. However, in the presence of post-action context, collecting more samples from a certain context may expedite the identification of the best arm. In such scenarios, choosing a suboptimal arm that increases the probability of observing that context is more effective than selecting the arm with the maximum reward. 
We show the suboptimality of the methods that ignore the post-action context both theoretically (Section \ref{sec: non-optimality}) and empirically (Section \ref{sec: experiment}).

\textbf{Contributions.}
We introduce \textit{best-arm identification with post-action context}, a fixed-confidence pure exploration problem in which each action generates both a reward and an informative post-action context. We study two structural regimes: a \emph{separator} setting, where the reward depends on the action only through the post-action context, and a \emph{non-separator} setting, where the reward may depend on both the action and the context. For both regimes, we derive instance-dependent lower bounds and design algorithms that achieve the corresponding asymptotically optimal sample complexity. In the separator setting, the lower bound depends only on the frequencies with which contexts are observed; this motivates our new sampling rule, \emph{G-tracking}, which tracks optimal context frequencies directly rather than tracking arm pull frequencies. In the non-separator setting, we show that the classical Track-and-Stop framework can be extended to exploit post-action contexts optimally. Finally, we show that algorithms that ignore the post-action context can be strictly sub-optimal and characterize this suboptimality.


\section{Preliminaries and Problem Setup} \label{sec: preli}

\textbf{Notations.} The set $[n]$ denotes $\{1, 2, \dots, n\}$ and $\Delta^{n-1}$ represents the $(n-1)$-dimensional standard simplex, defined as $\{w \in \mathbb{R}^n \mid w_i \geq 0 \wedge \sum_{i = 1}^{n} w_i = 1\}$. The notation $d(P, Q)$ denotes the Kullback–Leibler (KL) divergence between two probability measures $P$ and $Q$. Additionally, $d_B(\delta,1-\delta)$ refers to the KL divergence between two Bernoulli random variables with parameters $\delta$ and $1 - \delta$. The convex hull of \(n\) vectors \(\A_i \in \mathbb{R}^d\) is denoted by $\ch(\{\A_1, \A_2, \dots, \A_n\}) = \left\{ \sum_{i = 1 }^{n} \lambda_i \A_i \mid \lambda_i \geq 0 \ \wedge \sum_{i = 1}^{n} \lambda_i = 1 \right\}$.

\textbf{Best arm identification with post-action context problem.} In this problem, a player interacts with a multi-armed bandit environment $\V$ with $n$ arms. In each round $t$, the player selects an action $X = i_t \in [n]$ and observes a pair $(z_t, y_t)$, where $z_t$ is the value of the post-action context variable $Z$, and $y_t$ is the realization of the reward variable $Y$.

We assume the context variable $Z$ is discrete and takes values in $[k]$, depending only on the arm pulled. The context probability matrix $\A = [\A_1 \vert \A_2 \vert \ldots \vert \A_n] \in \mathbb{R}^{k \times n}$, where $\A_{j,i} = \pr(Z = j \mid X = i)$, encodes the probability of contexts given each action. We further assume that the reward distribution $\pr(Y \mid X=i, Z=j)$ follows a Gaussian distribution with unit variance. The matrix $\bmu = [\bmu_1 \vert \bmu_2 \vert \ldots \vert \bmu_n] \in \mathbb{R}^{k \times n}$ represents the mean values of the reward distributions, such that $\mu_{j,i} = \mathbb{E}(Y | X=i, Z=j)$. 

In the case of the separator context, the reward depends only on $Z$, that is, $\pr (Y \mid X = i, Z = j) = \pr (Y \mid Z = j)$,
for all $i \in [n]$ and $j \in [k]$. This implies that all columns of $\bmu$ are identical. We denote the joint distribution of the context and reward for a given action $i$ by $P^{\bmu. \A}_{i}= \pr(Z, Y \mid X=i)$. The expected reward for a given action $i$ is computed as 
\begin{align*}
    & \mathbb{E}[Y \mid X=i] = \sum_{j=1}^{k} \pr(Z = j \vert X = i)\mathbb{E}[Y \vert Z = j, X = i] = \A_i^\top \bmu_i.
\end{align*}
The best arm is the arm with the maximum expected reward, i.e., $i^*(\bmu, \A)=\argmax_{i \in [n]} \A_i^\top \bmu_i$, where the mean matrix $\bmu$ and the context probability matrix $\A$ characterize the instance. In this work, we assume that $\A$ is known to the learner and that the best arm is unique. Let $\I(\A) \subset \mathbb{R}^{k \times n}$ denote the set of $\bmu$ matrices that imply a unique best arm for a given $\A$. For simplicity, as $\A$ is known, we often use $\I, \istarmu, $ and $P^{\mu}$ instead of $\I(\A), i^*(\bmu, \A),$ and $P^{\mu, \A}$, respectively.

We consider the best arm identification problem in the \textit{fixed-confidence} setting, where the player aims to identify the best arm $\istarmu$ for any $\bmu \in \I$ with a pre-specified error probability $\delta \in (0,1)$. A sequential learning algorithm for this setting consists of three main components: (i) A \emph{sampling rule} (deterministic or stochastic) that determines which arm to pull in round $t$ based on the history up to round $t-1$, denoted by $\mathcal{H}_{t-1} = (i_1, (z_1, y_1), i_2, (z_2, y_2), \ldots, i_{t-1}, (z_{t-1}, y_{t-1}))$, (ii) A \emph{stopping rule} $\tau$, based on $\mathcal{H}_t$, determines when to stop the process, and (iii) A \emph{decision rule} $\hat{  i}_{\tau}$, which recommends the best action.

For any $\delta \in (0, 1)$, an algorithm is \textit{$\delta$-correct} if for any $\A$ and $\bmu \in \I(\A)$ it satisfies $\pr_{\bmu, \A}(\tau_{\delta} < \infty, \hat{i}_{\tau} \neq i^*(\bmu)) \leq \delta$. The goal of the learner is to find a $\delta$-correct algorithm that minimizes the expected number of samples $\mathbb{E}_{\bmu, \A}[\tau_{\delta}]$.

The sub-optimality gap of arm $i$ is defined as
\begin{align*}
    &\Delta_i \triangleq \mathbb{E}[Y \mid X=\istarmu] - \mathbb{E}[Y \mid X=i] = \A_{\istarmu}^\top \bmu_{\istarmu} - \A_i^\top \bmu_i.
\end{align*}
During the learning process, let $\muht$ denote the matrix containing the empirical estimates of all entries of $\bmu$ based on the samples collected up to round $t$ and  define
\begin{align*}
    \delhit{i} \triangleq \A_{\istarmut}^\top \bmuh_{\istarmut}(t) - \A_i^\top \bmuh_i(t),
\end{align*}
which is the gap of arm $i$ relative to the arm with the best empirical mean up to round $t$ estimated using $\muht$. Finally, $\numaction$, $\numcontext$, and $\numac$ represent the number of times action $i$, context $j$, and joint action-context pair $(i,j)$ have been observed up to round $t$, respectively.

\paragraph{Summary of Assumptions.} We assume that the set of contexts is finite. The context probability matrix $\A$ is known, and the best arm is unique. In Appendix \ref{apd: unknown context}, we provide a discussion on the complexity of the problem when $\A$ is unknown. The reward distribution for each action-context pair follows a Gaussian distribution with unit variance; however, our results can be easily generalized to the one-parameter exponential family. Additionally, we assume positivity, meaning that for each action, each context occurs with positive probability, i.e., $\amin = \min_{i,j} \A_{i,j} > 0$.

\subsection{General Lower Bound} \label{subsec: general-lower}

For the classical best arm identification problem, \cite{track-stop-garivier2016optimal} established a lower bound on the expected sample complexity of any $\delta$-correct algorithm. Extending this lower bound to our problem setting is straightforward, as we show in Proposition \ref{th: general_lower1}. Before stating the proposition, let us define Alt$(\bmu, \A)$ as the set of alternative parameter matrices $\bmu' \in \I$ such that $\istarmu \neq i^*(\bmu')$. 


\begin{proposition}\label{th: general_lower1}
     For any bandit environment with parameters $\A$ and $\mu \in \I(\A)$ and any $\delta$-correct algorithm,  $\mathbb{E}_{\bmu, \A}[\tau_{\delta}] \geq T^*(\bmu, \A)d_B(\delta, 1- \delta)$, where $T^*(\bmu, \A)$ is 
        \begin{align} \label{eq: general_lower1}
            T^*(\bmu, \A)^{-1} \triangleq \sup_{\mathbf{w} \in \Delta^{n-1}} \inf_{\bmu' \in \text{Alt}(\bmu, \A)} \sum_{i \in [n]} w_i d(P^{\bmu}_{i}, P^{\bmu'}_{i}).
        \end{align}
\end{proposition}
This lower bound indicates that the optimal sampling strategy is to play each action $i$ in proportions to $w^*_i$, a solution of Equation~\eqref{eq: general_lower1}. Note that the bound is similar to a previous result \citep{track-stop-garivier2016optimal}, with the key distinction that it depends on the context variable through $P^{\bmu}_{i}=\pr(Z, Y \mid X=i)$. 

We derive lower bounds for the BAI problem under both separator and non-separator post-action context settings by explicitly characterizing Equation~\eqref{eq: general_lower1}. The key difference lies in the structure of the alternative set $\text{Alt}(\bmu, \A)$, which varies between the two settings. In the separator case, all $\bmu_i$ vectors are identical, imposing additional constraints on $\text{Alt}(\bmu, \A)$ and yielding a strictly smaller lower bound. This tighter bound motivates the design of a tailored sampling algorithm to achieve optimal sample complexity in the separator setting. The analysis for the non-separator case is deferred to Appendix~\ref{sec: non-sep}, as the derivation closely follows the separator case and known sampling rules such as D-tracking are sufficient to achieve optimality.

\section{Separator Context} \label{sec: sep}

This section presents the results for the \textit{bandit with separator post-action context}, where the reward distribution depends solely on the value of the context, i.e., $\pr(Y \mid X=i, Z=j) = \pr(Y \mid Z=j)$, as illustrated in Figure \ref{fig: sep}. This condition implies that all columns of the matrix $\bmu$ are identical. For ease of notation, in this section, we represent $\bmu$ as a vector in $\mathbb{R}^k$.

\subsection{Lower Bound}
    The following theorem provides a lower bound on the expected sample complexity of any $\delta$-correct algorithm in the separator setting. 
    This lower bound is derived by explicitly solving the minimization problem defined in \eqref{eq: general_lower1}.
    
    \begin{restatable}[Separator Lower Bound]{theorem}{sepLowerBound} \label{thm : sep lower} 
            Let $\delta \in (0, 1)$. Consider a bandit instance with a separator context and Gaussian reward distribution with unit variance, parameterized by the matrix $\A$ and the vector $\bmu$. Then, any $\delta$-correct algorithm with stopping time $\taudel$ satisfies $\mathbb{E}[\tau_{\delta}] \geq T^*(\bmu, \A)  d_B(\delta, 1-\delta)$, where
                    \begin{equation} \label{eq: sep LB1}
                        T_{S}^*(\bmu, \A)^{-1} = \sup_{\mathbf{w} \in \Delta^{n-1}} \min_{i \neq i^*(\bmu)} \frac{\Delta_{i}^2}{2 \sum_{j \in [k]} \frac{(\A_{j, \istarmu} - \A_{j, i})^2}{\sum_{l \in [n]} w_l \A_{j, l}}}.
                    \end{equation} 
    \end{restatable}
    The values $w_{z,j} \triangleq \sum_{l \in [n]} w_l \A_{j, l}$ in Equation \eqref{eq: sep LB1} represent the expected frequency of observing context $Z = j$. Note that the optimization problem in \eqref{eq: sep LB1} depends only on $w_{z,j}$s. Consequently, defining the vector $\wb_z = (w_{z,1}, w_{z,2}, \dots, w_{z,k})$, Equation \eqref{eq: sep LB1} can be expressed as
    \begin{equation} \label{eq: sep LB2}
        T_{S}^*(\bmu, \A)^{-1} = \sup_{\mathbf{w}_z \in \ch(\A)} \min_{i \neq i^*(\bmu)} \frac{\Delta_{i}^2}{2 \sum_{j \in [k]} \frac{(\A_{j, \istarmu} - \A_{j, i})^2}{w_{z,j}}},
    \end{equation}
    where $\ch(\A)$ is the convex hull of the points $\A_1, \A_2, \ldots, \A_n$, and ${\mathbf{w}_z \in \ch(\A)}$ because $\wb_z$ is a convex combination of these points. The optimal solution to this new formulation represents the expected frequency of observing context values generated by any optimal algorithm.

    It can be shown that the objective functions in both Equations~\eqref{eq: sep LB1} and \eqref{eq: sep LB2} are concave. Since the domains $\Delta^{n-1}$ and $\ch(\A)$ are convex, both optimization problems can be tackled using convex optimization techniques. For more details on solving such optimization problems in BAI settings, see \citep{frank-wolf-wang2021fast,menard2019gradient}.

    \begin{algorithm}[t]
        \caption{Separator Track and Stop (STS)}
        \label{algo: sep}
        \begin{algorithmic}[1] 
            \STATE \textbf{Input:} Context probability matrix $\A$.
            \STATE \textbf{Initialization:} Pull arms until collecting at least one sample from $P(Y|Z=j)$ for each $j \in [k]$, then update $t$ accordingly.
            \WHILE{$\lamht \leq \cdelt$}
                \STATE Find policy $\pol(t + 1) \in \ch(\A)$ based on G-tracking rule.
                \STATE Play according to a distribution $\pi$ such that $\A \pi = \pol(t + 1)$. 
                \STATE Update $\muht$ and $t \gets t + 1.$
            \ENDWHILE 
            \STATE \textbf{Output:} $i^*(\muht)$. 
        \end{algorithmic}
    \end{algorithm}

\subsection{Learning Algorithm}

    \textbf{Stopping Rule.} The stopping rule operates independently of the sampling rule and determines, at each step, whether sufficient information has been gathered to terminate the algorithm or if further sampling is necessary. A widely used approach for designing stopping rules in one-parameter exponential family bandits, including Gaussian bandits, is based on Generalized Likelihood Ratio (GLR) tests \citep{track-stop-garivier2016optimal, kaufmann2021mixture}. We describe how we adapt these tests to our setting in Appendix~\ref{apd: GLR}. The GLR statistic in this setting is defined as
    \begin{equation} \label{eq: glr-def-sep}
            \lamht \triangleq  \inf_{\bmu' \in \text{Alt}(\muht, \A)} \sum_{j \in [k]} \numcontext \frac{(\hat{\bmu}_{j}(t) - \bmu'_{j})^2}{2},
    \end{equation}
    provided that $\muht \in \I(\A)$ and otherwise $\lamht$ is set to zero. \(\lamht\) reflects the current confidence in the best arm identified based on the observations so far. A higher value of \(\lamht\) indicates greater confidence. This definition can be simplified to
    \begin{align*}
             \lamht = \min_{i \neq \istarmut} \frac{\hat{\Delta}_{i}^2}{2 \sum_{j \in [k]} \frac{(\A_{j, \istarmut} - \A_{j, i})^2}{\numcontext}}.
    \end{align*}
    Appendix \ref{apd: GLR} and \ref{apd: glr-sep} provide a proof of how to compute the GLR statistic and a justification of the above simplification.

    To decide whether to terminate or continue playing, \(\lamht\) is compared to a sequential threshold \(\cdelt\), which determines if the confidence is sufficiently high. To design the sequential thresholds, we apply the result by \cite{kaufmann2021mixture} for adaptive sequential testing, which suggests
    \begin{equation} \label{eq: sep threshold}
           \cdelt = 2 \sum_{j \in [k]} \ln\left( 4 + \ln\left( \numcontext \right)\right) + k C^{g} \left( \frac{\ln\left(\frac{1}{\delta}\right)}{k} \right),
    \end{equation}
    where $C^{g}$ is the same function introduced in Equation \eqref{eq: non-sep threshold}. The stopping time $\tau_{\delta}$ is then defined as
    \begin{equation} \label{eq: sep stop rule}
         \tau_{\delta} \triangleq \inf \{ t \in \mathbb{N} \mid \lamht > \cdelt \}.
    \end{equation}
     At the stopping time $\tau_{\delta}$, the final suggestion $\hat{i}_{\tau}$ is equal to the unique estimated best arm $i^*(\bmuh(\tau_{\delta}))$. The following lemma establishes the correctness of this stopping rule.
     
    \begin{restatable}{lemma}{sepCorrectness} \label{lem: sep correctness}
             Consider a bandit instance with a separator context and Gaussian reward distribution with unit variance, parameterized by the matrix $\A$ and the vector $\bmu$. Any algorithm with the stopping rule of \eqref{eq: sep stop rule} is $\delta$-correct, that is, $\pr_{\bmu, \A}(\tau_{\delta} < \infty, \hat{i}_{\tau} \neq i^*(\bmu)) \leq \delta.$
    \end{restatable}

\textbf{Sampling Rule.} As mentioned earlier, the lower bound and stopping rule depend solely on how often each context value is observed, regardless of the actions that result in these contexts. Motivated by this observation, we introduce our new sampling method, called \textit{Geometric Tracking (G-tracking)}, which directly tracks the optimal frequency of observed contexts.



Note that this setting can be viewed as an active off-policy learning problem in $k$-armed bandit, where each arm corresponds to a context value. However, the agent cannot directly select a context to play. Instead, they have access to $n$ distinct policies (the original $n$ arms in our problem), where choosing policy $i$, probabilistically selects a context according to the known distribution $\A_i$. Each vector $\A_i$ lies in the $(k-1)$-dimensional simplex and encodes the probability of selecting a context by playing policy $i$.
Under this interpretation, the policy space is the convex hull of the points $\A_1, \A_2, \ldots, \A_n$, denoted by $\ch(\A)$. In each round $t$, choosing a policy $\pol(t) \in \ch(\A)$ is equivalent to fixing a distribution on context $Z$. Notably, there exists (although not necessarily unique) a distribution over arms $X$ that is compatible with $\pol(t)$, and the agent samples an arm at time $t$ according to this distribution.

We now propose a new sampling rule that directly tracks the optimal frequency of observing contexts and asymptotically achieves optimal sample complexity when combined with the stopping rule in Equation \eqref{eq: sep stop rule}. For any vector $\bmu \in \I(\A)$ (i.e., a unique best arm exists), we define the set $\wzstar{\bmu}$ as the set of solutions to the following optimization problem
\begin{align} \label{eq: sep-optimal-weights}
    \argmax_{\mathbf{w}_z \in \ch(\A)} \min_{i \neq i^*(\bmu)} \frac{\Delta_{i}^2}{2 \sum_{j \in [k]} \frac{(\A_{j, \istarmu} - \A_{j, i})^2}{w_{z,j}}},
\end{align} 
which corresponds to the formulation in Equation~\eqref{eq: sep LB2}.
Unlike the non-separator setting, in the separator setting, the optimal weights may not be unique, i.e., $\wzstar{\bmu}$ is not a singleton. The optimization problem \eqref{eq: sep-optimal-weights} is convex for each $\bmu \in \I$ and can be solved using numerical methods such as Frank-Wolfe \citep{pmlr-v28-jaggi13, frank-wolf-wang2021fast}. Although \eqref{eq: sep-optimal-weights} is non-smooth, the objective is the minimum of a finite number of smooth functions. The standard Frank–Wolfe algorithm does not provide theoretical convergence guarantees in this case, but \cite{frank-wolf-wang2021fast} establish such guarantees for optimization problems whose objectives admit this structure.


\begin{figure}
    \centering
    \includegraphics[width=0.65 \linewidth]{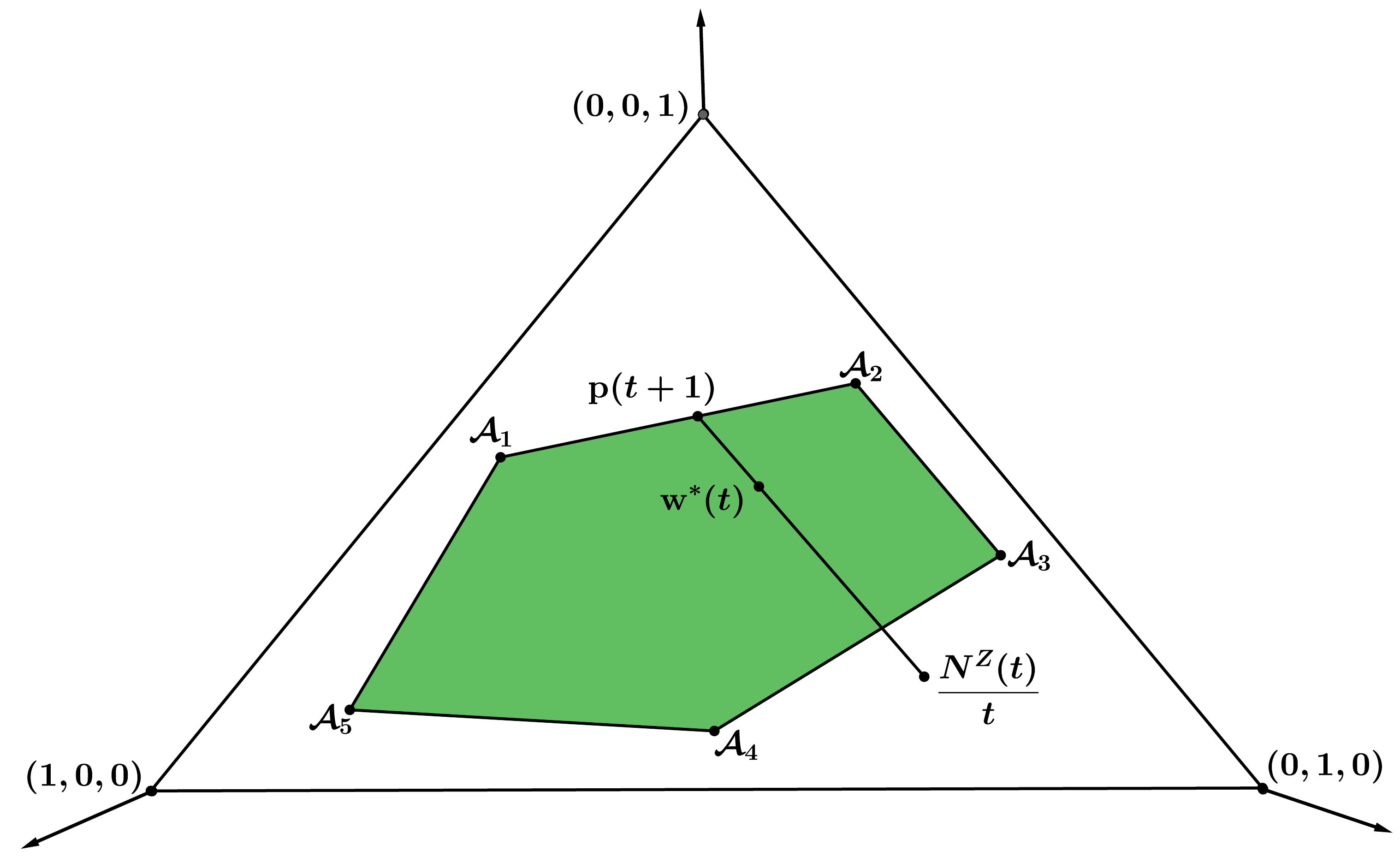}
    \caption{Illustration of G-tracking rule for an instance with $k=3, n=5$. The triangle depicts the two-dimensional simplex, and the green area shows the policy space 
    $\ch(\A)$.}
    \label{fig: sampling}
\end{figure}

At each round $t$, our sampling rule first solves the optimization problem \eqref{eq: sep-optimal-weights} using the current empirical estimate $\muht$ to find a solution $\wb^{*}(t) \in \wzstar{\muht}$. This vector lies within $\ch(\A)$. Recall that vector $\frac{\Nb^{Z}(t)}{t}$ denotes the actual frequency of observed contexts up to round $t$ and it lies in the simplex $\Delta^{k-1}$. Our sampling rule, called G-tracking, determines the next policy $\pol(t + 1) \in \ch(\A)$ by connecting $\frac{\Nb^{Z}(t)}{t}$ to $\wb^{*}(t)$ and extending this line until it intersects the boundary of $\ch(\A)$. The algorithm then pulls arms according to a distribution $\pi \in \Delta^{n-1}$, where $\A \pi = \pol(t + 1)$. A compatible distribution $\pi$ always exists because $\pol(t + 1) \in \ch(\A)$, but it may not be unique. For the formal mathematical formulation of G-tracking, see Appendix \ref{apd: form-g-tracking}.

Figure \ref{fig: sampling} illustrates the sampling rule for an instance with $k=3$ and $n=5$. Note that while the policy points always lie within $\ch(\A)$, $\frac{\Nb^{Z}(t)}{t}$ may lie outside this set (as shown in Figure \ref{fig: sampling}). In such cases, the policy point $\pol(t + 1)$ is the intersection point of the line passing through $\frac{\Nb^{Z}(t)}{t}$ and $\wb^{*}(t)$ with the boundary of policy space, ensuring that $\wb^{*}(t)$ lies between $\pol(t + 1)$ and $\frac{\Nb^{Z}(t)}{t}$.

The main advantage of the G-tracking rule is that it ensures $\wb^{*}(t)$ is a convex combination of the observed context frequencies $\frac{\Nb^{Z}(t)}{t}$ and the next action $\pol(t + 1)$. Thus, if a context $i$ is under-explored $\left( \frac{N^{Z}_{i}(t)}{t} \leq w^{*}_{i} (t) \right)$, then $\pol(t + 1)$ allocates a probability greater than $w^{*}_{i} (t)$ to that context. Conversely, for each over-explored context $i$ $\left( \frac{N^{Z}_i(t)}{t} \geq w^{*}_{i} (t)\right)$, the next policy allocates a probability lower than $w^{*}_{i} (t)$. This ``balancing" property holds for all contexts simultaneously.
Note that the balancing property holds for every point on the segment connecting $\wb^{*}(t)$ to $\pol(t + 1)$. Therefore, our proof remains valid for any algorithm that picks a point on this segment. In our algorithm, we choose the $\pol(t + 1)$ on the boundary of the convex hull to achieve the fastest possible tracking of the optimal weights.

The pseudocode of our algorithm called \textit{Separator Track and Stop (STS)} is presented in Algorithm \ref{algo: sep} which uses G-tracking for sampling with the stopping rule of \eqref{eq: sep stop rule}. The following theorem states that our algorithm achieves the optimal sample complexity.
\begin{restatable}[Separator Upper Bound]{theorem}{sepUpperBound} \label{thm : sep upper} 
       Algorithm \ref{algo: sep} applied to a bandit instance with a separator context and Gaussian reward distribution with unit variance, parameterized by the matrix $\A$ and the vector $\bmu$ attains
            \begin{equation*}
                \limsup_{\delta \rightarrow 0} \frac{\mathbb{E}_{\bmu, \A}[\tau_{\delta}]}{\logdel} \leq T_{S}^*(\bmu, \A),
            \end{equation*}
            where $T_{S}^*(\bmu, \A)$ is defined in Equation \eqref{eq: sep LB1}.
\end{restatable}

\section{Sub-Optimality of Ignoring Post-Action Context} \label{sec: non-optimality}

In this section, we explain why algorithms that ignore the post-action context are suboptimal in the separator setting. The same discussion for the non-separator setting is deferred to Appendix \ref{apx: non-optimal-non-sep}

\subsection{Separator Context} \label{apx: non-optimal-sep}

    In the non-separator setting, the sub-optimality of ignoring $Z$ stems solely from the increased variance in the reward distributions (i.e., the increase in the sub-Gaussian parameter). In contrast, the separator setting involves more substantial structural differences.
    The key distinction is that, in the separator setting, the action and reward are conditionally independent given the post-action context. As a result, any observation with context $Z = j$ is informative about $\mu_j$, regardless of which action produced it. This property implies that algorithms that ignore $Z$ can suffer arbitrarily large sub-optimality in the separator setting, even when the range of reward means is small. In the following, we describe two potential sources of sub-optimality in this case:

    \textbf{1.} In the separator setting with known context probability matrix $\A$, the learner only needs to estimate the means of $k$ different unit-variance Gaussian variables (i.e., the entries of $\bmu$) with high confidence to identify the best arm. In cases where $n \gg k$, ignoring $Z$ forces the learner to estimate the means of $n$ different random variables (sub-Gaussian with variance greater than $1$) with high confidence, which can significantly increase the required sample complexity.

    \textbf{2.} The correlation among the arms in the separator setting implies that, even when the mean values lie within a small range and $n = k$, algorithms that ignore $Z$ may still be highly sub-optimal. Consider the following example with $k = n = 3$ and the parameters: 
    \begin{align*}
    \bmu = \begin{bmatrix}
        0 \\
        0.5 \\
        1
    \end{bmatrix}, \quad ~
        \A = \begin{bmatrix}
            \epsilon & 2 \epsilon & \frac13 \\
            2 \epsilon & \epsilon & \frac13 \\
            1 - 3 \epsilon & 1 - 3 \epsilon & \frac13
        \end{bmatrix},
    \end{align*}
    where $\epsilon$ is a small positive number. In this example, the first arm is optimal, and the gap between arms 1 and 2 is $\Delta_2 = \A_1^{\top} \bmu - \A_2^{\top} \bmu = \frac{\epsilon}{2}$. According to known lower bounds in the BAI literature, any algorithm solving this instance with confidence $\delta$ must take more than $\frac{4}{\epsilon^2} \logDel$ samples \citep{kaufmann2016complexity}. 
    However, by using the post-action context, repeatedly pulling the third arm (which has a large sub-optimality gap and is thus rarely played by algorithms ignoring $Z$) can stop after just $\mathcal{O}(\logDel)$ samples. This follows from a standard concentration argument, since the components of $\bmu$ are separated by a constant gap, $\mathcal{O}(\ln (\frac{1}{\delta}))$ samples are enough to recover their ordering with high probability. This example illustrates that the cost of ignoring $Z$ in the separator setting can be substantial.


\subsection{Reduction to Linear Bandit} \label{sec: non-optimal-linear}

    A more sophisticated approach to solve the separator setting that ignores $Z$ but takes advantage of knowing the context probability matrix $\A$ is to reduce the problem to a linear bandit problem.
    
    In this reduction, each arm $i$ is mapped to $\A_{i} \in \mathbb{R}^{k}$. Then the reward from pulling action $i$ can be written as
    $$
        Y_i = \sum_{j = 1}^{k} \A_{j, i} \bmu_{j} + R + \varepsilon = \langle \A_{i}, \bmu \rangle  + R + \varepsilon,
    $$
    where $\varepsilon$ is a standard Gaussian, and $R$ is a discrete variable over $\{\bmu_j - \A_{i}^{\top} \bmu \mid j \in [k] \}$ such that $\forall j \in [k]: \pr (R  = \bmu_j - \A_{i}^{\top} \bmu) = \A_{j, i}.$
    If we assume the entries of $\bmu$ are bounded from both sides, then similar to the non-separator case, we conclude that $R + \varepsilon$ is a sub-Gaussian noise but with a variance greater than $1$ (which means that the sub-Gaussian parameter is bigger than $1$). Note that a similar reduction also works for the non-separator setting by corresponding a vector in $\mathbb{R}^{n \times k}$ to each action $X = i$, where positions from $ik+1$ to $ik + k$ are equal to $\A_i$ and others are zero. 
    
    In the following analysis, we disregard the increase in variance due to the reduction to linear bandit, as discussed above, and assume $\sigma = 1$. Even with this simplification, we show that by incorporating the post-action context, the sample complexity lower bound is always less than (or equal to) that of linear bandits on the same instance. This implies that the theoretical guarantees for linear bandit do not yield tight bounds in our setting, as we discuss below.

    There is no sample complexity bound for BAI in linear bandits with sub-Gaussian rewards. However, if we approximate the sample complexity bound for the Gaussian case using the result of \cite{linearLB-soare2015sequential}, we obtain
    \begin{align*}
        &T^*_1(\boldsymbol{\mu}, A)^{-1} \approx \sup_{\boldsymbol{w} \in \Delta^{n-1}} \min_{i \neq i^*} \frac{\Delta_i^2}{2 (A_{i^*} - A_i)^\top \left( \sum_j w_j A_j A_j^\top \right)^{-1} (A_{i^*} - A_i)},
    \end{align*}

    while for BAI with post-action context based on Equation \eqref{eq: sep LB1}, we have
    \[
    T^*_2(\boldsymbol{\mu}, A)^{-1} = \sup_{\boldsymbol{w} \in \Delta^{n-1}} \min_{i \neq i^*} \frac{\Delta_i^2}{2 (A_{i^*} - A_i)^\top D^{-1} (A_{i^*} - A_i)},
    \]
    where $D = \sum_j w_j \text{diag}(A_j)$. To show that $T^*_1 \geq T^*_2$, it suffices to prove that
    \[
    D - \sum_j w_j A_j A_j^\top = \sum_j w_j \left( \text{diag}(A_j) - A_j A_j^\top \right)
    \]
    is positive semi-definite. This follows from the Cauchy--Schwarz inequality
    \begin{align*}
        &\forall v \in \mathbb{R}^k: \quad v^\top \text{diag}(A_j) v = \sum_i v_i^2 A_{i,j} = 
        \left( \sum_i v_i^2 A_{i,j} \right) \left( \sum_i A_{i,j} \right) \geq (v^\top A_j)^2 = v^\top A_j A_j^\top v.
    \end{align*}
    Therefore, the linear bandit reduction necessarily increases the sample complexity.
    
    To illustrate the difference numerically, consider the small instance
    \[
    \boldsymbol{\mu} = \begin{bmatrix} 0.2 \\ 0.5 \\ 0.3 \end{bmatrix}, \quad
    A = \begin{bmatrix}
    0.9 & 0.2 & 0.5 & 0.1 \\
    0.05 & 0.4 & 0.4 & 0.8 \\
    0.05 & 0.4 & 0.1 & 0.1
    \end{bmatrix}.
    \]
    In this case, we compute $T^*_2(\boldsymbol{\mu}, A) \approx 165$, while the corresponding linear bandit lower bound is $T^*_1(\boldsymbol{\mu}, A) \approx 1054$, showing a substantial gap. 

\begin{figure}[t!]
        \centering
        \begin{subfigure}[t]{0.45\textwidth}
            \centering
            \includegraphics[width = \textwidth]{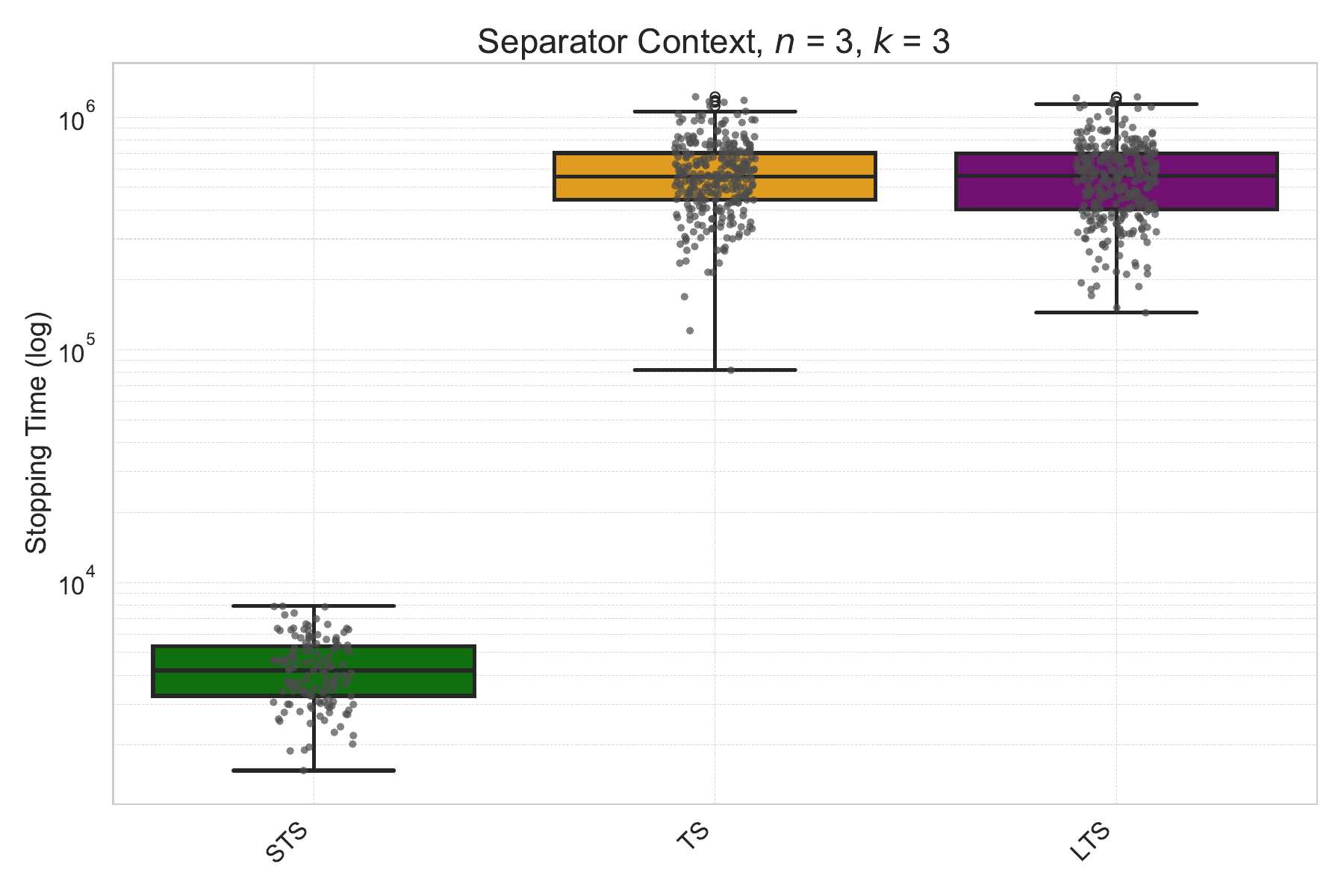}
            \caption{Comparison of stopping times over 150 runs.}
            \label{fig: sep-average-T}
        \end{subfigure}
        \hfill
        \begin{subfigure}[t]{0.45\textwidth}
            \centering
            \includegraphics[width = \textwidth]{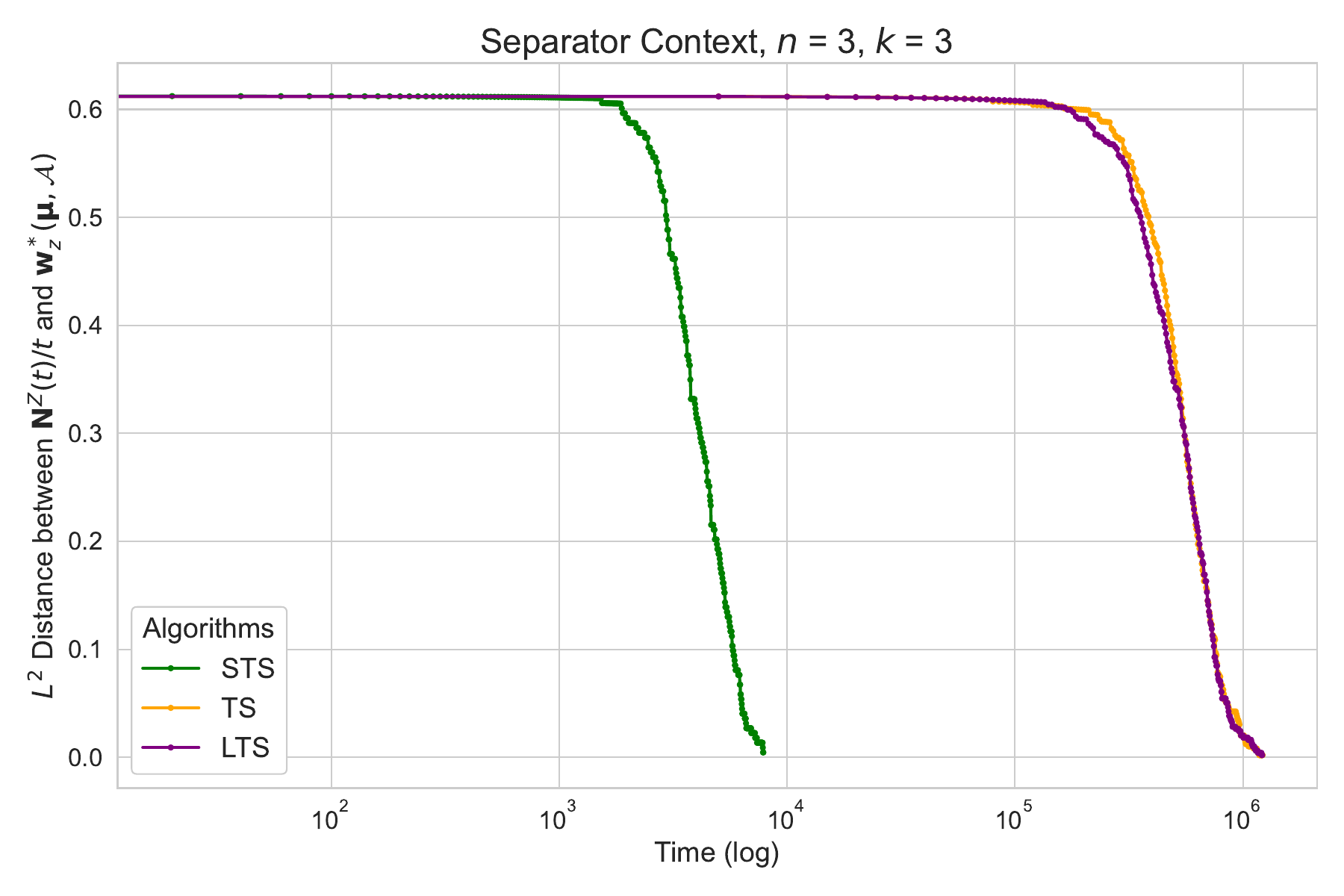}
            \caption{Average $L^2$ distance between the vectors $\frac{\Nb^Z(t)}{t}$ and $\wzstar{\bmu}$ during the learning process.}
            \label{fig: sep-w-distance}
        \end{subfigure}
        \caption{The results of different algorithms for an instance in Equation \eqref{eq: sep-instance}.}
        \label{fig: sep-main-fig}
\end{figure}

\section{Experiments} \label{sec: experiment}
This section presents a representative experiment for the separator setting, comparing our proposed algorithm with two benchmark algorithms. We use the Python CVXPY library \citep{cvx1-diamond2016cvxpy, cvx2-agrawal2018rewriting} to solve the required optimization problems in all experiments. Additional experiments on both synthetic and real-world data, together with a more comprehensive discussion, are provided in Appendix~\ref{apd: experiment}. The implementation is available at \url{https://github.com/ban-epfl/BAI-with-Post-Action-Context}.

\label{sub: sep-exp} We consider an instance with $n = 3$ arms and $k=3$ context values, specified by the following parameters
\begin{align} \label{eq: sep-instance}
     \delta = 0.01, \quad ~ 
     \bmu = \begin{bmatrix}
        1.0 \\
        0.1 \\
        0.3
    \end{bmatrix}, \quad ~
    \A = \begin{bmatrix}
        0.9 & 0.9 & 0.1 \\
        0.09 & 0.01 & 0.45 \\
        0.01 & 0.09 & 0.45
    \end{bmatrix}.
\end{align}
This instance is interesting because the first and second arms have very close average rewards, making it difficult to distinguish the best arm with high probability. To do so, one must observe a substantial number of samples from contexts $2$ and $3$, which have low probabilities of occurring when pulling arms $1$ and $2$. Consequently, an effective strategy involves frequently selecting arm $3$, even though it has a notably lower expected reward than arms $1$ and $2$.

This behavior differs from classic BAI, where suboptimal arms are typically chosen less often. In instances like this, algorithms designed for standard BAI (e.g., track-and-stop) become suboptimal because they do not take full advantage of the post-action context information.

We compare Algorithm \ref{algo: sep} (STS) with the following two baselines:
(i) \emph{Track-and-Stop (TS).} This classic algorithm ignores the context values and applies D-tracking to the actions.
(ii) \emph{Lazy Track-and-Stop (LTS).} It can be shown that the BAI problem with a separator post-action context can be reduced to a BAI problem in linear bandits, while the mean values of the reward distributions are bounded. For more details on this reduction, refer to Section \ref{sec: non-optimal-linear}. Consequently, identifying the best action translates into best-arm identification in a finite-arm linear bandit. We then apply Lazy Track-and-Stop \citep{linear-lazy-jedra2020optimal}, which is an optimal algorithm for BAI in linear bandits.

Figure \ref{fig: sep-average-T} illustrates the average number of arm pulls for each algorithm on the instance in Equation \eqref{eq: sep-instance}, aggregated over $150$ runs for each algorithm. Figure \ref{fig: sep-w-distance} shows the average $L^2$ distance of the vector $\frac{\Nb^{Z}(t)}{t}$ from the optimal frequency of contexts $\wzstar{\bmu}$ over time, which captures the convergence speed of each algorithm.
As the figures demonstrate, ignoring the post-action contexts can lead to significant sub-optimality.


\section{Conclusion}
We introduced a new BAI problem with post-action context in the fixed-confidence setting. We considered two settings for post-action context. For both settings, non-separator and separator context, we derived lower bounds on sample complexity and proposed algorithms that achieve asymptotic optimal sample complexity.

An interesting direction for future work is to consider the setting where context probability matrix $\A$ is unknown. In this case, as discussed in Appendix~\ref{apd: unknown context}, the optimization problem in Equation \eqref{eq: general_lower1} becomes non-convex which renders designing a computationally efficient algorithm challenging.


\bibliographystyle{alpha}
\bibliography{bibliography}

\appendix

\newpage
\begin{center} \label{sec: apd}
    {\Large \textbf{Appendix}}
\end{center}

The appendix is organized as follows. In Appendix \ref{apd: related-work}, we review the literature on related work. In Appendix \ref{sec: non-sep}, we provide the results of non-separator setting. Appendix \ref{apd: general bound} contains the proof of the general lower bound. In Appendix \ref{apd: GLR}, we discuss the GLR test. Appendices \ref{apd: non-sep proofs} and \ref{apd: sep proofs} provide proofs for the non-separator and separator cases, respectively. Appendix \ref{apd: opt-solving} addresses how to solve the optimization problem in Equation \eqref{eq: middle non-sep LB}, and Appendix~\ref{apd: unknown context} discusses the scenario where the context matrix is unknown. Finally, Appendix \ref{apd: experiment} provides additional details about our experiments and presents further numerical results.

\begin{table}[H]
    \centering
    
    \label{table: notations}
    \caption{Table of notations.}
    \small
    \begin{tabular}{c|l} 
        \toprule
        Notation & Description \\
        \hline
        $X$ & Action \\
        $Z$ & Post-action context \\
        $Y$ & Reward variables \\
        $n$ & Size of action set \\
        $k$ & Size of post-action context \\
        $[n]$ & $\{1, 2, \dots, n\}$ \\
        $\Delta^{n-1}$ & $(n-1)$-dimensional standard simplex = $\{w \in \mathbb{R}^n \mid w_i \geq 0 \wedge \sum_{i = 1}^{n} w_i = 1\}$ \\
        $d(P, Q)$ & Kullback–Leibler (KL) divergence between two probability measures $P$ and $Q$ \\
        $d_{B}(p, 1-p)$ & KL between two Bernoulli  with parameters $p$ and $1 - p$ \\
        $\ch(\{\A_1, \A_2, \dots, \A_n\})$ & Convex hull of \(n\) vectors \(\A_i \in \mathbb{R}^d\) = $ \left\{ \sum_{i = 1 }^{n} \lambda_i \A_i \mid \lambda_i \geq 0 \ \wedge \sum_{i = 1}^{n} \lambda_i = 1 \right\}$ \\
        $\A = [\A_1 \vert \A_2 \vert \ldots \vert \A_n]$ & Context probability matrix, where $\A_{i,j} = \pr(Z = i \mid X = j)$ \\
        $\bmu = [\bmu_1 \vert \bmu_2 \vert \ldots \vert \bmu_n]$ & Mean value matrix, where $\mu_{i,j} = \mathbb{E}(Y | X=i, Z=j)$ \\
        $P^{\bmu}_{i}, P^{\A, \bmu}_{i}$ & Joint distribution of $(Z, Y)$ after pulling arm $X = i$ in environment $\A, \bmu$ \\
        $\istarmu, i^*(\bmu, \A)$ & Best arm for an instance parameterized by $\bmu$ and $\A$ \\
        $\I, \I(\A)$ &  Set of $\bmu$ matrices that imply a unique best arm for a given $\A$ \\
        $\Is, \Is(\A)$ & Set of $\bmu$ vectors that imply a unique best arm for a given $\A$ in separator setting  \\  
        $\Delta_i$  &   Sub-optimality gap of arm $i$ \\
        $\delhit{i}$ &  Estimated sub-optimality gap of arm $i$ in round $t$ \\
        $\muht$ &       Matrix containing the empirical estimates of all entries of $\bmu$ up to round $t$ \\
        Alt$(\bmu, \A)$ & Set of alternative parameter matrices $\bmu' \in \I$ such that $\istarmu \neq i^*(\bmu')$ \\
        $\amin$ & $\min_{i,j} \A_{i,j}$   \\    
        $\numaction$ &    Total number of times arm $i$ has been played up to round $t$ \\
        $\numcontext$ &    Total number of times context $j$ has been observed up to round $t$ \\
        $\numac$ &   Total number of times arm-context $i, j$ has been observed up to round $t$ \\
        $\lamht$ &      GLR statistic at around $t$ \\
        $\mathbf{w}^*(\bmu, \A)$ &  The optimal weights (proportions) for actions in non-separator context \\
        $\wzstar{\bmu}$          &  The set of optimal weights (proportions) for contexts \\
        \bottomrule
    \end{tabular}
    
\end{table}

\section{Further Discussion on Related Work} \label{apd: related-work}

\paragraph{Best Arm Identification.} 
The literature on pure exploration in multi-armed bandit problems is extensive \citep{bandit-book1-lattimore2020bandit, pure2-thesis-stephens2023pure}. Here, we provide a brief overview of the most closely related works. The standard best-arm identification problem is primarily studied in two settings: the fixed-confidence setting, where the probability of error is predetermined \citep{track-stop-garivier2016optimal, SR-audibert2010best, kaufmann2020contributions, confidence-jamieson2014best, lb-tsitsiklis-mannor2004sample}, and the fixed-budget setting, where the time horizon is fixed \citep{budget-confidence-gabillon2012best, SH-karnin2013almost, lb-budget-carpentier2016tight}.

\cite{track-stop-garivier2016optimal} propose the first (asymptotically) optimal algorithm for the best-arm identification problem, known as track and stop, which attempts to track the optimal proportion of arm plays. The track-and-stop idea has since been widely adopted to design algorithms for various identification problems, including the algorithms presented in this paper.

In the fixed-confidence setting, other identification problems with different objectives have been explored in the literature, such as subset selection \citep{subset-selection1-kalyanakrishnan2012pac, subset-selection2-kaufmann2013information}, checking for the existence of an arm with a low mean \citep{badarm1-tabata2020bad, badarm2-kaufmann2018sequential}, and finding all the good arms \citep{all-good-arms-mason2020finding}. The most general form of the identification problem with bandit feedback has also been studied under various assumptions and is sometimes called \textit{sequential identification} \citep{kaufmann2021mixture}, \textit{General-Samp} \citep{general-samp-chen2017nearly}, \textit{partition identification} \citep{partition-id-juneja2019sample}, and is typically encompassed by the term pure exploration \citep{frank-wolf-wang2021fast, mutiple-correct-answers-degenne2019pure}. 
Our problem can be viewed as a general identification problem in which the agent cannot take arbitrary actions but can only choose probability distributions over them.

Best-arm identification in contextual bandits has also received considerable attention, with optimal algorithms and different constraints \citep{context1-li2022instance, context2-kato2021role}. In these works, however, the context is assumed to be generated (independently) before taking an action, which differs from our setting.

\begin{figure}
            \centering
            \begin{tikzpicture}
                \tikzset{line width=2pt, outer sep=0pt,
                ell/.style={draw,fill=white, inner sep=2pt,
                line width=2pt},
                }; 
                    \node[name=X, shape=circle, draw] at (-1, 0){$X$};
                    \node[name=Z, shape=circle, draw] at (-0.75, 1.5){$S$};
                    \node[name=Y, shape=circle, draw] at (1, 0){$Y$};
                
                \begin{scope}[>={Stealth[black]},
                              every edge/.style={draw=black,very thick}]
                    \path[->] (X) edge (Y);
                    \path[->] (Z) edge (Y);
                \end{scope}
            \end{tikzpicture}
            \caption{The structure of variables in  \cite{russac2021b}.}
            \label{fig: subpop}
\end{figure}

In particular, \cite{russac2021b} introduce an additional observation, called a \emph{sub-population}, and study four different modes. In the first mode, the \emph{action mode}, the learner chooses both the action and the sub-population. In the other modes, however, the sub-population is sampled independently of the chosen action. The third mode, called the \emph{agnostic mode}, is structurally related to the causal graph in Figure~\ref{fig: subpop}: the learner chooses an action and observes both the sub-population and the reward afterward. However, since the sub-population distribution is fixed and independent of the chosen action, this mode corresponds in our framework only to the special case where all columns of $\A$ are identical. Moreover, their objective is to identify the subset of actions that outperform a given control action. Therefore, their agnostic formulation only overlaps with our setting in the restricted two-action case with identical columns of $\A$. In contrast, our setting allows the post-action context distribution to depend on the chosen action and aims to identify a single best action. These differences make the two problems intrinsically distinct.

A more closely related line of research, recently explored, involves the bandit problem with mediator feedback. This setting is similar to our separator context setting but investigates different objectives \citep{metelli2021policy, mediator-poiani2023pure, mediator2-reddy2023best, mediator-cum-eldowa2024information}.

In the setting of BAI with fixed-confidence, deriving a closed-form expression for GLR requires that the reward distributions belong to a one-parameter exponential family, which includes the Gaussian distribution with known variance. This assumption is commonly found in most existing works. However, recently \cite{tuynman2025batch} extend the GLR-based approach to the more general class of sub-Gaussian distributions. \cite{tuynman2025batch} show that by adjusting the threshold in the stopping rule, one can adapt the GLR expression originally designed for Gaussian distributions to achieve an appropriate stopping rule for sub-Gaussian distributions (for more details, please refer to Lemmas 3.1 and 3.2 in \citep{tuynman2025batch}). This result shows we can further relax our assumption from Gaussian distribution to sub-Gaussian distributions but for simplicity we consider Gaussian distributions in this paper.

\paragraph{Causal Bandits.}
Another line of research in bandits, where the agent receives additional information beyond the reward, is the field of \textit{causal bandits}, first introduced by \cite{lattimore2016causal} and subsequently explored in \citep{causal-context-madhavan2024causal, causal-fateme-jamshidi2024confounded, causal-fairness-huang2022achieving, causal-pure-xiong2022combinatorial, shahverdikondori2025graph}. In causal bandit problems, a (known or unknown) causal graph is assumed among a set of variables, and the actions correspond to interventions on a subset of variables within the graph. After choosing each intervention (i.e., action), the agent observes the values of all nodes, with the goal of identifying interventions that maximize the expected value of a reward node.

In this paper, we study the best-arm identification problem for two particular graphs with three variables in a causal bandit setting. \cite{causal-benign-bilodeau2022adaptively, causal-pareto-liu2024causal} introduce the notions of post-action context and \emph{conditionally benign} property, which in absence of latent variables is identical to separator case in Section \ref{sec: sep}. However, their main focus is on cumulative regret minimization in the minimax setting, aiming to design algorithms that perform well in both conditionally benign and non-conditionally benign environments.

\paragraph{Linear Bandits.}
The separator setting of our problem can be reduced to a linear bandit problem through the following steps. We can disregard the context variable and assume that the reward value of action $i$ is drawn from a distribution with mean $\A_i^{\top} \bmu_i$ and a zero-mean noise $\epsilon_i$. Note that in this case, the mean of this new variable is the same as the mean reward for action $i$, and if the reward means are bounded, the noise becomes sub-Gaussian (for more details on this reduction, refer to Section \ref{sec: non-optimal-linear}). 

In this setting, action $i$ corresponds to a point $\A_i$ in $\mathbb{R}^k$, resulting in a best-arm identification problem in linear bandits with finitely many arms. This problem was first introduced by \cite{linear-bai-soare2014best} and has been widely studied. Optimal algorithms for best arm identification in linear bandits are provided in \citep{linear-lazy-jedra2020optimal, linear-degenne2020gamification, linear-mohammadi22improved}. \cite{linear-lazy-jedra2020optimal} propose an algorithm inspired by track and stop, called lazy track and stop, while \cite{linear-degenne2020gamification} present an algorithm based on game-theoretic intuition. For additional algorithms on best arm identification in linear bandits, see \citep{linear-elimination-xu2018fully, linear-transductive-fiez2019sequential}. The causal bandit with linear assumption on the causal model is also been studied recently \citep{causal-linear1-varici2023causal, causal-linear2-yan2024robust}. 

\paragraph{Best Policy Identification in RL.}
Another direction related to our problem is the best-policy identification problem in reinforcement learning, which has gained considerable attention in recent years \citep{rl1-menard2021fast, rl2-al2021navigating, rl3-wagenmaker2022beyond, rl4-al2023towards}. In best-policy identification, an agent interacts with an environment consisting of actions and states, and the goal is to select, from a given set of policies, the one that maximizes the expected reward. This problem generalizes best-arm identification in bandits, and many problems in the bandit literature can be viewed as special cases of this framework. For instance, in best-arm identification in contextual bandits, if we treat different contexts as distinct starting states, set the time horizon $H=1$, and consider the actions as the set of policies, the algorithms for best-policy identification can solve the problem; for a detailed discussion, see Section~6 of \cite{rl-difference-estimation-narang2024sample}.

One might consider reducing a bandit with post-action contexts to an RL setting by treating post-action contexts as states. However, this approach does not apply to either the non-separator or the separator settings. The key difference lies in how the reward is determined: in RL, the reward depends on the previous action and the state from which the action is taken. In contrast, in our problem, the reward depends on the last action and the new state reached after the action (non-separator) or only on the new state (separator). This fundamental difference prevents a direct application of RL algorithms to solve our problem.

\textbf{Pure exploration with safety constraints.} 
The classic best-arm identification algorithm fails to address many problems in which constraints on action selection arise for critical reasons. Such constraints are often used to ensure safety in learning algorithms and can restrict different phases of learning—whether in exploration or recommendation. Furthermore, constraints can be known or unknown, depending on the application. Prior work explores best-arm identification under various settings \citep{wang2021fairness, tang2024pure, faizal2022constrained, camilleri2022active}.

\cite{carlsson2024pure} also tackles a related problem, but with the key distinction that the goal is to find an optimal policy, a distribution over actions, rather than identifying a single arm. In fact, \cite{carlsson2024pure} is the closest work to ours, although they can choose any action (i.e., setting $Z$ to a specific value in our setting) during exploration. This flexibility is not available in our setting, constituting the primary difference from their work.

\section{Non-Separator Context} \label{sec: non-sep}

   In this section, we assume the context variable is a non-separator as depicted in Figure~\ref{fig: nonsep}, meaning the distribution of the reward variable can depend on both the action and the context. We first derive an instance-dependent lower bound on the number of samples required for any $\delta$-correct algorithm and then propose an algorithm that asymptotically achieves this lower bound. The omitted proofs for this section appear in Appendix \ref{apd: non-sep proofs}.

\subsection{Lower Bound}
    We establish a lower bound by explicitly solving the minimization problem defined in Equation \eqref{eq: general_fi}.
 
    \begin{restatable}[Non-Separator Lower Bound]{theorem}{nonsepLowerBound} \label{thm : non-sep lower} 
        Let $\delta \in (0, 1)$. Consider a bandit instance with a non-separator context and Gaussian reward distribution with unit variance, parameterized by matrices $\A$ and $\bmu$. Then, any $\delta$-correct algorithm with stopping time $\tau_{\delta}$ satisfies $\mathbb{E}[\tau_{\delta}] \geq T^*(\bmu, \A)  d_B(\delta, 1-\delta)$, where
                \begin{equation} \label{eq: middle non-sep LB}
                    T_{NS}^*(\bmu, \A)^{-1} = \sup_{\mathbf{w} \in \Delta^{n-1}} \min_{i \neq i^*(\bmu)} \frac{\Delta_{i}^2}{2}\left( \frac{w_{i^*(\bmu)}w_i}{w_{i^*(\bmu)} + w_i}\right).
                \end{equation} 
    \end{restatable}

    To design the sampling rule for the algorithm, one must solve the optimization problem given in Equation \eqref{eq: middle non-sep LB}. Solving this problem reduces to finding the root of a strictly decreasing function within a known interval, which can be efficiently done using methods such as binary search. For further details, refer to Appendix \ref{apd: opt-solving}.

\subsection{Learning Algorithm}

    In this part, we introduce the components of our algorithm and show that it achieves optimal sample complexity.
    
    \textbf{Stopping Rule.} Similar to the separator setting, we use a Generalized Likelihood Ratio (GLR) test for the stopping rule. The GLR statistic in this setting is defined as
    \begin{equation} \label{eq: glr-def}
        \lamht \triangleq  \inf_{\mathbf{\bmu'} \in \text{Alt}(\muht, \A)} \sum_{i \in [n], j \in [k]} N_{j,i}(t) \frac{(\hat{\bmu}_{j,i}(t) - \bmu'_{j,i})^2}{2},  
    \end{equation}
    provided that $\muht \in \I(\A)$ (i.e., $\muht$ induces a unique best arm). Otherwise, $\lamht$ is set to zero. 
    We can simplify Equation \eqref{eq: glr-def} to
    $$
        \lamht = \min_{i \neq \istarmut} \frac{\hat{\Delta}_{i}^2}{2 \sum_{j \in [k]} \frac{\A_{j, \istarmut}^2}{N_{j,\istarmut}(t)}  + \frac{\A_{j, i}^2}{N_{j,i}(t)}},
    $$
    which allows $\lamht$ to be computed efficiently. Appendix \ref{apd: glr-non-sep} provides a detailed proof of this simplification.
    
    To derive the sequential thresholds $\cdelt$, we use the deviation inequalities proposed by \cite{kaufmann2021mixture}, based on \textit{mixture martingales}. For each $\delta \in (0,1]$ and $t \in \mathbb{N}$, the threshold $\cdelt$ is derived as
    \begin{equation} \label{eq: non-sep threshold}
       \cdelt = 4k \ln\left( 4 + \ln\left( \frac{t}{2k} \right)\right) + 2k C^{g} \left( \frac{\ln\left(\frac{n-1}{\delta}\right)}{2k} \right),
    \end{equation}
    
    where $C^g(x) \simeq x + \ln(x)$ (refer to Definition \ref{def: c-g function}). For more details on the derivation of the thresholds, see the proof of Lemma \ref{lem: non-sep correctness} in Appendix \ref{apd: non-sep proofs}. The stopping time $\tau_{\delta}$ is then defined as
    \begin{equation} \label{eq: non-sep stop rule}
         \tau_{\delta} \triangleq \inf \{ t \in \mathbb{N} \mid \lamht > \cdelt \}.
    \end{equation}
    At the stopping time $\tau_{\delta}$, the final suggestion $\hat{i}_{\tau}$ is equal to the unique estimated best arm $i^*(\bmuh(\tau_{\delta}))$. The following lemma establishes the correctness of this stopping rule when combined with any sampling rule.


    \begin{restatable}{lemma}{nonsepCorrectness} \label{lem: non-sep correctness}
        Consider a bandit instance with a non-separator context and Gaussian reward distribution with unit variance, parameterized by matrices $\A$ and $\bmu$. Any algorithm with the stopping rule \eqref{eq: non-sep stop rule} is $\delta$-correct, that is, 
          $  \pr_{\bmu, \A}(\tau_{\delta} < \infty, \hat{i}_{\tau} \neq i^*(\bmu)) \leq \delta.$
    \end{restatable}

\textbf{Sampling Rule.}
    We now propose a sampling strategy, which asymptotically achieves optimal sample complexity with the stopping rule in \eqref{eq: non-sep stop rule}. Define $\mathbf{w}^*(\bmu, \A)$ as the solution to the optimization problem
    \begin{align*}
         \argmax_{\mathbf{w} \in \Delta^{n-1}} \min_{i \neq i^*(\bmu)} \frac{\Delta_{i}^2}{2}\left( \frac{w_{i^*(\bmu)}w_i}{w_{i^*(\bmu)} + w_i}\right).
    \end{align*}
    For any $\bmu \in \I$, the optimal weights $\mathbf{w}^*(\bmu, \A)$ are unique and can be efficiently estimated (see Appendix \ref{apd: opt-solving}). In our theoretical analysis, we assume the existence of an oracle that computes $\mathbf{w}^*(\bmu, \A)$ exactly. However, in practice, we estimate the optimal weights $\mathbf{w}^*(\bmu, \A)$. These weights correspond to the fraction of times each arm is pulled by an optimal algorithm. Based on this intuition, various algorithms have been developed for different identification problems in the literature. These algorithms compute the optimal weights at each step given the estimated reward matrix $\muht$ and track these sequential weights. We adopt the \textit{D-tracking} sampling rule first introduced by \cite{track-stop-garivier2016optimal}. At round $t$ of the algorithm, the D-tracking rule defines $\mathcal{U}_t = \{ i \in [n] \ \big| \ N^{X}_{i} (t) \leq \max\left(\sqrt{t} - \frac{n}{2}, \, 0\right) \},$ which contains the arms that are under-sampled. The algorithm then selects the next action $X_{t+1}$ as 
    \begin{align} \label{eq: d-tracking}
        X_{t+1} = \begin{cases}
            \argmin_{i \in \mathcal{U}_t} N^{X}_{i}(t) \hspace{2.5cm} \text{ if } \mathcal{U}_t \neq \emptyset, & \\ 
            \argmax_{i \in [n]} t w_i^*(\muht, \A) - N^{X}_{i}(t) \quad \text{ o.w.} & {}
        \end{cases}
    \end{align}
    It is noteworthy that D-tracking operates correctly only when  $\hat{\bmu} \in \I(\A)$, that is, a single optimal arm exists. If this condition is not satisfied at any step $t$, the next action is selected uniformly random. This scenario almost surely ceases to occur as the number of samples increases and $\hat{\bmu}(t)$ converges to $\bmu\in \I(\A)$, which is assumed to imply a single best arm.
    
    The pseudocode of our algorithm, called \textit{Non-Separator Track and Stop (NSTS)}, is presented in Algorithm \ref{algo: non-sep}. The following theorem proves its optimality.

    \begin{algorithm}[H]
        \caption{Non-Separator Track and Stop (NSTS)}
        \label{algo: non-sep}
        \begin{algorithmic}[1] 
            \STATE \textbf{Input:} Context probability matrix $\A$.
            \STATE \textbf{Initialization:} Pull arms until collecting at least one sample from $P(Y|X=i, Z=j)$ for each $i \in [n], j \in [k]$, then set $t$ to the number of rounds played.
            \WHILE{$\lamht \leq \cdelt$}
                \STATE Pull arm $X_{t+1}$ based on D-tracking rule \eqref{eq: d-tracking}. 
                \STATE Update $\muht$ and $t \gets t + 1.$
            \ENDWHILE 
            \STATE \textbf{Output:} $i^*(\muht)$. 
        \end{algorithmic}
    \end{algorithm}

    \begin{restatable}[Non-Separator Upper Bound]{theorem}{nonsepUpperBound} \label{thm : non-sep upper} 
            Algorithm \ref{algo: non-sep} applied to a bandit instance with a non-separator context and Gaussian reward distribution with unit variance, parameterized by matrices $\A$ and $\bmu$ attains
                \begin{equation*}
                    \limsup_{\delta \rightarrow 0} \frac{\mathbb{E}_{\bmu, \A}[\tau_{\delta}]}{\logdel} \leq T_{NS}^*(\bmu, \A),
                \end{equation*}
                where $T^*(\bmu, \A)$ is defined in Equation \eqref{eq: middle non-sep LB}.
    \end{restatable}

\subsection{Non-Optimality of Ignoring Non-Separator Context} \label{apx: non-optimal-non-sep}

    By comparing the non-separator lower bound presented in Equation \eqref{eq: middle non-sep LB} with the lower bound for the BAI problem with unit-variance Gaussian rewards in \cite{kaufmann2016complexity}, we observe that the optimization problems characterizing the lower bounds are identical. At first glance, this may suggest that the post-action context is unnecessary. However, the key distinction lies in the reward distribution: when $Z$ is ignored, the resulting reward distribution becomes a mixture of $k$ unit-variance Gaussians, rather than a single Gaussian.

    Since mixtures of Gaussians do not belong to the exponential family of distributions, algorithms developed for BAI within this family, like track and stop, are not directly applicable, and there are no theoretical guarantees for these distributions.
    One possible way is to exploit the fact that a mixture of Gaussians is sub-Gaussian. More precisely, the reward distribution for each action $X = i$ follows a sub-Gaussian distribution with $\sigma^2 = 1 + \frac{(b - a)^2}{4}$, where $b - a$ denotes the range of the reward means for $X=i$ and different context values, i.e., $\mu_{j, i} \in [a,b]$ for all $j \in [k]$ (by Lemma 4 of \cite{russac2021b}). For BAI with sub-Gaussian rewards of parameter $\sigma^2$, the sample complexity can be upper bounded by the sample complexity of BAI with $\sigma^2$-variance Gaussian rewards.

\section{Proof of General Lower bound: proposition \ref{th: general_lower1}} \label{apd: general bound}
    The proof employs the transportation lemma as outlined in Lemma 1 of \cite{kaufmann2016complexity}. The proof of the transportation lemma, which can be found in Appendix A and A.1 of \cite{kaufmann2016complexity}, relies on changes of distributions and the log-likelihood ratio expression. Extending this theorem to our setup is quite straightforward; we simply need to incorporate the context variables into the log-likelihood ratio formula, which means using    
    $$
        L_t = \sum_{a = 1}^{n} \sum_{s = 1}^{N^{X}_a(t)} \log \left ( \frac{f_a(Z_{a,s}, Y_{a,s})}{f'_a(Z_{a,s}, Y_{a,s})} \right),
    $$
    instead of 
    $$
        L_t = \sum_{a = 1}^{n} \sum_{s = 1}^{N^{X}_i(t)} \log \left ( \frac{f_a(Y_{a,s})}{f'_a(Y_{a,s})} \right),
    $$
    in Appendix A of \cite{kaufmann2016complexity}. All the proofs of the transportation lemma remain unchanged. Therefore, based on the transportation lemma for our setting, for each $\bmu' \in \text{Alt}(\bmu, \A)$, we have  

    $$
        \sum_{i = 1}^{n} \expec {N^{X}_{i}(\taudel)} d(P^{\bmu}_{i}, P^{\bmu'}_{i}) \geq \sup_{\E \in \mathcal{F}_{\taudel}} d (\pr_{\bmu}(\E), \pr_{\bmu'}(\E)).
    $$
    Define event $\E' = \{\taudel < \infty, \hat{i}_{\taudel} \neq i^*(\bmu)\}$, then $\pr_{\bmu}(\E') \geq 1 - \delta$ and  $\pr_{\bmu'}(\E') \leq \delta$ according to $\delta$-correctness property of the given algorithm. Therefore, based on the property of KL divergence, we have
    $$
        \sup_{\E \in \mathcal{F}_{\taudel}} d (\pr_{\bmu}(\E), \pr_{\bmu'}(\E)) \geq  d (\pr_{\bmu}(\E'), \pr_{\bmu'}(\E')) \geq d_{B} (\delta, 1 - \delta),
    $$
    which implies that
    $$
        \sum_{i = 1}^{n} \expec {N^{X}_{i}(\taudel)} d(P^{\bmu}_{i}, P^{\bmu'}_{i}) \geq d_{B} (\delta, 1 - \delta).
    $$
    Therefore, 
    $$ 
       \inf_{\bmu' \in \text{Alt}(\bmu, \A)} \sum_{i = 1}^{n} \expec {N^{X}_{i}(\taudel)} d(P^{\bmu}_{i}, P^{\bmu'}_{i}) \geq d_{B} (\delta, 1 - \delta).
    $$
    Then, we have    
    \begin{align*}
        \inf_{\bmu' \in \text{Alt}(\bmu, \A)} \sum_{i = 1}^{n} \expec {N^{X}_{i}(\taudel)} d(P^{\bmu}_{i}, P^{\bmu'}_{i}) &= \expec{\taudel}  \inf_{\bmu' \in \text{Alt}(\bmu, \A)} \sum_{i = 1}^{n} \frac{\expec {N^{X}_{i}(\taudel)}}{\expec{\taudel}} d(P^{\bmu}_{i}, P^{\bmu'}_{i}) \\
              & \leq \expec{\taudel} \sup_{\mathbf{w} \in \Delta^{n-1}} \inf_{\bmu' \in \text{Alt}(\bmu, \A)} \sum_{i \in [n]} w_i d(P^{\bmu}_{i}, P^{\bmu'}_{i}). 
    \end{align*}
    Finally,
    $$
        \expec{\taudel}    \geq \bigg ( \sup_{\mathbf{w} \in \Delta^{n-1}} \inf_{\bmu' \in \text{Alt}(\bmu, \A)} \sum_{i = 1}^{n} w_i d(P^{\bmu}_{i}, P^{\bmu'}_{i}) \bigg )^{-1} d_{B} (\delta, 1 - \delta)  
    $$
    which concludes the proof.

\section{Generalized Likelihood Ratio Test}\label{apd: GLR}

    The Generalized Likelihood Ratio (GLR) is used to measure our confidence in a hypothesis. Let have $\Omega_0$ and $\Omega_1$ be two model spaces. We wish to know whether a model $\lambda$ belongs to $\Omega_0$ or $\Omega_1$. To do this, we can define two hypotheses: (i) $H_0: (\lambda \in \Omega_0)$ and (ii) $H_1: (\lambda \in \Omega_1$). The log-likelihood ratio of this test after $t$ observations is defined as 
        \begin{equation}\label{eq: def-log-likeli}
              \lamht \triangleq \log \frac{\sup_{\lambda \in \Omega_{0} \cup \Omega_{1}} \lcal(X_1, X_2, \dots, X_t; \lambda)}{\sup_{\lambda \in \Omega_{0}} \lcal(X_1, X_2, \dots, X_t; \lambda)},
        \end{equation}
        where $\lcal(X_1, X_2, \dots, X_t; \lambda)$ is the likelihood function of observations (i.e., $X_1, X_2, \dots, X_t$) associated to parameter $\lambda$. If $\lamht$ has a high value, we could reject the hypothesis $\hcal_0$.

    Our problem in both settings can be viewed as a \emph{general identification}  \citep{kaufmann2021mixture}. In this problem, we assume $\I$ is partitioned into $n$ subsets $\ocal_1, \ocal_2, \ldots, \ocal_n$, and the goal is to determine which subset contains vector $\bmu$. At round $t$, we have an estimate $\hat{\bmu}(t)$ and we need to decide whether we can confidently choose the correct subset with low risk. Refer to Section 4 of \cite{kaufmann2021mixture} for more details. In our problem, we can define the subsets    
    
    \begin{align*}
        \mathcal{O}_i \triangleq \left\{ \bmu \in \I \ \big| \A_i^\top \bmu_i > \A_j^\top \bmu_j, \ \forall j \neq i \right\}.
    \end{align*}

    Suppose that at round $t$ of our problem, $\istarmut$ is not unique, meaning $\muht$ is not in $\I$. In this case, we lack sufficient information to determine the best arm (or, find the correct subset in general identification problem), so we must continue sampling. If $\istarmut$ indicates a unique arm, however, we need a measure to evaluate our confidence. We consider the following hypothesis test.
        \begin{align*}
              &\hcal_0: (\istarmu \neq \istarmut),   \\
              &\hcal_1: (\istarmu = \istarmut) ,
        \end{align*}
    which is equivalent to
         \begin{align*}
              &\hcal_0: (\bmu \notin \ocal_{\istarmut}) =  (\bmu \in \text{Alt}(\muht, \A)),  \\
              &\hcal_1: (\bmu \in \ocal_{\istarmut}).
        \end{align*}
    Hence, if $\lamht$ is high, we can reject the $\hcal_0$, indicating with high confidence that \(\istarmut\) is indeed the best arm
    
    Note that at round $t$, the observed samples are $((Z_1, Y_1), \dots, (Z_t, Y_t))$ for the chosen actions $(i_1, \dots, i_t)$. Consequently,     
        $$
            \lamht = \log \frac{\sup_{\bmu \in I} \lcal((Z_1, Y_1), \dots, (Z_t, Y_t); \bmu)}{\sup_{\bmu \in \text{Alt}(\muht, \A)} \lcal((Z_1, Y_1), \dots, (Z_t, Y_t); \bmu)},
        $$
    which expands to
        \begin{align}\label{eq: glr-context}    
                  & \log \frac{\sup_{\bmu' \in I} \prod_{s \in [t]} \pr(Z = Z_s \mid X = i_s) \pr_{\bmu'} (Y = Y_s \mid X = i_s, Z = Z_s)}{\sup_{\bmu' \in \text{Alt}(\muht, \A)} \prod_{s \in [t]} \pr(Z = Z_s \mid X = i_s) \pr_{\bmu'} (Y = Y_s \mid X = i_s, Z = Z_s)} \nonumber \\
                  =  & \log \frac{\sup_{\bmu' \in I} \prod_{s \in [t]} \A_{Z_s, i_s} \pr_{\bmu'} (Y = Y_s \mid X = i_s, Z = Z_s)}{\sup_{\bmu' \in \text{Alt}(\muht, \A)} \prod_{s \in [t]} \A_{Z_s, i_s} \pr_{\bmu'} (Y = Y_s \mid X = i_s, Z = Z_s)} \nonumber \\
                  = & \log \frac{\sup_{\bmu' \in I} \prod_{s \in [t]} \pr_{\bmu'} (Y = Y_s \mid X = i_s, Z = Z_s)}{\sup_{\bmu' \in \text{Alt}(\muht, \A)} \prod_{s \in [t]} \pr_{\bmu'} (Y = Y_s \mid X = i_s, Z = Z_s)} 
        \end{align}
    
    In the following two parts, for each case of context, we derive a simpler expression for $\lamht$ under the non-separator and separator context.
    
    \subsection{Non-separator context}\label{apd: glr-non-sep}
        Note that in this case, $\pr_{\bmu} (Y = Y_s \mid X = i_s, Z = Z_s)$ is not necessarily equal to $\pr_{\bmu} (Y = Y_s \mid Z = Z_s)$. By assumption, $Y \mid X = i_s, Z = Z_s$ follows a Gaussian distribution with unit variance, which is a one-parameter exponential family. For this case, we can simplify Equation \eqref{eq: glr-context} to
   
        \begin{align*}
             \lamht &= \inf_{\bmu' \in \text{Alt}(\muht, \A)} \sum_{i \in [n], j \in [k]} N_{j,i}(t) d(\N(\hat{\bmu}_{j,i}(t), 1), \N(\bmu'_{j,i}, 1)) \\
                    &= \inf_{\bmu' \in \text{Alt}(\muht, \A)} \sum_{i \in [n], j \in [k]} N_{j,i}(t) \frac{(\hat{\bmu}_{j,i}(t) - \bmu'_{j,i})^2}{2}
        \end{align*}
        where $\N(x, 1)$ denotes Gaussian distribution with mean $x$ and unit variance and we used the fact $d(\N(x, 1), \N(y, 1)) = \frac{(x - y)^2}{2}$.

        \textbf{Computation of $\lamht$}.
            Recall that Alt$(\bmu, \A) = \cup_{i \neq \istarmu} \C_i$,  where $\C_i = \{ \bmu' \in \I \mid \A_i^\top \bmu'_i > \A_{\istarmu}^\top \bmu'_{\istarmu} \}$. Thus, we can rewrite $\lamht$ as
            $$
                \lamht = \min_{s \neq  \istarmut} \inf_{\bmu' \in \C_s} \sum_{i \in [n], j \in [k]} N_{j,i}(t) \frac{(\hat{\bmu}_{j,i}(t) - \bmu'_{j,i})^2}{2},
            $$
            which is equal to
            \begin{align*}
                &\inf_{\bmu' \in \I} \sum_{i \in [n], j \in [k]} N_{j,i}(t) \frac{(\hat{\bmu}_{j,i}(t) - \bmu'_{j,i})^2}{2} \\
                & \text{s.t.} \quad \A_s^{\top} \bmu'_s > \A_{\istarmu}^{\top} \bmu'_{\istarmu}.
            \end{align*}
            Note that $N_{j,i}(t) > 0$ for all $i\in [n]$ and $j \in [k]$ due to initialization phase. \(\I\) is dense in \(\mathbb{R}^{k \times n}\). Hence, we can replace the constraint \(\bmu' \in \I\) with \(\bmu' \in \mathbb{R}^{k \times n}\), thus making
            \begin{align*}
                &\inf_{\bmu' \in \mathbb{R}^{k \times n}} \sum_{i \in [n], j \in [k]} N_{j,i}(t) \frac{(\hat{\bmu}_{j,i}(t) - \bmu'_{j,i})^2}{2} \\
                & \text{s.t.} \quad \A_s^{\top} \bmu'_s > \A_{\istarmu}^{\top} \bmu'_{\istarmu}.
            \end{align*}
            Because the objective function is continuous and we are taking infimum, we can replace the strict inequality constraint with a non-strict one
            \begin{align*}
                &\inf_{\bmu' \in \mathbb{R}^{k \times n}} \sum_{i \in [n], j \in [k]} N_{j,i}(t) \frac{(\hat{\bmu}_{j,i}(t) - \bmu'_{j,i})^2}{2} \\
                & \text{s.t.} \quad \A_s^{\top} \bmu'_s \geq \A_{\istarmu}^{\top} \bmu'_{\istarmu}.
            \end{align*} 
            We solve this using the Lagrangian
            \begin{equation}\label{eq: lag-glr-non}
                   L(\bmu', \lambda) = \sum_{i \in [n], j \in [k]} N_{j,i}(t) \frac{(\hat{\bmu}_{j,i}(t) - \bmu'_{j,i})^2}{2}+ \lambda(\A_{\istarmu}^{\top} \bmu'_{\istarmu} - \A_s^{\top} \bmu'_s).
            \end{equation}

            Taking derivatives and simplifying, we obtain
            \begin{align*}
                \bmu'_{j. i} = \begin{cases}
                                    \hat{\bmu}_{j,\istarmu}(t) - \frac{\lambda \A_{j, \istarmu}}{N_{j, \istarmu}(t)} &    i = \istarmu, \\
                                    \hat{\bmu}_{j,s}(t) + \frac{\lambda \A_{j, s}}{N_{j, s}(t)}      &    i = s, \\
                                    \hat{\bmu}_{j,i}(t)        &    \text{o.w.}                        
                                \end{cases}       
             \end{align*}
            Substituting $\bmu'_{j,i}$ back into \eqref{eq: lag-glr-non} yields
                \begin{align*}
                    \lambda^* = \frac{\hat{\Delta}_{s}}{\sum_{j \in [k]} \frac{\A_{j, \istarmut}^2}{N_{j,\istarmut}(t)}  + \frac{\A_{j, s}^2}{N_{j,s}(t)}}.
                \end{align*}
            Hence,
                 \begin{align*}
                    \inf_{\bmu' \in \C_s} \sum_{i \in [n], j \in [k]} N_{j,i}(t) \frac{(\hat{\bmu}_{j,i}(t) - \bmu'_{j,i})^2}{2} =  \frac{\hat{\Delta}_{s}^2}{2 \sum_{j \in [k]} \frac{\A_{j, \istarmut}^2}{N_{j,\istarmut}(t)}  + \frac{\A_{j, s}^2}{N_{j,s}(t)}},
                \end{align*}
            Finally, the simplified form of $\lamht$ becomes
                \begin{align*}
                      \lamht = \min_{s \neq  \istarmut} \frac{\hat{\Delta}_{s}^2}{2 \sum_{j \in [k]} \frac{\A_{j, \istarmut}^2}{N_{j,\istarmut}(t)}  + \frac{\A_{j, s}^2}{N_{j,s}(t)}},
                \end{align*}
            
    \subsection{Separator context}\label{apd: glr-sep}
        In the separator case, $\pr_{\bmu} (Y = Y_s \mid X = i_s, Z = Z_s) = \pr_{\bmu} (Y = Y_s \mid Z = Z_s)$ holds. Thus,  
        $$
            \lamht  = \log \frac{\sup_{\bmu' \in I} \prod_{s \in [t]} \pr_{\bmu'} (Y = Y_s \mid Z = Z_s)}{\sup_{\bmu' \in \text{Alt}(\muht, \A)} \prod_{s \in [t]} \pr_{\bmu'} (Y = Y_s \mid Z = Z_s)} 
        $$
        In Section \ref{sec: sep}, we assume $\bmu$ is a vector in $\mathbb{R}^k$, where $\bmu_i = \expec{Y \mid Z = i}$. Similar to the non-separator setting, because $Y \mid Z = i_s$ is a Gaussian variable with unit variance, then obtain
        \begin{align*}
            \lamht &= \inf_{\bmu' \in \text{Alt}(\muht, \A)} \sum_{j \in [k]} \numcontext d(\N(\hat{\bmu}_{j}(t), 1), \N(\bmu'_{j}, 1)) \\
                  &= \inf_{\bmu' \in \text{Alt}(\muht, \A)} \sum_{j \in [k]} \numcontext \frac{(\hat{\bmu}_{j}(t) - \bmu'_{j})^2}{2}.
        \end{align*}

        \textbf{Computation of $\lamht$}.         
        Let $\Is(\A)$ be the set of vectors in $\mathbb{R}^k$ that ensure a unique best arm in the separator setting with context probability matrix $\A$. The computation here is similar to that in the non-separator setting. Another expression for $\lamht$ is  
            \begin{align*}
                 \lamht = \min_{s \neq  \istarmut} \inf_{\bmu' \in \C_s} \sum_{j \in [k]} \numcontext \frac{(\hat{\bmu}_{j}(t) - \bmu'_{j})^2}{2}.
            \end{align*}
            where, in this case, $\C_s = \{\bmu' \in \Is \mid \A_s^{\top} \bmu' > \A_{\istarmu}^{\top} \bmu' \}$. Since $\Is$ is dense in $\mathbb{R}^{k}$ and the objective function is continuous, the 
            inner term is equivalent to       
            \begin{align*}
                &\inf_{\bmu' \in \mathbb{R}^{k}} \sum_{j \in [k]} \numcontext \frac{(\hat{\bmu}_{j}(t) - \bmu'_{j})^2}{2} \\
                & \text{s.t.} \quad \A_s^{\top} \bmu' > \A_{\istarmu}^{\top} \bmu'.
            \end{align*}
            Because of the initialization phase, $\numcontext > 0$ for all $j \in [k]$. We solve this using the Lagrangian method 
            \begin{equation*}
                   L(\bmu', \lambda) = \sum_{j \in [k]} \numcontext \frac{(\hat{\bmu}_{j}(t) - \bmu'_{j})^2}{2} + \lambda(\A_{\istarmu}^{\top} \bmu' - \A_s^{\top} \bmu').
            \end{equation*}
            After differentiation and some algebra, we obtain  
            \begin{align*}
                \bmu^{*}_{j} &= \hat{\bmu}_{j}(t) -  \frac{\lambda^*(\A_{j, \istarmu} - \A_{j, s})}{N_j(t)}               \\
                \lambda^* &= \frac{\hat{\Delta}_{s}}{\sum_{j \in [k]} \frac{(\A_{j, \istarmut} - \A_{j, s})^2}{\numcontext}}.                            
             \end{align*}
             Thus,
             \begin{align*}
                 \lamht = \min_{s \neq  \istarmut} \frac{\hat{\Delta}^2_{s}}{2\sum_{j \in [k]} \frac{(\A_{j, \istarmut} - \A_{j, s})^2}{\numcontext}}.
             \end{align*}   

\section{Proofs of Section \ref{sec: non-sep}} \label{apd: non-sep proofs}

The definition of $C^g$ from \cite{kaufmann2021mixture} is as follows.

\begin{definition}\label{def: c-g function}
   The function $C^g$ is defined as
    \begin{align*}
            C^g(x) \coloneqq \min_{\lambda \in (\frac12, 1]} \frac{g(\lambda) + x}{\lambda},
    \end{align*}
          where function $g: \left(\frac{1}{2}, 1\right] \rightarrow \mathbb{R}$ is given by
    \begin{align*}
            g(\lambda) = 2\lambda - 2\lambda\ln(4\lambda) + \ln(\zeta(2\lambda)) - \frac12 \ln(1 - \lambda),          
    \end{align*}
    and $\zeta$ denotes the Riemann zeta function, defined as $\zeta(s) = \sum_{n=1}^{\infty} n^{-s}$.
\end{definition}

\subsection{Proof of Theorem \ref{thm : non-sep lower}}

For a given vector $\mathbf{w} \in \mathbb{R}^{n}$, define 
\begin{align} \label{eq: general_fi}
        f_j(\mathbf{w}, \bmu, \A) \triangleq \inf_{\bmu' \in \C_j} \sum_{i \in [n]} w_i d(P^{\bmu}_{i}, P^{\bmu'}_{i}).
\end{align}
Equation \eqref{eq: general_lower1} can then be reformulated as
\begin{align} \label{eq: general_lower2}
        T^*(\bmu, \A)^{-1} = \sup_{\mathbf{w} \in \Delta^{n-1}} \min_{i \neq \istarmu} f_i(\mathbf{w}, \bmu, \A).
\end{align}

We prove Theorem \ref{thm : non-sep lower} by explicitly solving the optimization problem \eqref{eq: general_fi} for the non-separator setting and incorporating the solution into the general lower bound formula \eqref{eq: general_lower2}.
\begin{restatable}{lemma}{nonsepfi} \label{lem : non-sep fi} 
            Consider a bandit instance with a non-separator context and Gaussian reward distribution with unit variance, parameterized by matrices $\bmu$ and $\A$. For any weight vector $\mathbf{w}$ such that $\forall s \in [n]: w_s > 0$, the value of $\fiwmu = \inf_{\bmu' \in \C_i} \sum_{s \in [n]} w_s d(P^{\bmu}_{s}, P^{\bmu'}_{s})$ can be explicitly determined as
            \begin{align*}
                \fiwmu = \frac{\Delta_{i}^2}{2}\left( \frac{w_{i^*(\bmu)}w_i}{w_{i^*(\bmu)} + w_i}\right).  
            \end{align*}
\end{restatable}

\begin{proof}
            First note that $P^{\bmu}_{s}$ and $P^{\bmu'}_{s}$ are the distribution of random vectors $(Z,Y)$ and $(Z',Y')$ corresponding to bandit instances with context probability matrix $\A$ and means matrix $\bmu$ and $\bmu'$, respectively. Since reward distributions are Gaussian with unit variance, the KL divergence can be written as            
            \begin{align*}
                d(P^{\bmu}_{s}, P^{\bmu'}_{s}) = &\sum_{j \in [k]} P^{\bmu}_{s}(Z = j) d(P^{\bmu}_{s} (Y \mid Z = j), P^{\bmu'}_{s} (Y \mid Z = j)) \\
                                         =&\sum_{j \in [k]} \A_{j, s} d(\N(\mu_{j,s} , 1), \N(\mu'_{j,s}, 1))= \sum_{j \in [k]} \A_{j,s} \frac{(\mu_{j,s} - \mu'_{j, s})^2}{2}. 
            \end{align*}
            
            Recall the definition 
            $$\C_i = \left\{ \bmu' \in \I \big| \A_i^{\top}\bmu'_i > \A_{i^*(\bmu)}\bmu'_{i^*(\bmu)} \right\}.$$
            
           Note that $\sum_{s \in [n]} w_s d(P^{\bmu}_{s}, P^{\bmu'}_{s})$ is continuous with respect to $\bmu$ and $\I$ is dense in $\mathbb{R}^{k \times n}$. Therefore, we can express $\fiwmu$ as the solution to the following optimization problem
            \begin{align*}
                \fiwmu = &\inf_{{\bmu' \in \mathbb{R}^{k \times n}}} \sum_{s \in [n]} w_s \sum_{j \in [k]} \A_{j,s}  \frac{(\mu_{j,s} - \mu'_{j,s})^2}{2} \\
                & \text{s.t.} \quad \A_i^{\top} \bmu'_i > \A_{\istarmu}^{\top} \bmu'_{\istarmu}.
            \end{align*}
            The solution to this problem is equivalent to the solution to the following problem when $\alpha \rightarrow 0$,
            \begin{align*}
                &\inf_{\bmu' \in \mathbb{R}^{k \times n}} \sum_{s \in [n]} w_s \sum_{j \in [k]} \A_{j,s}  \frac{(\mu_{j,s} - \mu'_{j,s})^2}{2} \\
                & \text{s.t.} \quad \A_i^{\top} \bmu'_i \geq \A_{\istarmu}^{\top} \bmu'_{\istarmu} + \alpha.
            \end{align*}
            To solve the second problem, we write the Lagrangian as 
            \begin{align*}
                L(\bmu', \lambda) = \sum_{s \in [n]} w_s \sum_{j \in [k]} \A_{j,s}  \frac{(\mu_{j,s} - \mu'_{j,s})^2}{2} + \lambda(\A_{\istarmu}^{\top} \bmu'_{\istarmu} - \A_i^{\top} \bmu'_i + \alpha).
            \end{align*}
            By calculating the derivative of $L$ with respect to each variable and performing some algebra, we derive the optimal values for $\bmu'^*, \lambda^*$: 
            \begin{align*}
                \bmu'^*_{j,t} = 
                \begin{cases}
                    \mu_{j, \istarmu} - \frac{\lambda}{w_{\istarmu}} \quad &t = \istarmu, \\ 
                    \mu_{j,i} + \frac{\lambda}{w_i} \quad & t=i, \\
                    \mu_{j,t} \quad & o.w.
                \end{cases}
            \end{align*}
            \begin{align*}
                \lambda^* = \frac{\Delta_i + \alpha}{\left( \frac{w_{\istarmu} + w_i}{w_{\istarmu} w_i}\right)}. 
            \end{align*}
            By inserting these values into the objective function, we have
            \begin{align*}
                \fiwmu = \lim_{\alpha \rightarrow 0} \frac{(\Delta_i + \alpha)^2}{2} \left( \frac{w_{\istarmu} w_i}{w_{\istarmu} + w_i} \right) = \frac{\Delta_i^2}{2} \left( \frac{w_{\istarmu} w_i}{w_{\istarmu} + w_i} \right),
            \end{align*}
            which completes the proof.
        \end{proof}
    
    For a given $\wb \in \Delta^{n-1}$, if $w_i = 0$ for some $i \in [n]$, then we can find $\bmu'$ such that $\pr (P^{\bmu'}_{s}) = \pr (P^{\bmu}_{s})$ for all $s \in [k]$ and also $i^*(\bmu) \neq i^*(\bmu)$. Therefore, the term $\inf_{\bmu' \in \text{Alt}(\bmu, \A)} \sum_{i \in [n]} w_i d(P^{\bmu}_{i}, P^{\bmu'}_{i})$ is equal to zero. Since we want to find $\wb \in \Delta^{n-1}$ that maximizes
        \begin{align*}
               \inf_{\bmu' \in \text{Alt}(\bmu, \A)} \sum_{i \in [n]} w_i d(P^{\bmu}_{i}, P^{\bmu'}_{i}),
        \end{align*}
    we must restrict ourselves to \(\wb\) with strictly positive entries. By combining Lemma \ref{lem : non-sep fi} and Equation \eqref{eq: general_lower2}, we conclude Theorem \ref{thm : non-sep lower}.

\subsection{Proof of Lemma \ref{lem: non-sep correctness}}

\nonsepCorrectness*

\begin{proof}
         As discussed in Appendix \ref{apd: GLR}, our problem can be viewed as a general identification, in which $\I$ is partitioned into $M$ subsets $\ocal_1, \ocal_2, \ldots, \ocal_M$. The goal is to determine which subset contains vector $\bmu$. \cite{kaufmann2021mixture} provide a $\delta$-correct stopping rule for general identification based on the \textit{rank} of problem, defined as follows.         
        \begin{definition}[Rank]
            Consider a general identification problem specified by a partition $\ocal = \bigcup_{i=1}^M \mathcal{O}_i$. We say this problem has rank $R$ if for every $i \in \{1, \ldots, M\}$, we can write
            \[
            \mathcal{O} \setminus \mathcal{O}_i = \bigcup_{q \in [Q]} \left\{ \lambda \in \I \; \middle| \; (\lambda_{k^{i,q}_1}, \ldots, \lambda_{k^{i,q}_R}) \in \mathcal{L}_{i,q} \right\},
            \]
            for a family of arm indices $k^{i,q}_r \in [K]$ and open sets $\mathcal{L}_{i,q}$ indexed by $r \in [R]$, $q \in [Q]$ and $i \in [M]$. In other words, the rank is $R$ if every set $\mathcal{O} \setminus \mathcal{O}_i$ is a finite union of sets that are each defined in terms of only $R$ arms. 
        \end{definition} 
        Based on the rank of a general identification problem, \cite{kaufmann2021mixture} provide the following theorem, which shows that the GLR stopping rule is $\delta$-correct.        
        \begin{theorem}[\cite{kaufmann2021mixture}] \label{th: rank-id}
            For any identification problem of rank $R$ with $M$ partitions and Gaussian reward distributions with unit variance, the GLR stopping rule \ref{eq: non-sep stop rule} is $\delta$-correct with threshold
            \[\hat{c}_t(\delta) = 2R \ln\left(4 + \ln\left(\frac{t}{R}\right)\right) + R C^{g} \left(\frac{\ln \left( \frac{M-1}{\delta} \right) }{R} \right)
            .\]
            where \(C^{g}\) is defined in Definition \ref{def: c-g function}.
        \end{theorem}  
    
        Next, we show that the rank of the best arm identification problem in a bandit with a non-separator context is $2k$. The partitions in this problem are given by
        \begin{align*}
           \mathcal{O}_i = \left\{ \bmu \in \I \ \big| \A_i^\top \bmu_i > \A_j^\top \bmu_j, \ \forall j \neq i \right\},
        \end{align*}
        for all $i \in [n]$. It follows that
        \begin{align*}
            \I \setminus \mathcal{O}_i = \bigcup_{j \neq i} \{ \bmu \in \I \big| (\bmu_{i}, \bmu_j) \in \mathcal{L}_{i,j} \},
        \end{align*} 
        where $\mathcal{L}_{i,j} = \{ (\bmu_1, \bmu_2) \in \mathbb{R}^{2k} \big| \A_i^{\top} \bmu_1 < \A_j^{\top} \bmu_2 \}$. This shows that our problem has rank $R = 2k$. By applying Theorem \ref{th: rank-id} with $R = 2k$ and $M = n$ to our problem, we complete the proof.
    \end{proof}

\subsection{Proof of Theorem \ref{thm : non-sep upper}}


\begin{proof}
        At the beginning of the proof, we provide some important lemmas that we need.
    

    \begin{lemma} \label{lem: thresholds upper bound}
        There exist constants $D,E$, such that the sequential thresholds $\cdelt$ defined in \eqref{eq: non-sep threshold}, have the following property
        \begin{align*}
            \forall t > E, \forall \delta \in (0,1]: \cdelt \leq \ln \left( \frac{Dt}{\delta} \right)
        \end{align*}
    \end{lemma}
    \begin{proof}

  
        Note that, based on the Definition of \ref{def: c-g function}, we have

        \begin{align*}
            C^g(x) = \min_{\lambda \in (0.5, 1]} \frac{x + g(\lambda)}{\lambda}
        \end{align*}

        Substituting the definition of $g$ gives
        \begin{align*}
            C^g(x) &= \min_{\lambda \in (0.5, 1]} \frac{x + 2\lambda - 2\lambda\ln(4\lambda) + \ln(\zeta(2\lambda)) - \frac12 \ln(1 - \lambda)}{\lambda} \\
                 &= \min_{\lambda \in (0.5, 1]} \frac{x - \frac12 \ln(1 - \lambda) }{\lambda} + \frac{2\lambda - 2\lambda\ln(4\lambda) + \ln(\zeta(2\lambda))}{\lambda}.
        \end{align*}
        Now, for large value of $x$, if we set $\lambda = 1 - \frac{1}{x}$, then $\zeta(2\lambda)$ is upper bounded by a constant and therefore, we have
        \begin{align*}
            C^g(x) \leq c' + \frac{x + \frac12 \ln(x)}{1 - \frac{1}{x}}.
        \end{align*}
        where $c'$ is a constant value. Since $x$ is large, we have
        $$
          \frac{x + \frac12 \ln(x)}{1 - \frac{1}{x}} \leq x + \ln(x).
        $$
        Consequently, $C^g(x)$ can be upper bounded by
        $$
            C^g(x) \leq x + \ln(x) + c'
        $$
        for sufficiently large values of $x$.

        For small values of $x$, we can bound by setting $\lambda = 0.75$,
        $$
              C^g(x) \leq  \frac{4}{3}(x + g(0.75)) = \frac{4}{3}x + c''.
        $$
        Therefore, by choosing an appropriate constant $c$, we obtain
        $$
             C^g(x) \leq x + \ln(x) + c.
        $$

        Returning to the expression for $\cdelt$, we have
        \begin{align*}
            \cdelt &= (4k) \ln\left( 4 + \ln\left( \frac{t}{2k} \right)\right) + (2k) C^{g} \left( \frac{\ln(\frac{n-1}{\delta})}{2k} \right)  \\   
            & \leq (4k) \ln\left( 4 + \ln\left( \frac{t}{2k} \right)\right) + 2k \left(\frac{\ln(\frac{n-1}{\delta})}{2k}  + \ln \left(\frac{\ln(\frac{n-1}{\delta})}{2k}\right)  + c \right)  \\
            &= \ln\left(e^{2kc}(n-1) \frac{(4 + \ln\left( \frac{t}{2k} \right))^{4k}}{\delta}\right) + 2k \left(\ln \left(\frac{\ln(\frac{n-1}{\delta})}{2k}\right) \right) \\
            &= \ln\left(e^{2kc}(n-1) \frac{(4 + \ln\left( \frac{t}{2k} \right))^{4k}}{\delta} \left(\frac{\ln(\frac{n-1}{\delta})}{2k}\right)^{2k} \right )
        \end{align*}
        And consequently, for sufficiently large $t$, there exists a constant $D$ such that
        $$
            \cdelt \leq \ln \left( \frac{Dt}{\delta} \right).
        $$

    \end{proof}

    \begin{lemma} \label{lem: non-sep continuity}
        For each $\bmu \in \I$ and context probability matrix $\A$, the mapping $\bmu \rightarrow \mathbf{w}^*(\bmu, \A)$ is continuous.
    \end{lemma}
    \begin{proof}
        Let $\M$ be the interior of $\Delta^{n -1}$, i.e., all $\wb \in \Delta^{n-1}$ with strictly positive entries. The mapping $(\bmu', \mathbf{w}') \rightarrow \min_{i \neq i^*(\bmu)} \frac{\Delta_{i}^2}{2}\left( \frac{w_{i^*(\bmu)}w_i}{w_{i^*(\bmu)} + w_i}\right) $ is jointly continuous on $(\I \times \M)$, as it is the minimum of finitely many continuous functions.
        Because constraint set $\Delta^{n-1}$ does not depend on $\bmu'$, and $\wstar{\bmu}$ is unique, Berge's Maximum Theorem implies the continuity of the mapping $\bmu \rightarrow \wstar{\bmu}$.
    \end{proof}

    Since we analyze the expected sample complexity asymptotically, the algorithm’s initialization phase contributes only a constant term with respect to $\delta$.  Note that the waiting time to collect a sample $(X = i$, $Z = j)$ follows a geometric distribution with mean $\frac{1}{\A_{j,i}}$, which is at most $\frac{1}{\amin}$. Consequently, the expected time for initialization is at most $\frac{nk}{\amin}$, which is a constant with respect to $\delta$ and therefore $\limsup_{\delta \rightarrow 0} \frac{\frac{nk}{\amin}}{\logdel} = 0$. For simplicity of the proof, we assume $t=0$ at the end of the initialization phase.

    Note that the set $\mathcal{O}_{\istarmu}$ is open. By combining this fact and Lemma \ref{lem: non-sep continuity}, for every $\epsilon > 0$, there exists a constant $\xi = \xi({\epsilon}, \bmu, \A)$ such that
    \begin{align*}
        \text{if} \infnorm{\bmu' - \bmu} \leq \xi \Rightarrow \bmu' \in \I, ~  i^*(\bmu') = i^*(\bmu), ~ \infnorm{\wstar{\bmu'} - \wstar{\bmu}} \leq \epsilon. 
    \end{align*}    

    Let $\I_{\epsilon} = \{ \bmu' \in \I \big| \infnorm{\bmu' - \bmu} \leq \xi \}$. We now define the following event and show that it occurs with high probability.

    \begin{align*}
        \eteps \triangleq \bigcap_{t = T^{1/4}}^T (\muht \in \I_{\epsilon}).
    \end{align*}
    Note that on the event $\eteps$, Algorithm \ref{algo: non-sep} does not play randomly for $t >  T^{\frac{1}{4}}$, and uses the D-tracking rule.

    \begin{lemma} \label{lem: non-sep event prob}
        There exist two constants $B$ and $C$ such that
        \begin{align*}
            \pr(\E^c_T(\epsilon)) \leq BT \exp{(-CT^{1/8})}.
        \end{align*}
    \end{lemma}
    We provide the proof in Appendix \ref{apd: proof-non-sep event prob}.
    
    We define another event 
    \begin{align*}
            \E'_T(\epsilon) \triangleq \bigcap_{t = T^{1/4}}^T \bigcap_{i \in [n], j \in [k]} \left(\left \vert N_{j,i}(t) - \A_{j, i}N^{X}_i(t) \right \vert \leq t \epsilon \right),
    \end{align*}
    which occurs with high probability.

     \begin{lemma} \label{lem: non-sep event prob2}
        There exist two constants $B'$ and $C'$ such that
        \begin{align*}
            \pr(\E'^c_T(\epsilon)) \leq B'T \exp{(-C'T^{3/8})}.
        \end{align*}
    \end{lemma}

    We provide the proof of this lemma in Appendix \ref{apd: proof-non-sep event prob2}.
        
    Here, we review an important lemma from \cite{track-stop-garivier2016optimal}, which demonstrates the effectiveness of the D-tracking rule.

    \begin{lemma} [\cite{track-stop-garivier2016optimal}, Lemma 20] \label{lem: d-tracking-sqrt}
        For each $\epsilon > 0$, there exists a constant $T_{\epsilon}$, such that for each $T \geq T_{\epsilon}$, on event $\eteps$, we have
        \begin{align*}
            \forall t \geq \sqrt{T}: \infnorm{\frac{\Nb^{X}(t)}{t} - \wstar{\bmu}} \leq 3(n-1)\epsilon.
        \end{align*}
    \end{lemma}

    Define $w^{*}_{j,i} \triangleq \A_{j,i}\wb^{*}_{i}(\bmu, \A)$. By the triangle inequality, we have
        \begin{align*}
            \left \vert \frac{N_{j,i}(t)}{t} - w^{*}_{j,i} \right \vert & \leq \left \vert \frac{N_{j,i}(t)}{t} - \A_{j, i} \frac{N^{X}_{i}(t)}{t} \right \vert + \left \vert \A_{j, i} \frac{N^{X}_{i}(t)}{t}  - w^{*}_{j,i}   \right \vert  \\
            &= \left \vert \frac{N_{j,i}(t)}{t} - \A_{j, i} \frac{N^{X}_{i}(t)}{t} \right \vert + \left \vert \A_{j, i} \frac{N^{X}_{i}(t)}{t} - \A_{j,i}\wb^{*}_{i}(\bmu,\A) \right \vert \\
            &\leq \left \vert \frac{N_{j,i}(t)}{t} - \A_{j, i} \frac{N^{X}_{i}(t)}{t}  \right \vert + \left \vert \frac{N^{X}_{i}(t)}{t}  - \wb^{*}_{i}(\bmu, \A) \right \vert.
        \end{align*}
    where the last inequality holds because $\A_{j,i} < 1$.   

    According to Lemma \ref{lem: d-tracking-sqrt}, on the event $\eteps \cap \etepspr$, we have
    $$
        \left \vert \frac{N_{j,i}(t)}{t} - w^{*}_{j,i} \right \vert \leq \left \vert \frac{N_{j,i}(t)}{t} - \A_{j, i} \frac{N^{X}_{i}(t)}{t}  \right \vert + \left \vert \frac{N^{X}_{i}(t)}{t}  - \wb^{*}_{i}(\bmu, \A) \right \vert \leq (3n - 2)\epsilon
    $$        
    
    Finally, On the event $\eteps \cap \etepspr$, for each $t \geq T^{\frac{1}{4}}$
    \begin{enumerate}
        \item $\muht$ stays close to $\bmu$.
        \item $i^*(\muht) = i^*(\bmu)$, hence Alt$(\muht, \A) = $ Alt$(\bmu, \A)$.
        \item $\frac{N_{j,i}(t)}{t}$ stays close to $\wb^{*}_{j,i}$ for all $i \in [n]$ and $j \in [k]$.
     \end{enumerate}
    
    Combining these properties allows us to derive a lower bound on $\lamht$. Recall the set $\mathcal{O}_i = \left\{ \bmu \in \I \ \big| \A_i^\top \bmu_i > \A_j^\top \bmu_j, \ \forall j \neq i \right\}$. We define the function $g: \mathcal{O}_{\istarmu} \times [0,1] ^{k \times n} \rightarrow \mathbb{R}$ as
    \begin{align*}
        g(\bmu', \mathbf{w}') = \inf_{\mathbf{\lambda} \in \text{Alt}(\bmu, \A)} \sum_{i \in [n], j \in [k]} w'_{j,i} \frac{(\mu'_{j,i} - \mathbf{\lambda}_{j,i})^2}{2},
    \end{align*}
    and the constant 
    \begin{align*}
        \cstareps = \inf_{\substack{\bmu' \in \I_{\epsilon} \\ \mathbf{w}': \vert w'_{j, i} - w^{*}_{j, i}  \vert \leq (3n-2) \epsilon}} g(\bmu', \mathbf{w}').
    \end{align*}

        Note that the mapping $(\mathbf{\lambda}, \bmu', \mathbf{w}') \rightarrow \sum_{i \in [n], j \in [k]} w'_{j,i} \frac{(\mu'_{j,i} - \mathbf{\lambda}_{j,i})^2}{2}$ is jointly continuous and the constraints set Alt$(\bmu, \A)$ is independent of $(\bmu', \mathbf{w}')$. Therefore, applying Berge's maximum theorem \citep{berge1963topological} implies that $g$ is continuous and the constant $\cstareps$ exists.
    
        With these definitions, let $T \geq T_{\epsilon}$ and event $\eteps \cap \etepspr$ holds, then for $t \geq \sqrt{T}$, $\lamht$ is equal to $g(\muht, \frac{N(t)}{t})$, implying that $\lamht \geq t \cstareps$. Using this inequality, for $T \geq T_{\epsilon}$ and on event $\eteps \cap \etepspr$
        \begin{align*}
            \min(\taudel, T) &\leq \sqrt{T} + \sum_{t = \sqrt{T}}^{T} \mathbbm{1}_{(\taudel > t)} \leq \sqrt{T} + \sum_{t = \sqrt{T}}^{T} \mathbbm{1}_{(\lamht \leq \cdelt)} \\
            & \leq \sqrt{T}  + \sum_{t = \sqrt{T}}^{T} \mathbbm{1}_{(t \cstareps \leq \cdelt)} \leq \sqrt{T} + \frac{\hat{c}_T(\delta)}{\cstareps}. 
        \end{align*}
        Defining
        \begin{align*}
            T_{\epsilon}(\delta) = \inf \left\{ T \in \mathbb{N} \bigg \vert \sqrt{T} + \frac{\hat{c}_T(\delta)}{\cstareps} \leq T \right\},
        \end{align*}
        for $T \geq \max(T_{\epsilon}, T_{\epsilon}(\delta))$, event $\eteps \cap \etepspr$ implies $\taudel \leq T$ which shows 
        \begin{align*}
            \pr_{\bmu, \A}(\taudel > T) \leq \pr(\E_T^c(\epsilon) \cup \E'^c_T(\epsilon)) \leq BT \exp(-C T^{1/8}) +  B'T \exp(-C'T^{3/8}).
        \end{align*}
        Now using this inequality and the fact that the initialization phase of our algorithm has a sample complexity of at most $\frac{nk}{\amin}$, we have
        \begin{align}        
            \mathbb{E}_{\bmu, \A}[\taudel] &\leq T_{\epsilon} +  T_{\epsilon}(\delta) + \sum_{T \geq \max(T_{\epsilon}, T_{\epsilon}(\delta))} \pr (\taudel > T ) \nonumber \\  
                           &\leq T_{\epsilon} +  T_{\epsilon}(\delta) + \sum_{T \geq \max(T_{\epsilon}, T_{\epsilon}(\delta))} \pr(\E^c_T(\epsilon) \cup \E'^c_T(\epsilon)) \nonumber \\
                           &\leq T_{\epsilon} +  T_{\epsilon}(\delta) +  \sum_{T=1}^{\infty} BT \exp(-CT^{1/8}) + \sum_{T=1}^{\infty} B'T \exp(-C'T^{3/8})  \label{eq: non-sep expected upper bound}.
        \end{align}
        In this inequality, $\sum_{T=1}^{\infty} BT \exp(-CT^{1/8}) + \sum_{T=1}^{\infty} B'T \exp(-C'T^{3/8}) $ is a constant value, $T_{\epsilon}$ is independent of $\delta$ and we provide an upper bound on the value of $T_{\epsilon}(\delta)$. We define a new constant $C(\alpha)$ for $\alpha > 0$ as
        \begin{align*}
            C(\alpha) = \inf \left\{ T \in \mathbb{N} \bigg| T - \sqrt{T} \geq \frac{T}{1 + \alpha} \right\}.
        \end{align*}
        Then using the upper bound on thresholds in Lemma \ref{lem: thresholds upper bound}, one obtains
        \begin{align*}
            T_{\epsilon}(\delta) \leq E + C(\alpha) + \inf \left\{ T \in \mathbb{N} \bigg| \frac{\ln \left( \frac{DT}{\delta} \right)}{\cstareps} \leq \frac{T}{1 + \alpha} \right\}. 
        \end{align*}
        Using Proposition $8$ in \cite{kaufmann2021mixture}, we have the following
        \begin{align*}
            &\inf \left\{ T \in \mathbb{N} \bigg| \frac{\ln \left( \frac{DT}{\delta} \right)}{\cstareps} \leq \frac{T}{1 + \alpha} \right\} \\
            & \leq \frac{1 + \alpha}{\cstareps} \left[ \ln \left( \frac{(1 + \alpha)D}{\cstareps \delta} \right) + \ln \left( \ln \left( \frac{(1 + \alpha)D}{\cstareps \delta} \right) + \sqrt{2 \ln \left( \frac{(1 + \alpha)D}{\cstareps \delta} \right) - 2} \right) \right]. 
        \end{align*}
        Inserting this inequality in \eqref{eq: non-sep expected upper bound}, we obtain the following for every $\alpha, \epsilon > 0$
        \begin{align*}
            \limsup_{\delta \rightarrow 0} \frac{\mathbb{E}_{\bmu, \A} [\taudel]}{\logdel} \leq \frac{1 + \alpha}{\cstareps}.
        \end{align*}
        Letting both of $\alpha$ and $\epsilon$ go to zero, by continuity of $g$,
        \begin{align*}
            \lim_{\delta \rightarrow 0} \cstareps = T_{NS}^*(\bmu, \A)^{-1}
        \end{align*}
        holds which implies 
        \begin{align*}
            \limsup_{\delta \rightarrow 0} \frac{\mathbb{E}_{\bmu, \A}[\tau_{\delta}]}{\logdel} \leq T_{NS}^*(\bmu, \A).
        \end{align*}     
    \end{proof}

    \subsection{Proof of Lemma \ref{lem: non-sep event prob}} \label{apd: proof-non-sep event prob}
        As a consequence of the forced exploration phase in the D-tracking algorithm, for any round $t$ with $t > n^2$, each arm has been pulled at least $\sqrt{t} - n$ times (refer to Lemma 7 in \cite{track-stop-garivier2016optimal}). Applying the union bound over $t$, we obtain
        \begin{align*}
            \pr(\E_T^c(\epsilon)) &\leq \sum_{t = T^{1/4}}^T \pr(\muht \notin I_{\epsilon})
        \end{align*}

        Similarly, we can apply the union bound for the event $\{\muht \notin I_{\epsilon}\}$ over each arm-context pair
       \begin{align*}
             \pr(\muht \notin I_{\epsilon}) \leq \sumij \pr(\vert \muh_{j,i}(t) - \mu_{j,i} \vert > \xi).
       \end{align*}
       Hence,
        \begin{align*}
            \pr(\E^c_T(\epsilon)) \leq \sum_{t = T^{1/4}}^T \sumij \pr(|\muh_{j,i}(t) - \mu_{j,i}| > \xi).
        \end{align*}
        Recall that $N_{j,i}(t)$ is the number of samples collected from $\pr(Y|X=i, Z=j)$ until round $t$. Then,
        \begin{align*}
            &\pr \left(\vert \muh_{j,i}(t) - \mu_{j,i} \vert > \xi \right) \\
            &= \pr\left(\vert \muh_{j,i}(t) - \mu_{j,i} \vert > \xi \bigg| N_{j,i}(t) \geq \frac{\amin}{2} (\sqrt{t} - n) \right)  \pr\left(N_{j,i}(t) \geq \frac{\amin}{2} (\sqrt{t} - n) \right)\\
            &+  \pr \left(\vert \muh_{j,i}(t) - \mu_{j,i} \vert > \xi \bigg| N_{j,i}(t) < \frac{\amin}{2} (\sqrt{t} - n) \right)  \pr \left( N_{j,i}(t) < \frac{\amin}{2} (\sqrt{t} - n) \right) 
        \end{align*}    
        By upper bounding $\pr\left(N_{j,i}(t) \geq \frac{\amin}{2} (\sqrt{t} - n) \right) \leq 1$ and $\pr \left(\vert \muh_{j,i}(t) - \mu_{j,i} \vert > \xi \bigg| N_{j,i}(t) < \frac{\amin}{2} (\sqrt{t} - n) \right) \leq 1$, we have
        \begin{align*}
            \pr \left(\vert \muh_{j,i}(t) - \mu_{j,i} \vert > \xi \right) &\leq \pr \left( \vert \muh_{j,i}(t) - \mu_{j,i} \vert > \xi \bigg| N_{j,i}(t) \geq \frac{\amin}{2} (\sqrt{t} - n) \right) \\
            &+ \pr \left( N_{j,i}(t) < \frac{\amin}{2} (\sqrt{t} - n) \right).
        \end{align*}
        Both terms can be bounded using Hoeffding's inequality. Since rewards are $1$-sub-Gaussian random variables. First,
        \begin{align*}
            \pr \left(\vert \muh_{j,i}(t) - \mu_{j,i} \vert > \xi \bigg| N_{j,i}(t) \geq \frac{\amin}{2} (\sqrt{t} - n) \right) \leq 2\exp \left( - \frac{\amin (\sqrt{t} - n)\xi^2}{4} \right).
        \end{align*}
        For the second term, note that each context $j$ occurs with probability at least $\amin$ for each action $i$.
        We can write $N_{j,i}(t)$ as the sum of $N^{X}_{i}(t)$ Bernoulli random variables with parameter $\A_{j,i}$, denoted by $S_i$. Then,
        \begin{align*}
            &\pr \left( N_{j,i}(t) < \frac{\amin}{2} (\sqrt{t} - n) \right) \\ 
            &= \pr \left( \sum_{m = 1}^{N^{X}_{i}(t)} S_m < \frac{\amin}{2} (\sqrt{t} - n) \right)
            \overset{S_m \geq 0}{\leq} \pr \left( \sum_{m = 1}^{\sqrt{t} - n} S_m < \frac{\amin}{2} (\sqrt{t} - n)\right)  \\
            &= \pr \left( \sum_{m = 1}^{\sqrt{t} - n} S_m - (\sqrt{t} - n) \A_{j, i} < \frac{\amin}{2} (\sqrt{t} - n) - (\sqrt{t} - n) \A_{j, i}\right)    \\
            &\overset{\A_{j, i} \geq \amin}{\leq} \pr \left( \sum_{m = 1}^{\sqrt{t} - n} S_m - (\sqrt{t} - n) \A_{j, i} < -\frac{\amin}{2} (\sqrt{t} - n)\right) 
        \end{align*}
        Because Bernoulli variables are $\frac{1}{2}$-sub-gaussian, Hoeffding’s inequality implies
        \begin{align*}
            \pr \left( \sum_{m = 1}^{(\sqrt{t} - n)} S_m - (\sqrt{t} - n) \A_{j, i} < -\frac{\amin}{2} (\sqrt{t} - n)\right)  \leq \exp \left( -\frac{\amin^2(\sqrt{t} - n)}{2} \right )
        \end{align*}
         Combining both bounds, we get
        \begin{align*}
             &\pr(\E^c_T(\epsilon)) \leq \sum_{t = T^{1/4}}^T \sumij \pr(|\muh_{j,i}(t) - \mu_{j,i}| > \xi) \\ 
            &\leq \sum_{t = T^{1/4}}^T \sumij  2\exp \left( - \frac{\amin (\sqrt{t} - n)\xi^2}{4} \right) + \exp \left( -\frac{\amin^2(\sqrt{t} - n)}{2}  \right )
        \end{align*}
        Define
        \begin{align*}
            C = \min \left( \frac{\amin \xi^2}{4}, \frac{\amin^2}{2} \right), B = nk \left( 2\exp \left( \frac{n \amin \xi^2}{4} \right) + \exp \left( \frac{n \amin^2}{2} \right) \right).
        \end{align*}
        Then, we have
        \begin{align*}
            \pr(\E_T^c(\epsilon)) \leq \sum_{t = T^{1/4}}^T B \exp (-C\sqrt{t}) \leq \sum_{t = T^{1/4}}^T B \exp (-CT^{1/8}) \leq BT \exp (-CT^{1/8}),
        \end{align*}
        which concludes the proof.

     \subsection{Proof of Lemma \ref{lem: non-sep event prob2}} \label{apd: proof-non-sep event prob2}
        Note that
        \begin{align*}
            N_{j,i}(t) &= \sum_{s = 1}{t} \mathbbm{1}_{(A_t = i, Z_t = j)} = \sum_{m = 1}^{N^{X}_{i}(t)} \mathbbm{1}_{(Z_{t_{m}} = j)},
       \end{align*}
       where $t_m$s denotes the rounds in which the action $i$ is pulled. Indeed, $\mathbbm{1}_{(Z_{t_{m}} = j)}$ is a Bernoulli random variable with parameter $\A_{j,i}$. Since Bernoulli variables are $\frac{1}{2}$-sub-Gaussian, after applying Hoeffding's inequality, we have
        \begin{align*}
            \pr\left ( \left \vert N_{j,i}(t) - \A_{j,i} N^{X}_{i}(t)\right \vert \geq t \epsilon \right) =  \pr\left ( \left \vert \sum_{m = 1}^{N^{X}_{i}(t)} \mathbbm{1}_{(Z_{t_{m}} = j))} - \A_{j,i} N^{X}_{i}(t)\right \vert \geq t \epsilon \right ) &\leq 2 \exp \left (-2\frac{t^2\epsilon^2}{N^{X}_{i}(t)} \right ) \\
            \leq 2 \exp \left (-4\frac{t^2\epsilon^2}{\sqrt{t}} \right ) = 2 \exp \left (-4t^{3/8}\epsilon^2 \right ) 
        \end{align*}
        where the last inequality is achieved by the fact $N^{X}_{i}(t) \geq \sqrt{t} - \frac{n}{2} \geq \frac{\sqrt{t}}{2}$.
       Therefore, applying the union bound on $\E'^c_T(\epsilon)$ leads to
       \begin{align*}
             \pr(\E'^c_T(\epsilon)) &\leq \sum_{t = T^{1/ 4}}^T \sum_{i \in [n], j \in [k]}        \pr\left ( \left \vert N_{j,i}(t) - \A_{j,i} N^{X}_{i}(t)\right \vert \geq t \epsilon \right) \\ 
                            &\leq \sum_{t = T^{1/ 4}}^T \sum_{i \in [n], j \in [k]} 2 \exp \left (-4 t^{1.5}\epsilon^2 \right ) \\
                            &\leq 2nkT \exp \left (-4 T^{3/8}\epsilon^2 \right ).
       \end{align*} 
       Setting $B' = 2nk$ and $C' = 4\epsilon^2$, completes the proof.

\section{Proofs of Section \ref{sec: sep}} \label{apd: sep proofs}
We define $\Is(\A)$ as the set of vectors in $\mathbb{R}^k$ that ensure a unique best arm in the separator setting with context probability matrix $\A$.

\subsection{Mathematical Formulation of G-tracking}\label{apd: form-g-tracking}

At round $t$, the G-tracking rule selects the next target point as the intersection of the ray from $\overline{N}(t) = \frac{N^Z(t)}{t}$ toward $w^*(t)$ with the boundary of the convex hull of the context probability vectors $\A_1, \A_2, \dots, \A_n$.

This ray is given by the set of points of the form:
$$
\overline{N}(t) + \alpha(w^*(t) - \overline{N}(t)), \quad \text{for } \alpha \geq 1.
$$
We want the largest $\alpha \geq 1$ such that this point still lies within $\text{conv}(\A_1, \dots, \A_n)$, which is the set $\left\{ \sum_{i=1}^n \pi_i \A_i \,\middle|\, \pi \in \Delta^{n-1} \right\}$.

This gives rise to the following linear program:

\begin{align*}
    &\max_{\alpha, \pi} \ \alpha \\
    &\text{s.t. } \ \overline{N}(t) + \alpha(w^*(t) - \overline{N}(t)) = \sum_{i=1}^n \pi_i \A_i, \\
    &\sum_{i=1}^n \pi_i = 1, \quad \pi \geq 0, \quad \alpha \geq 1.
\end{align*}

To implement the G-tracking rule, we solved the above linear program, and $\pi_i$ is the desired policy.

\subsection{Proof of Theorem \ref{thm : sep lower}}
    
Similar to the non-separator case, we prove this theorem by explicitly solving the optimization problem \eqref{eq: general_fi} for the separator setting and inserting this solution into the general lower bound formula \eqref{eq: general_lower2}. Note that in contrast to the non-separator case, an action can obtain zero proportion in optimal weights.

\begin{restatable}{lemma}{sepfi}\label{lem: sepfi}
            Consider a bandit instance with a non-separator context and Gaussian reward distribution with unit variance, parameterized by matrices $\bmu$ and $\A$. For any weight vector $\mathbf{w}$ with non-zero indices, the value of $\fiwmu = \inf_{\bmu' \in \C_i} \sum w_i d(P^{\bmu}_{i}, P^{\bmu'}_{i})$ can be explicitly determined as
            \begin{align*}
                \fiwmu = \frac{\Delta_{i}^2}{2 \sum_{j \in [k]} \frac{(\A_{j, \istarmu} - \A_{j, i})^2}{\sum_{l \in [n]}w_{l}\A_{j,l}}}.  
            \end{align*}
\end{restatable}

\begin{proof}
            First we calculate $d(P^{\bmu}_{i}, P^{\bmu'}_{i})$ for two rewards mean vectors $\bmu$ and $\bmu'$ using the assumption of Gaussian rewards with unit variance:
            \begin{align*}
                d(P^{\bmu}_{i}, P^{\bmu'}_{i}) = \sum_{j \in [k]} \A_{j,i} d(\mu_{j}, \mu'_{j}) = \sum_{j \in [k]} \A_{j,i} \frac{(\mu_{j} - \mu'_{j})^2}{2}. 
            \end{align*}
            In the separator setting, $\Is$ indeed is a subset of vectors $\bmu \in \mathbb{R}^k$ that imply a unique best arm. Consequently, the sets $\C_i$ is defined as $\{ \bmu' \in \Is \mid \A_i ^{\top} \bmu' > \A_{\istarmu}^{\top} \bmu' \}$. Then using the continuity of $f_i$ in $\bmu$ and the fact that $\Is$ is dense in $\mathbb{R}^k$, we can write $\fiwmu$ as the solution of the following optimization problem: 
            \begin{align*}
                \fiwmu = &\argmin_{{\bmu' \in \mathbb{R}^{k}}} \sum_{s \in [n]} w_s \sum_{j \in [k]} \A_{j,s}  \frac{(\mu_{j} - \mu'_{j})^2}{2} \\
                & \text{s.t.} \quad \A_i^{\top} \bmu' > \A_{\istarmu}^{\top} \bmu'.
            \end{align*}
            The solution to this problem is equivalent to the solution to the following problem when $\alpha \rightarrow 0$:
            \begin{align*}
                &\argmin_{\bmu' \in \mathbb{R}^{k}} \sum_{s \in [n]} w_s \sum_{j \in [k]} \A_{j,s}  \frac{(\mu_{j} - \mu'_{j})^2}{2} \\
                & \text{s.t.} \quad \A_i^{\top} \bmu' \geq \A_{\istarmu}^{\top} \bmu' + \alpha.
            \end{align*}
            To solve this problem, we define the Lagrangian as 
            \begin{align*}
                L(\bmu',\lambda) \;=\; \sum_{s \in [n]} w_s \sum_{j \in [k]} \A_{j,s}\,\frac{(\mu_{j} - \mu'_{j})^2}{2} \;+\; \lambda\,\Big((\A_{\istarmu} - \A_i)^{\top} \bmu' + \alpha \Big).
            \end{align*}
            Taking the derivative with respect to \( \mu'_j \) and $\lambda$ and setting it to zero with some basic algebra gives the following optimal values
            \begin{align*}
                \mu'_j \;&=\; \mu_j \;-\; \lambda\,\frac{(\A_{j,\istarmu} -  \A_{j,i})}{\displaystyle\sum_{l \in [n]} w_l A_{j,l}}, \\
                \lambda^* &= \frac{\Delta_i + \alpha}{\sum_{j \in [k]} \frac{(\A_{j, \istarmu} - \A_{j, i})^2}{\sum_{l \in [n]}w_{l}\A_{j,l}}}.
            \end{align*}
            Inserting these values into the objective function proves the lemma.
            \begin{align*}
                \fiwmu = \lim_{\alpha \rightarrow 0} \frac{(\Delta_{i} + \alpha)^2}{2 \sum_{j \in [k]} \frac{(\A_{j, \istarmu} - \A_{j, i})^2}{\sum_{l \in [n]}w_{l}\A_{j,l}}} = \frac{\Delta_{i}^2}{2 \sum_{j \in [k]} \frac{(\A_{j, \istarmu} - \A_{j, i})^2}{\sum_{l \in [n]}w_{l}\A_{j,l}}}
            \end{align*}
\end{proof}
Combining the aforementioned lemma with Equation \eqref{eq: general_lower2} completes the proof.

\subsection{Proof of Lemma \ref{lem: sep correctness}}

\sepCorrectness*

\begin{proof}
    As discussed in Appendix \ref{apd: GLR}, our problem can be viewed as a general Identification problem. In separator case, $\bmu$ are $k$-dimensional vectors in $\Is$. Thus, we can redefine $\ocal_i$
    $$
        \mathcal{O}_i = \left\{ \bmu \in \I^s \ \big| \A_i^\top \bmu > \A_j^\top \bmu, \ \forall j \neq i \right\}
    $$
        
    The proof then follows directly from Proposition~15 of \cite{kaufmann2021mixture}, which implies that the GLR stopping rule with the sequential threshold in \eqref{eq: sep threshold} is \(\delta\)-correct. Note that in Proposition~15 of \cite{kaufmann2021mixture}, the threshold is given for one-parameter exponential families. Since our reward distributions are Gaussian, we employ a tighter threshold that the authors introduce in their paper.
\end{proof}

\subsection{Proof of Theorem \ref{thm : sep upper}}
    Recall that $\Is(\A)$ denotes the set of vectors in $\mathbb{R}^k$ which implies a unique best arm in the separator setting with context probability matrix $\A$. For a point $p \in \mathbb{R}^k$ and a set $C \in \mathbb{R}^k$, we use $\dls{p}{C}$ to denote the distance of $p$ to $C$ with respect to $L^2$ norm, more precisely 
    \begin{align*}
        \dls{p}{C} = \inf_{c \in C} \lsnorm{p - c}.
    \end{align*}
    Similar to the non-separator case, we need to present a few helper lemmas. 

    \begin{lemma} \label{lem: sep continuouity}
        For each $\bmu \in \Is$ and context probability matrix $\A$, the set $\wzstar{\bmu}$ is a convex set and the set-valued mapping $\bmu \rightarrow \wzstar{\bmu}$ is upper hemicontinuous. 
    \end{lemma}
    \begin{proof}
        Recall that $\wzstar{\bmu}$ is the set of solutions that maximize
        \begin{align*}
          F(\wb, \bmu) \triangleq \inf_{\bmu' \in \text{Alt}(\bmu, \A)} \sum_{j \in [k]} w_j \frac{(\bmu_{j} - \bmu'_{j})^2}{2}
        \end{align*}
        where $\wb \in \ch(\A)$. Since this function is the infimum over linear functions in $\wb$, thus, $F$ is concave with respect to $\wb$. Consequently, if $\wb_1$ and $\wb_2$ are both maximizers of $F$, then any convex combination of them is also a maximizer, implying that $\wzstar{\bmu}$ is a convex set. 
        
        As shown in Lemma \ref{lem: sepfi}, $F(\wb)$ can be written as  
        \begin{align*}
             \min_{i \neq i^*(\bmu)} \frac{\Delta_{i}^2}{2 \sum_{j \in [k]} \frac{(\A_{j, \istarmu} - \A_{j, i})^2}{w_{z,j}}},
        \end{align*}
        which is jointly continuous on $\ch(\A) \times I^{s}$. By applying Theorem 22 in \cite{mutiple-correct-answers-degenne2019pure}, we conclude that the mapping $\bmu \rightarrow \wzstar{\bmu}$ is upper hemicontinuous.
    \end{proof}

    As we analyze the sample complexity asymptotically, the initialization phase of the algorithm contributes only a constant term. This is because, for each context $i$, the expected waiting time to collect a sample with $Z = i$ is at most $\frac{1}{\amin}$, as the probability of observing context $i$ at each round is at least $\amin$. Thus, the expected number of initialization steps is at most $\frac{k}{\amin}$, which is a constant, and $\limsup_{\delta \rightarrow 0} \frac{\frac{k}{\amin}}{\logdel} = 0$. For simplicity, we set $t=0$ at the end of the initialization phase.

    For a non-negative real number $\alpha$, we define the $L^2(\alpha)$-approximation set of $\wzstar{\bmu}$ as 
    \begin{align*}
        \wepsstar{\alpha} = \left\{p \in \Delta^{k-1} \mid \dls{p}{\wzstar{\bmu}} \leq \alpha \right \}.
    \end{align*}
    Based on Lemma \ref{lem: sep continuouity}, the set $\wzstar{\bmu}$ is convex which shows that for each $\alpha$, the set $\wepsstar{\alpha}$ is also convex. \footnote{This follows directly from the definition of convexity, as any convex combination of two points in $\wepsstar{\alpha}$ remains within $\wepsstar{\alpha}$.}

    By Lemma \ref{lem: sep continuouity} and the fact that the set of vectors $\bmu'$ which has the same best arm as $\bmu$ is open, for each real number $\epsilon > 0$, there exist a constant $\xi = \xi(\epsilon, \bmu, \A)$ such that 
    \begin{align*}
        \text{if} \infnorm{\bmu' - \bmu} \leq \xi \Rightarrow \bmu' \in \Is(\A), ~  i^*(\bmu') = i^*(\bmu), \forall \mathbf{w}' \in \wzstar{\bmu'}: \mathbf{w}' \in \wepsstar{\epsilon}. 
    \end{align*}  

    Let $\I_{\epsilon} = \{ \bmu' \in \I(\A) \big| \infnorm{\bmu' - \bmu} \leq \xi \}$. For each $T$, we define two events that are likely to happen as
    \begin{align*}
        &\eteps \triangleq \bigcap_{t = T^{1/4}}^T (\muht \in \I_{\epsilon}), \\ 
        &\etepspr \triangleq \bigcap_{t = T^{1/4}}^T \left( \lsnorm{{\Nb^{Z}(t)} - {\sum_{i \in [t]} \pol(i)}} < t^{\frac34} \right).
    \end{align*}
    Recall that $\Nb^{Z}(t)$ shows the number of samples collected from contexts until round $t$ and $\pol(i)$ shows the policy on the context values that is chosen by G-tracking at round $i$. The following lemma shows these events occur with high probability.

    \begin{lemma} \label{lem: sep event prob}
        The exist two constants $B$ and $C$ such that
        \begin{align*}
            \pr((\eteps \cap \etepspr)^c) \leq BT \exp(-CT^{1/8}).
        \end{align*}
    \end{lemma}

    We defer the proof of Lemma \ref{lem: sep event prob} to Appendix \ref{subsec: sep event prob}. For now, we present a lemma that establishes the convergence of the G-tracking rule.

    \begin{lemma} \label{lem: G-tracking}
        For each $\epsilon > 0$, there exist a constant $T_{\epsilon}$ such that for each $T \geq T_{\epsilon}$, on event $\eteps \cap \etepspr$, we have 
        \begin{align*}
            \forall t \geq \sqrt{T} : \dls{\frac{\Nb^{Z}(t)}{t}} {\wzstar{\bmu}} \leq 5 \epsilon. 
        \end{align*}
    \end{lemma}
    \begin{proof}      
        First, we introduce some geometric notations. For points $A, B, C, D \in \mathbb{R}^k$, let $\overline{AB}$ show the length of the segment connecting $A$ to $B$, $AB$ show the line passing through $A$ and $B$, and $AB \parallel CD$ show that vectors $\overrightarrow{AB}$ and $\overrightarrow{CD}$ are parallel. At each round $t$, let $\overline{\pol}(t) = \frac{\sum_{i \in [t]}\pol(i)}{t}, \overline{\Nb}(t) = \frac{\Nb^{Z}(t)}{t}$, and
        \begin{align*}
            Y_t = \dls{\overline{\Nb}(t)}{\wepsstar{2\epsilon}}, \quad X_t = \dls{\overline{\pol}(t)}{\wepsstar{2\epsilon}}. 
        \end{align*}

        \begin{figure}[H]
            \centering
            \includegraphics[width=0.9\linewidth]{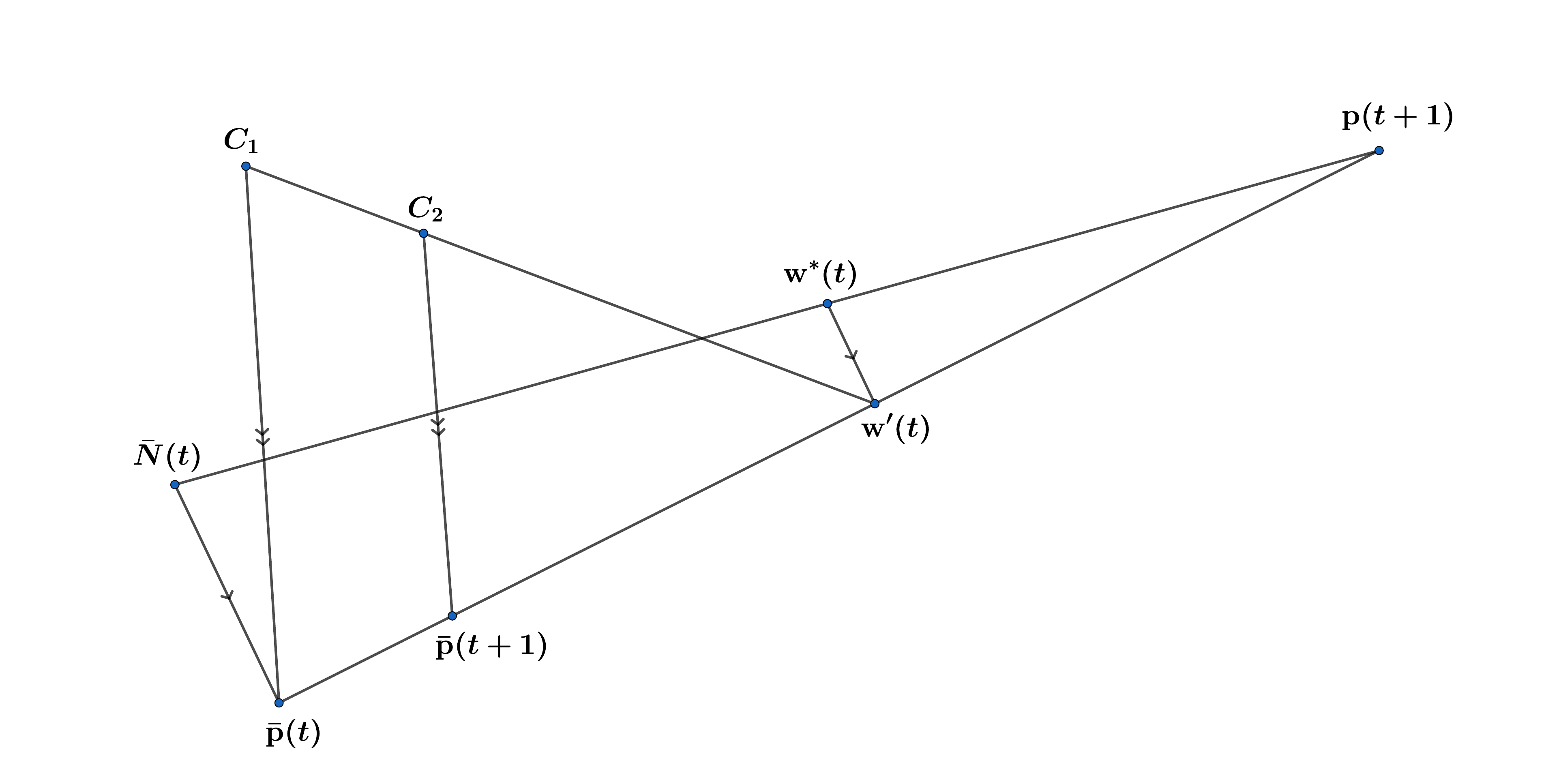}
            \caption{Illustration of the relative positions of points in the proof of Lemma \ref{lem: G-tracking}, where all points lie in $\Delta^{k-1}$.}
            \label{fig: G-lemma}
        \end{figure}
    
        Moreover, let $\mathbf{w}'(t)$ be the unique point on the line $\overline{\pol}(t) \pol(t+1)$ such that $\overline{\Nb}(t)\overline{\pol}(t) \parallel \wb^{*}(t)\mathbf{w}'(t)$. $C_1$ denotes the closest point in $\wepsstar{2 \epsilon}$ to $\overline{\pol}(t)$ (note that $\wepsstar{2 \epsilon}$ is closed) and $C_2$ denotes the unique point on $C_1 \mathbf{w}'(t)$ such that $C_2 \overline{\pol}(t+1) \parallel C_1 \overline{\pol}(t)$. Figure \ref{fig: G-lemma} provides a complete illustration of all these points.
    
        We set $T_{\epsilon} = \epsilon^{-16}$. By this choice of $T_{\epsilon}$, if $T \geq T_{\epsilon}$, for each $t \geq T^{\frac14}$ we have  
        \begin{align} \label{eq: g-lemma-eq3}
            \overline{\overline{\pol}(t) \overline{\pol}(t+1)} \leq \epsilon \quad \text{ and } \quad \overline{\overline{\Nb}(t) \overline{\pol}(t)} \leq \epsilon. 
        \end{align}
        
        The reason is that 
        \begin{align*}
            \overline{\overline{\pol}(t) \overline{\pol}(t+1)} = \frac{1}{t+1} \overline{\overline{\pol}(t) \pol(t+1)} \leq \frac{\sqrt{2}}{t+1} \leq \frac{\sqrt{2}}{T^{1/4}} \leq \epsilon.  
        \end{align*}
        The first inequality holds because the points $\overline{\pol}(t)$ and $\pol(t+1)$ are on $\Delta^{k-1}$ and the distance of any pair of points of the simplex is at most $\sqrt{2}$. For the second property, note that as $\etepspr$ holds, $\overline{\overline{\Nb}(t) \overline{\pol}(t)} < t^{-1/4} < T^{-1/16} \leq \epsilon$.  
    
        Based on the definition, $C_1 \in \wepsstar{2 \epsilon}$ and as $\eteps$ holds, $\wb^{*}(t) \in \wepsstar{\epsilon}$. Using Thales's theorem in triangle $\overline{\Nb}(t) \overline{\pol}(t) \pol(t+1)$ , we obtain $\overline{\wb^{*}(t) \mathbf{w}'(t)} \leq \overline{\overline{\Nb}(t) \overline{\pol}(t)} \leq \epsilon$ which means that $\mathbf{w}'(t) \in \wepsstar{2 \epsilon}$. By convexity of the set $\wepsstar{2 \epsilon}$, we have $C_2 \in \wepsstar{2 \epsilon}$. 
    
        We prove the following two statements for each $t \geq T^{1/4}$: 
        
        \begin{align}\label{eq: g-lemma-eq1}
            \text{if }  X_t > \epsilon \Longrightarrow X_{t+1} \leq \frac{t}{t+1} X_t, 
        \end{align}
        \begin{align} \label{eq: g-lemma-eq2}
            \text{if } X_t \leq \epsilon \Longrightarrow \forall t' > t: X_{t'} \leq 2 \epsilon. 
        \end{align} 
    
        To prove \eqref{eq: g-lemma-eq1}, note that as $\mathbf{w}'(t) \in \wepsstar{2 \epsilon}$, if $X_t > \epsilon$ then $\overline{\overline{\pol}(t) \mathbf{w}'(t)} > \epsilon$. This shows that based on \eqref{eq: g-lemma-eq3}, $\overline{\pol}(t+1)$ lines on the segment between $\overline{\pol}(t) \mathbf{w}'(t)$. Then using the Thales's theorem in triangle $C_1 \overline{\pol}(t) \mathbf{w}'(t)$, we have 
        \begin{align*}
            X_{t+1} &= \dls{\overline{\pol}(t+1)}{\wepsstar{2 \epsilon}} \leq \overline{C_2 \overline{\pol}(t+1)} = \frac{\overline{\mathbf{w}'(t) \overline{\pol}(t+1)}}{\overline{\mathbf{w}'(t) \overline{\pol}(t)}} \overline{C_1 \overline{\pol}(t)} \\
            &\leq \frac{\overline{\pol(t+1) \overline{\pol}(t+1)}}{\overline{\pol(t+1) \overline{\pol}(t)}} \overline{C_1 \overline{\pol}(t)} = \frac{t}{t+1} \; \dls{\overline{\pol}(t)}{\wepsstar{2 \epsilon}} = \frac{t}{t+1} X_t.
        \end{align*}
        For the second inequality, we used the property that for positive numbers $a,b,c$ with $a < b$, we have $\frac{a}{b} < \frac{a+c}{b+c}$.
    
        We prove \eqref{eq: g-lemma-eq2} by contradiction. Assume $X_t \leq \epsilon$ and there exist a $t' > t$ such that $X_{t'} > 2 \epsilon$ and $t'$ is the smallest number with this property. As $X_{t'} > 2 \epsilon$, then \eqref{eq: g-lemma-eq3} implies $X_{t'-1} > \epsilon$. Using \eqref{eq: g-lemma-eq1}, we have $X_{t'} \leq \frac{t'}{t'+1} X_{t'-1}$ which implies $X_{t'-1} > 2 \epsilon$ which is a contradiction and proves \eqref{eq: g-lemma-eq2}. 
    
        Now we introduce the time $\tau$ as 
        \begin{align*}
            \tau \triangleq \inf \left\{ t > T^{1/4} \mid X_t \leq \epsilon \right \}. 
        \end{align*}
    
        Note that as a result of \ref{eq: g-lemma-eq1}, $\tau$ is finite. The aim is to prove $\tau \leq \sqrt{T}$. To show this assume $\tau > \sqrt{T}$. In this case, for all $t \in [\lceil T^{1/4} \rceil , \lfloor \sqrt{T} \rfloor]$, $X_t > \epsilon$ which means that $X_{t+1} \leq \frac{t}{t+1} X_t$. Then 
        \begin{align*}
            X_{\lfloor \sqrt{T} \rfloor} &\leq  X_{\lceil T^{1/4} \rceil} \times \frac{\lceil T^{1/4} \rceil}{\lceil T^{1/4} \rceil + 1} \times \frac{\lceil T^{1/4} \rceil + 1}{\lceil T^{1/4} \rceil + 2} \times \cdots \times \frac{\lfloor \sqrt{T} \rfloor - 1}{\lfloor \sqrt{T} \rfloor} \\
            &= X_{\lceil T^{1/4} \rceil} \frac{\lceil T^{1/4} \rceil}{\lfloor \sqrt{T} \rfloor} \leq \sqrt{2} \frac{\lceil T^{1/4} \rceil}{\lfloor \sqrt{T} \rfloor} \leq \epsilon,
        \end{align*}
        where we used $X_{\lceil T^{1/4} \rceil} \leq \sqrt{2}$ and $T \geq T_{\epsilon}$. The above inequality shows a contradiction which implies $\tau \leq \sqrt{T}$. Then, we have
        \begin{align*}
            \tau \leq \sqrt{T} \Longrightarrow \forall t> \sqrt{T}: X_t \leq 2 \epsilon \xLongrightarrow{\eqref{eq: g-lemma-eq3}}  Y_t \leq 3 \epsilon \Longrightarrow \dls{\frac{\Nb^{Z}(t)}{t}} {\wzstar{\bmu}} \leq 5 \epsilon. 
        \end{align*}
    
    \end{proof}

    On the good event $\eteps \cap \etepspr$, for each $t \geq \sqrt{T}$, these three properties hold: (i) $\muht$ is close to $\bmu$, (ii) $i^*(\muht) = i^*(\bmu)$ which implies Alt$(\muht, \A) = $ Alt$(\bmu, \A)$, and (iii) the points $\Nb^{Z}(t)/t$ to the set of solutions $\wzstar{\bmu}$ stay close.

    Similar to the non-separator setting, by combining these properties we can derive a lower bound on $\lamht$. Recall the set $\mathcal{O}_i = \left\{ \bmu \in \I^s \ \big| \A_i^\top \bmu > \A_j^\top \bmu, \ \forall j \neq i \right\}$. We define the function $g: \mathcal{O}_{\istarmu} \times [0,1] ^ k \rightarrow \mathbb{R}$ as
    \begin{align*}
        g(\bmu', \mathbf{w}') = \inf_{\mathbf{\lambda} \in \text{Alt}(\bmu, \A)} \sum_{j \in [k]} \frac{w'_j}{2}(\mu'_i - \mathbf{\lambda}_i)^2,
    \end{align*}
    and the constant 
    \begin{align*}
        \cstareps = \inf_{\substack{\bmu' \in \I^s_{\epsilon} \\ \mathbf{w}': \dls{\mathbf{w}'} {\wzstar{\bmu}} \leq 5 \epsilon}} g(\bmu', \mathbf{w}').
    \end{align*}

    Note that the mapping $(\mathbf{\lambda}, \bmu', \mathbf{w}') \rightarrow \sum_{j \in [k]} \frac{w'_j}{2}(\mu'_i - \mathbf{\lambda}_i)^2$ is jointly continuous and the constraints set Alt$(\bmu, \A)$ is independent of $(\bmu', \mathbf{w}')$. Then Berge's maximum theorem \citep{berge1963topological} implies that $g$ is continuous.
    
    With these definitions, let $T \geq T_{\epsilon}$ and the event $\eteps \cap \etepspr$ holds. Then, for $t \geq \sqrt{T}$, $\lamht$ is equal to $g(\muht, \frac{\Nb^{Z}(t)}{t})$, and the three mentioned properties imply that $\lamht \geq t \cstareps$. 

    The rest of the proof follows exactly the same as the non-separator setting.

    \subsection{Proof of Lemma \ref{lem: sep event prob}} \label{subsec: sep event prob}
    We first bound both $\pr(\E^c_T(\epsilon))$ and $\pr(\E'^c_T(\epsilon))$ separately, and then we apply the union bound on $\pr((\eteps \cap \etepspr)^c)$. 

    \textbf{Bounding $\E^c_T(\epsilon)$.} Note that according to the union bound, we have
    \begin{align*}
        \pr(\E^c_T(\epsilon)) \leq \sum_{t = T^{1/4}}^T \sum_{j \in [k]} \pr(|\muh_{j}(t) - \mu_{j}| > \xi).
    \end{align*}
    Now, note that we have
    \begin{align*}
        \pr(|\muh_{j}(t) - \mu_{j}| > \xi) &= \pr(\vert \muh_{j}(t) - \mu_{j} \vert > \xi \mid N^{Z}_j(t) < \amin \sqrt{t}) \pr(N^{Z}_j(t) < \amin \sqrt{t})  \\ 
                                           &+  \pr(\vert \muh_{j}(t) - \mu_{j} \vert > \xi \mid N^{Z}_j(t) \geq \amin \sqrt{t} )  \pr(N^{Z}_j(t) \geq \amin \sqrt{t}).         
    \end{align*}
    Since $\pr(|\muh_{j}(t) - \mu_{j}| > \xi \mid N^{Z}_j(t) < \amin \sqrt{t})$ and $\pr(N^{Z}_j(t) \geq \amin \sqrt{t})$ are less than $1$, we have
    $$
        \pr(\vert \muh_{j}(t) - \mu_{j} \vert > \xi) \leq \pr(N^{Z}_j(t) < \amin \sqrt{t})  + \pr(\lvert \muh_{j}(t) - \mu_{j} \rvert > \xi \mid N^{Z}_j(t) \geq \amin \sqrt{t} ) 
    $$

    Note that $N^{Z}_j(t)$ can be written as 
    $$
        N^{Z}_j(t) = \sum_{s = 1}^{t} \mathbbm{1}_{(z_s = j)}.
    $$
    Therefore, if $W_s$ denotes $\mathbbm{1}_{(z_s = j)}$, it is a Bernoulli random variable with parameter greater than $\amin$. By using Hoeffding's inequality, we have
    \begin{align*}
        &\pr \left( N^{Z}_j(t) < \amin \sqrt{t} \right) = \pr \left( \sum_{s = 1}^{t} W_s < \amin \sqrt{t} \right) \\ 
        &\leq \pr \left ( \sum_{s = 1}^{t} W_s - \mathbb{E}[N^{Z}_{j}(t)] < \amin \sqrt{t} - t\amin \right) \\
        & \leq \exp \left ( \frac{-2\amin^2(t - \sqrt{t})^2}{t} \right) = \exp \left (-2\amin^2(\sqrt{t} - 1)^2 \right) \\
        & \leq \exp \left (-\amin^2 t \right)
    \end{align*}
    The last inequality is correct for sufficiently large $t$.

    Similarly, by applying Hoeffding's inequality, we have
    $$
        \pr(\lvert \muh_{j}(t) - \mu_{j} \rvert > \xi \mid N^{Z}_j(t) \geq \amin \sqrt{t} )  \leq 2 \exp \left (-\frac{\xi^2 \amin \sqrt{t}}{2} \right).
    $$
    Therefore,
    $$
          \pr(|\muh_{j}(t) - \mu_{j}| > \xi) \leq  \exp \left (-\amin^2 t \right) +  2 \exp \left (-\frac{\xi^2 \amin \sqrt{t}}{2} \right), 
    $$
    and then,
    $$
      \pr(\E^c_T(\epsilon)) \leq \sum_{t = T^{1/4}}^T \sum_{j \in [k]} \pr(|\muh_{j}(t) - \mu_{j}| > \xi) \leq B_1T \exp (-C_1 T^{1/8}) 
    $$
    where $B_1$ and $C_1$ are constants and the functions of $\xi,~\amin,$ and $k$.

    \textbf{Bounding $\E^{'c}_T(\epsilon)$.} Similar to the previous part, we use the union bound over $t$, then we have
    \begin{align}\label{eq: sep-event-union}
        \pr(\E^{'c}_T(\epsilon)) &\leq \sum_{t = T^{\frac14}}^{T} \pr\left( \lsnorm{{\Nb^{Z}(t)} - {\sum_{i \in [t]} \pol(i)}} > t^{\frac34} \right) \nonumber \\
        &\leq \sum_{t = T^{\frac14}}^{T} \pr\left( \sqrt{k} \infnorm{{\Nb^{Z}(t)} - {\sum_{i \in [t]} \pol(i)}} > t^{\frac34} \right) \\
        & = \sum_{t = T^{\frac14}}^{T} \pr\left( \infnorm{{\Nb^{Z}(t)} - {\sum_{i \in [t]} \pol(i)}} > \frac{t^{\frac34}}{k^{\frac12}} \right),
    \end{align}
    where we also used the inequality $\lsnorm{{\Nb^{Z}(t)} - {\sum_{i \in [t]} \pol(i)}} \leq  \sqrt{k} \infnorm{{\Nb^{Z}(t)} - {\sum_{i \in [t]} \pol(i)}}$.
     
    Again, after applying the union bound over $j \in [k]$, we have
    \begin{align}\label{eq: sep-event-union2}
        \pr\left( \infnorm{{\Nb^{Z}(t)} - {\sum_{i \in [t]} \pol(i)}} > \frac{t^{\frac34}}{k^{\frac12}} \right)  \leq \sum_{j = 1}^{k} \pr\left( \left \lvert N^{Z}_{j}(t) - {\sum_{i \in [t]} p_{j}(i)} \right \rvert > \frac{t^{\frac34}}{k^{\frac12}} \right)         
    \end{align}
    where $p_{j}(i)$ denotes the j-th element of the vector $\pol(i)$. Now, we can write $N^{Z}_{j}(t)$ as the sum of $t$ random variables as follows.
    $$
           N^{Z}_j(t) = \sum_{s = 1}^{t} \mathbbm{1}_{(z_s = j)} = \sum_{i = 1}^{t} W_i
    $$
    Note that according to Algorithm \ref{algo: sep}, $W_i$ is a Bernoulli random variable with parameter $\pol(i)_{i}(s)$. Therefore, we have
    $$
        \pr\left( \left \lvert N^Z_{j}(t) - {\sum_{i \in [t]} p_{j}(i)} \right \rvert > \frac{t^{\frac34}}{k^{\frac12}} \right)  = \pr\left( \left \lvert \sum_{i = 1}^{t} W_i - {\sum_{i \in [t]} p_{j}(i)} \right \rvert > \frac{t^{\frac34}}{k^{\frac12}} \right),
    $$
    Then, by applying Hoeffding's inequality, we obtain
    $$
        \pr\left( \left \lvert \sum_{i = 1}^{t} W_i - \sum_{i \in [t]} p_{j}(i) \right \rvert > \frac{t^{\frac34}}{k^{\frac12}} \right) \leq 2 \exp \left( - \frac{2t^{\frac32}}{kt} \right)  =   2 \exp \left( - \frac{2t^{\frac12}}{k} \right) 
    $$
    Finally, by combining the above inequality and Equations \eqref{eq: sep-event-union} and \eqref{eq: sep-event-union2}, we have
    $$
        \pr(\E^{'c}_T(\epsilon)) \leq \sum_{t = T^{\frac14}}^{T} 2k \exp \left( - \frac{2t^{\frac12}}{k} \right)
                                \leq TB_2 \exp \left( - \frac{2T^{\frac18}}{k} \right)
    $$
    where $B_2 = 2k$ and $C_2 = \frac{2}{k}$.

    All in all, by combining the derived upper bounds for $\pr (\E^{c}_T(\epsilon))$ and $\pr(\E^{'c}_T(\epsilon))$, we have
    $$
        \pr((\eteps \cap \etepspr)^c) \leq \pr (\E^{c}_T(\epsilon)) + \pr(\E^{'c}_T(\epsilon)) \leq BT \exp (-CT^{\frac18}),
    $$
    where $B = 2\max (B_1, B_2)$ and $C = \min (C_1, C_2)$.

\section{Discussion on Solving the Optimization Problem of Lower bound} \label{apd: opt-solving}

The optimization problem \eqref{eq: middle non-sep LB} can be solved efficiently and with arbitrary precision. Notably, the same optimization problem arises in the classic BAI problem with Gaussian rewards \citep{bai-gaussian-barrier2022non}. Although the definitions of $\Delta_i$ differ between these two settings—and in this paper, these gaps depend on the context probability matrix—the optimization problem itself remains similar, allowing the use of similar techniques to solve it. 
    \begin{restatable}[\cite{bai-gaussian-barrier2022non}]{lemma}{nonsepw} \label{lem: non-sep w1}
        Consider the  equation
        \begin{equation} \label{eq: non-sep w1}
            \sum_{\substack{i \in [n] \\ i \neq \istarmu}} \frac{1}{(w \Delta_i^2 - 1)^2} = 1 , 
        \end{equation}
         where $\Delta_i > 0$. This equation has a unique solution $w^*$ in the interval $\left[\frac{2}{\Delta_{min}^2}, \frac{1 + \sqrt{n-1}}{\Delta_{min}^2}\right]$. Define $\ub = \left(\frac{w^*}{w^*\Delta_1^2 -1}, \frac{w^*}{w^*\Delta_2^2 -1}, \ldots , w^*, \dots, \frac{w^*}{w^*\Delta_n^2 -1}\right)$, where $\istarmu$-th element is $w^*$. The normalized vector $\frac{\ub}{\norm*{\ub}_{1}} \in \Delta^{n-1}$ is the unique solution of Equation \eqref{eq: middle non-sep LB}.
    \end{restatable}

    As discussed in \cite{bai-gaussian-barrier2022non}, the left-hand side of Equation \eqref{eq: non-sep w1} is strictly decreasing over the interval $\left[\frac{2}{\Delta_{min}^2}, \frac{1 + \sqrt{n-1}}{\Delta_{min}^2}\right]$. Consequently, simple numerical methods, such as binary search, can be employed to approximate the solution efficiently. Leveraging this property, we can efficiently compute the weights for the optimization problem in Equation \eqref{eq: middle non-sep LB}. This capability enables the design of an efficient learning algorithm, as discussed later.

\section{Discussion on the Case of Unknown Context Probability Matrix} \label{apd: unknown context}

    The assumption that the context probability matrix $\A$ is known does not always hold in practice. A natural question, then, is whether the problem becomes fundamentally harder, in terms of minimum sample complexity, when $\A$ is unknown. For instance, in the setting of contextual bandits, \cite{rl-difference-estimation-narang2024sample} showed that the sample complexity of BAI with unknown context distributions remains of the same order as in the case with known distributions. This indicates that, in that setting, not knowing the context distribution does not significantly increase the difficulty of the problem.
    
    However, the following lemma shows that this is not the case in the separator post-action context setting. In this setting, the absence of knowledge about $\A$ can arbitrarily increase the sample complexity lower bound, underscoring a fundamental difference from the contextual bandit setting.

    \begin{lemma} \label{lem: unknown-A}
        Consider the instance with $n=k=3$ and parameters 
        \begin{align} \label{eq: unknown-A-instance}
             \A = \begin{bmatrix}
                    \epsilon & 2 \epsilon & \frac12 \\
                    1 -  \epsilon & 1 - 2 \epsilon & \frac12
                \end{bmatrix}, \quad ~    
            \bmu = \begin{bmatrix}
                0 \\
                1
         \end{bmatrix},       
        \end{align}
        where $\epsilon$ is a small positive real number. Then, we have
        \begin{enumerate}
            \item On instance \eqref{eq: unknown-A-instance} with known $\A$, the characteristic $T_{S}^*(\bmu, \A)$ defined in \eqref{eq: sep LB1} is equal to $8$.
            \item For any $\delta$-correct algorithm $(P_t, \tau, \hat{i})$ interacting with instance \eqref{eq: unknown-A-instance} with unknown $\A$, $\mathbb{E}_{\bmu, \A}[\tau] \geq \frac{1}{4\epsilon} d_B(\delta, 1- \delta)$. 
        \end{enumerate}
    \end{lemma}

    \begin{proof}
        To prove the first part, recall that
        \begin{align*}
            T_{S}^*(\bmu, \A)^{-1} = \sup_{\mathbf{w} \in \Delta^{2}} \min_{i \neq i^*(\bmu, \A)} \frac{\Delta_{i}^2}{2 \sum_{j \in [k]} \frac{(\A_{j, i^*(\bmu, \A)} - \A_{j, i})^2}{\sum_{l \in [n]} w_l \A_{j, l}}}. 
        \end{align*}
        Note that $i^*(\bmu, \A) = 1$, $\Delta_2 = \epsilon$, and $\Delta_3 = \frac12 - \epsilon$. Moreover, for each $i \neq i^*(\bmu, \A), j \in [k]$, $\Delta_i^2 = (\A_{j, i^*(\bmu, \A)} - \A_{j, i})^2$. After inserting and simplification, we have 
        $$
        T_{S}^*(\bmu, \A)^{-1} = \sup_{\mathbf{w} \in \Delta^{2}}  \frac{1}{2 \left( \frac{1}{w_1 \epsilon + 2w_2 \epsilon + \frac{w_3}{2}} + \frac{1}{w_1 (1 - \epsilon) + w_2 (1 - 2\epsilon) + \frac{w_3}{2}} \right)} = \sup_{\mathbf{w} \in \Delta^{2}} \frac{w_{z,1} w_{z,2}}{2(w_{z,1} + w_{z,2})},
        $$
        where $w_{z,1} = w_1 \epsilon + 2w_2 \epsilon + \frac{w_3}{2}$ and $w_{z,2} = w_1 (1 - \epsilon) + w_2 (1 - 2\epsilon) + \frac{w_3}{2}$. Then setting $\wb = (0,0,1)$ implies $w_{z,1} = w_{z,2} = \frac12$ , then
        $$
        T_{S}^*(\bmu, \A)^{-1} \geq \frac18.
        $$

        On the other hand, since $w_{z,1} + w_{z,2} = 1$, we have $w_{z,1} w_{z,2} \leq \frac14$ and $\frac{w_{z,1} w_{z,2}}{2(w_{z,1} + w_{z,2})} \leq \frac{1}{8}$, which implies 
        $$
        T_{S}^*(\bmu, \A)^{-1} \leq \frac18.
        $$
        Therefore, we conclude that $T_{S}^*(\bmu, \A) = 8$.

        To prove the second part, note that Equation \eqref{eq: general_lower1}, by considering the instance $(\bmu', \A') \in \text{Alt}(\bmu, \A)$ with parameters 
        \begin{align*} 
        \bmu' = \bmu, \quad ~
            \A' = \begin{bmatrix}
                3 \epsilon & 2 \epsilon & \frac12 \\
                1 -  3 \epsilon & 1 - 2 \epsilon & \frac12
            \end{bmatrix},
        \end{align*}
        implies that for any $\delta$-correct algorithm $(P_t, \tau, \hat{i})$, we have $\mathbb{E}_{\bmu, \A}[\tau] \geq T^*(\mu, \A) d_B(\delta, 1- \delta)$, where
        $$
        T^*(\bmu, \A)^{-1} \leq \sup_{\mathbf{w} \in \Delta^{2}} \sum_{i \in \{1,2,3\}} w_i d\left( P_i^{\bmu, \A}, P_i^{\bmu',\A'} \right).
        $$
        Since these two instances only differ in the first column of their context probability matrix, then
        $$
        \sup_{\mathbf{w} \in \Delta^{2}} \sum_{i \in \{1,2,3\}} w_i d\left( P_i^{\bmu, \A}, P_i^{\bmu',\A'} \right) = \sup_{\mathbf{w} \in \Delta^{2}} w_1 d(\A_1, \A'_1),
        $$
        where $d(\A_1, \A'_1)$ denotes the KL divergence between two binary variables and is equal to
        $$
        d(\A_1, \A'_1) = \epsilon \ln\left( \frac{\epsilon}{3 \epsilon}\right) + (1 - \epsilon) \ln\left( \frac{1 - \epsilon}{1 - 3 \epsilon}\right) = - \epsilon \ln(3) + (1- \epsilon) \ln\left( \frac{1 - \epsilon}{1 - 3 \epsilon}\right).
        $$
        We can upper bound $d(\A_1, \A'_1)$ as
        \begin{align*}
            &\ln\left( \frac{1 - \epsilon}{1 - 3 \epsilon}\right) \leq \frac{1 - \epsilon}{1 - 3 \epsilon} - 1 = \frac{2 \epsilon}{1 - 3\epsilon} \\
            &\xLongrightarrow{1-\epsilon < 2(1 - 3\epsilon)} d(\A_1, \A'_1) \leq -\epsilon \ln(3) + 4 \epsilon < 4 \epsilon \\
            &\Longrightarrow
            T^*(\bmu, \A)^{-1} \leq \sup_{\mathbf{w} \in \Delta^{2}} w_1 d(\A_1, \A'_1) \leq 4\epsilon,
        \end{align*}
        which completes the proof.

    \end{proof}

    Now, we discuss the main challenge in designing optimal algorithms for the case of unknown $\A$. When $\A$ is unknown, the problem becomes more challenging because the set of alternative parameters (Alt) changes in structure. Recall that all tracking algorithms based on \cite{track-stop-garivier2016optimal} must solve the optimization problem in Equation \eqref{eq: general_lower1} to determine the optimal weights. Alternatively, one can use the formulation in Equation~\eqref{eq: general_fi}, but both approaches involve a minimization over some set. The main practical complexity of tracking algorithms lies in solving this optimization problem.
    
    If an oracle were available to compute the optimal weights in Equation~\eqref{eq: general_lower1}, one could still use tracking algorithms even when \(\A\) is unknown and achieve the optimal sample complexity. Note that the GLR statistic and the stopping rule need to be adapted to this setting.

    When $\A$ is unknown, both the matrices $(\bmu,\A)$ should be learned by the agent during the learning process and if the error of estimation of each of them is high, the agent cannot stop. Let $\I^u$ denotes the set of parameter pairs $(\bmu, \A)$ such that the best arm $i^*(\bmu, \A)$ is unique. Then for a problem instance parameterized by $(\bmu, \A)$, we define 
    $$
    \text{Alt}(\bmu, \A)  = \{(\bmu', \A') \in \I^u \mid  i^*(\bmu', \A') \neq i^*(\bmu, \A)  \}. 
    $$

    Similar to the case of known $\A$, $\text{Alt}(\bmu, \A)$ can be written as Alt$(\bmu, \A) = \cup_{i \neq i^*(\bmu, \A)} \C_i$, where
    $$
    \C_i = \{ (\bmu', \A') \in \I^u \mid \A_i^{'\top} \bmu'_i > \A_{i^*(\bmu, \A)}^{'\top} \bmu'_{i^*(\bmu, \A)} \}.
    $$

    When \(\A\) is known, the sets \(\C_j\) or \(\text{Alt}\) are convex, so standard convex optimization tools can be applied. However, when $\A$ is unknown, the sets $\C_j$ or Alt, may not be convex, and thus finding optimal weights via convex optimization methods becomes infeasible. Similar issues have been noted in the reinforcement learning literature \citep{al2021adaptive}.    
    
    As mentioned earlier, the case of separator context is a special case of non-separator setting, so it suffices to show non-convexity for the separator case.  Note that if $\A$ is unknown, the set $\text{Alt}(\bmu, \A)$ changes because $\A$ itself can vary within the alternative set. Formally, the new definition of \text{Alt}$(\bmu, \A)$ for this setting is
    
    $$
        \text{Alt}(\bmu, \A)  = \{(\bmu', \A') \in \I \mid  i^*(\bmu', \A') \neq i^*(\bmu, \A)  \},
    $$
    where $\I$ denotes the set of all instances (parameterized by $(\bmu, \A)$) that have a unique best arm. Similar to known setting, we have Alt$(\bmu, \A) = \cup_{i \neq i^*(\bmu, \A)} \C_i$, and
    $$
    \C_i = \{ (\bmu', \A') \in \I \mid \A_i^{'\top} \bmu' > \A_{i^*(\bmu, \A)}^{\top} \bmu \}.
    $$
    As mentioned earlier, we need to solve
    $$
        \inf_{(\bmu', \A') \in \text{Alt}(\bmu, \A)} \sum_{i \in [n]} w_i d(P^{\bmu, \A}_{i}, P^{\bmu', \A'}_{i}),
    $$
    or equivalently, using Equation~\eqref{eq: general_fi},
    \begin{align}
        f_j(\mathbf{w}, \bmu, \A) = \inf_{(\bmu', \A') \in \C_j} \sum_{i \in [n]} w_i d(P^{\bmu, \A}_{i}, P^{\bmu', \A'}_{i}).
    \end{align}
    
    For the case of known $\A$, each set $\C_j$ is convex, so $f_j(\mathbf{w}, \bmu, \A)$ reduces to a convex problem that can be solved via convex optimization methods. However, if $\A$ is unknown, that convexity property no longer holds. The following example shows that neither $\text{Alt}(\bmu, \A)$ nor $\C_j$ is necessarily convex.

    \textbf{Example.} Consider $n = k = 2$. Suppose $i^*(\bmu, \A) = 1$, then $\text{Alt}(\bmu, \A)$ equals $\C_2$. Therefore, we only need to show that $\text{Alt}(\bmu, \A)$ is not convex. Let
    \begin{align}
        &\A = \begin{bmatrix}
                0.6 & 0.4 \\
                0.4 & 0.6
            \end{bmatrix}, ~
        \bmu = \begin{bmatrix}
                    3\\
                    2
        \end{bmatrix}
    \end{align}
    Clearly, \(i^*(\bmu, \A) = 1\) because \(\mathbb{E}[Y \mid X=1] = 2.6\) and \(\mathbb{E}[Y \mid X=2] = 2.4\). Next, let
    \begin{align}
        &\A^{(1)} = \begin{bmatrix}
                0.4 & 0.2\\
                0.6 & 0.8
            \end{bmatrix}, ~
        \bmu^{(1)} = \begin{bmatrix}
                    1 \\
                    3
        \end{bmatrix}, \\
        &\A^{(2)} = \begin{bmatrix}
                0.1 & 0.2 \\
                0.9 & 0.8
            \end{bmatrix}, ~
        \bmu^{(2)} = \begin{bmatrix}
                    5 \\
                    1
        \end{bmatrix}.
    \end{align}
    We see $i^*(\bmu^{(1)}, \A^{(1)}) = i^*(\bmu^{(2)}, \A^{(2)}) = 2$. Therefore, $(\bmu^{(1)}, \A^{(1)})$ and $(\bmu^{(2)}, \A^{(2)})$ belong to $\text{Alt}(\bmu, \A)$. Consider a convex combination of $(\bmu^{(1)}, \A^{(1)})$ and $(\bmu^{(2)}, \A^{(2)})$ as
    
    \begin{align}
        \A^{(3)} = \frac{1}{2} (\A^{(1)}  + \A^{(2)}) = \begin{bmatrix}
                0.25 & 0.2 \\
                0.75 & 0.8
            \end{bmatrix},
        \bmu^{(3)} = \frac{1}{2} (\bmu^{(1)}  + \bmu^{(2)}) = \begin{bmatrix}
                    3 \\
                    2
        \end{bmatrix}.
    \end{align}
    Now, we have $i^*(\bmu^{(3)}, \A^{(3)}) = 1$. It indicates that $(\bmu^{(3)}, \A^{(3)})$ is \emph{not} in $\text{Alt}(\bmu, \A)$, which shows that $\text{Alt}(\bmu, \A)$ is not a convex set.

\section{Further Experiments} \label{apd: experiment}

This section contains additional details on the experimental setup and the results of further experiments for both non-separator and separator settings.

\subsection{Real-World Data Experiment}
\label{subsec:real_world_exp}

We also evaluate our methods on instances constructed from the \emph{KuaiSAR} recommendation logs \cite{real-dataset-sun2023kuaisar}, a real-world short video recommendation dataset containing user--item interaction records together with feedback signals such as clicks, likes, and watch behavior. The goal of these experiments is not to provide a large-scale benchmark, but rather to verify on real interaction data that the empirical behavior is consistent with the theoretical predictions of our framework.

For the non-separator setting, we construct a BAI instance directly from the real data. We use first level content categories as actions, resulting in $n=7$ actions with sufficient support. Concretely, each action corresponds to recommending an item from one of these content categories, so pulling an arm means selecting a category and then observing the feedback generated by an interaction record associated with that category. Operationally, whenever the learner selects an action, one row of the dataset corresponding to that action is sampled uniformly at random, and the observed post-action context and reward are taken from that sampled interaction record.
As post-action context, we consider three discrete feedback outcomes observed after recommendation exposure:
(i) no click,
(ii) click only, and
(iii) click+like.
As reward, we use the clipped watch-ratio
$$
R = \min\!\left(\frac{\text{playing time}}{\text{duration}},\,3\right),
$$
where \emph{playing time} is the total time the user spent watching the recommended video and \emph{duration} is the nominal video length. This ratio is a standard normalized engagement measure; it can exceed $1$ because users may replay or rewatch parts of a video, so the accumulated playing time can be longer than the video duration. We clip the ratio at $3$ to control extreme values while preserving meaningful variation in engagement. To obtain a statistically meaningful BAI instance, we retain actions with sufficient support across all three context values and enforce a minimum gap between the best and second best arms. This yields a genuine non-separator instance: even after conditioning on the post-action context, the expected reward still depends on the action.

For the separator setting, we start from the same real interaction data and use it to estimate the post-action context distribution of each action. We then construct rewards so that they depend only on the realized context. More precisely, for each context value $z$, we compute its empirical mean reward $\mu_z$ from the data and generate rewards as
$$
Y = \mu_z + \mathcal{N}(0,1).
$$
Hence, the action still influences which context is observed, but conditional on the context, the reward is independent of the action. Although this separator instance is not fully raw real-world data, it remains tightly grounded in the dataset because both the action-dependent context probabilities and the context level reward means are estimated from real user interactions.

For both settings, we report the average stopping time over $20$ independent runs. For the TS baseline on the non-separator instance, since the reward is bounded in $[0,3]$, we use the Gaussian stopping rule with variance proxy $\sigma^2=(3-0)^2/4=9/4$, corresponding to the standard sub-Gaussian parameter guaranteed by Hoeffding's lemma for bounded random variables. The results are shown in Table~\ref{tab:real_world_results}. In the non-separator setting, NSTS significantly improves over TS. In the separator setting, STS dramatically outperforms context-ignoring baselines, which fail to stop within the horizon of $50{,}000$ rounds. These results are consistent with the main message of the paper: exploiting post-action context leads to substantial gains when used in a manner aligned with the underlying problem structure.

\begin{table}[t]
\centering
\caption{Real-world data experiments based on KuaiSAR \cite{real-dataset-sun2023kuaisar}. Reported values are average stopping times over $20$ independent runs.}
\label{tab:real_world_results}
\begin{tabular}{llc}
\hline
Setting & Algorithm & Avg. stopping time \\
\hline
Non-separator (KuaiSAR, $7$ actions) & TS & $16355$ \\
Non-separator (KuaiSAR, $7$ actions) & NSTS (ours) & $\mathbf{6973}$ \\
\hline
Separator (KuaiSAR-based, $7$ actions) & STS (ours) & $\mathbf{411}$ \\
Separator (KuaiSAR-based, $7$ actions) & TS & $> 50000$ \\
Separator (KuaiSAR-based, $7$ actions) & LTS & $> 50000$ \\
\hline
\end{tabular}
\end{table}

\subsection{Non-Separator Context} 

We generate random instances for each combination of $n \in \{5, 10, 15\}$ and $k \in \{3, 5, 7\}$. Each instance has $n$ arms and $k$ context values, with rewards drawn from Gaussian distributions of unit variance. The mean matrix \(\bmu\) lies in \([0,10]^{k \times n}\), and \(\amin = \min_{i,j} \A_{j,i} \ge \frac{1}{4k}\). Without loss of generality, we assume the first arm is the best. We randomly generate the matrices $\bmu$ and $\A$, ensuring that for each $i > 1$, $\Delta_i \in [\frac{1}{2n}, \frac{i+1}{2n}]$.  This constraint prevents unchallenging instances in which all gaps are large. Throughout our experiments, we set the confidence parameter $\delta$ to $0.1$.

We compare our algorithm (NSTS) with the classic track-and-stop (TS) algorithm, which disregards the context variable and instead operates solely on the rewards with D-tracking. Note that the TS does not have a theoretical guarantee in this setting because each arm's reward distribution is a mixture of Gaussians, which does not belong to the one-parameter exponential family required for theoretical guarantees of TS \citep{track-stop-garivier2016optimal}. Note that as the expected values of the rewards are in $[0,10]$, it can be shown that the rewards for each arm follow a sub-Gaussian distribution. We then use this property to design a stopping rule. For more details, refer to Appendix \ref{apd: experiment}. 

Table \ref{tab: non-sep} reports the average number of rounds required by each algorithm before stopping, aggregated over $75$ runs. The results indicate that leveraging post-action contexts can significantly improve performance.

\begin{table}[H]
        \centering
        \small
        \caption{Average stopping times of NSTS and TS algorithms for different values of \(n\) and \(k\).  
        The Ratio column shows TS/NSTS.}
        \begin{tabular}{c ccc ccc ccc}
        \toprule
        \multirow{2}{*}{\( n \)} & 
        \multicolumn{3}{c}{\(k = 3\)} & 
        \multicolumn{3}{c}{\(k = 5\)} & 
        \multicolumn{3}{c}{\(k = 7\)} \\
        \cmidrule(lr){2-4}\cmidrule(lr){5-7}\cmidrule(lr){8-10}
         & NSTS & TS & Ratio & NSTS & TS & Ratio & NSTS & TS & Ratio \\
        \midrule
        5  & 30635  & 427872  & 14.0 & 31515  & 308976  & 9.8 & 60960  & 534147  & 8.8 \\
        10 & 92034  & 1244462 & 13.5 & 224236 & 2090792 & 9.3 & 320224 & 2433472 & 7.6 \\
        15 & 272064 & 3667765 & 13.5 & 425685 & 3768745 & 8.9 & 428615 & 3573301 & 8.3 \\
        \bottomrule
        \end{tabular} 
    \label{tab: non-sep}
\end{table}

    ‌Before presenting further details on the results, we discuss the sequential thresholds used in the algorithms. For NSTS, we used the exact thresholds from \eqref{eq: sep threshold}, though they tend to be overly conservative, keeping the error probability well below $\delta$.
    
    As discussed earlier, there is no theoretical guarantee for TS when interacting with a bandit featuring post-action context. This is because, in such cases, the reward distribution of each arm becomes a mixture of Gaussians, which does not belong to the exponential family of distributions. Moreover, experimental results indicate that applying the same thresholds as in the classic bandit setting with unit-variance Gaussian rewards does not control the error rate below $\delta$ (refer to Appendix \ref{apx: non-optimal-non-sep} for a detailed discussion). 
    
    To design new thresholds for TS, we utilized Lemma of \cite{russac2021b}. Since $\bmu \in [0,10]^{k \times n}$, the reward distribution of the arms is sub-Gaussian with parameter $\sigma = \sqrt{26}$. For employing the TS algorithm, we approximate the reward distribution with a Gaussian distribution with variance $\sigma^2 = 26$. Consequently, we can use the known threshold for classic bandit with Gaussian distribution \citep{kaufmann2021mixture}.
    
    Figure \ref{fig: non-sep-dist} shows the average $L^2$ distance between the vector $\frac{\Nb^{X}(t)}{t}$ and the optimal vector $\wstar{\bmu}$ over time for all nine instances introduced in the main text simulated over $75$ runs. For each run, after its stopping time, the distance is considered as zero in the average.  The results indicate that ignoring the post-action context leads to significant sub-optimality and misalignment in tracking the correct proportion of actions.




\begin{figure}
    \centering
    \begin{tabular}{cc}
        \includegraphics[width=0.35\textwidth]{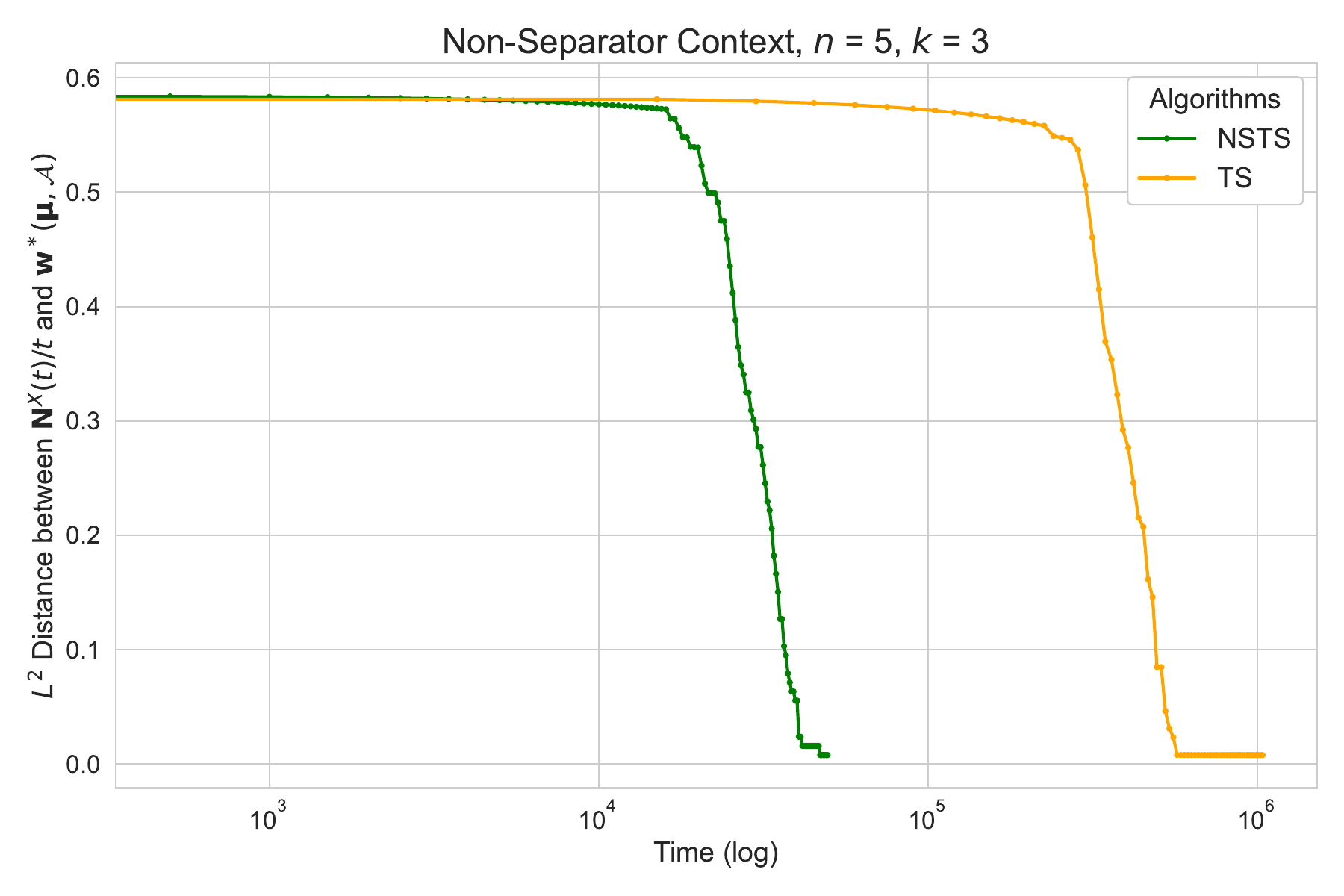} &
        \includegraphics[width=0.35\textwidth]{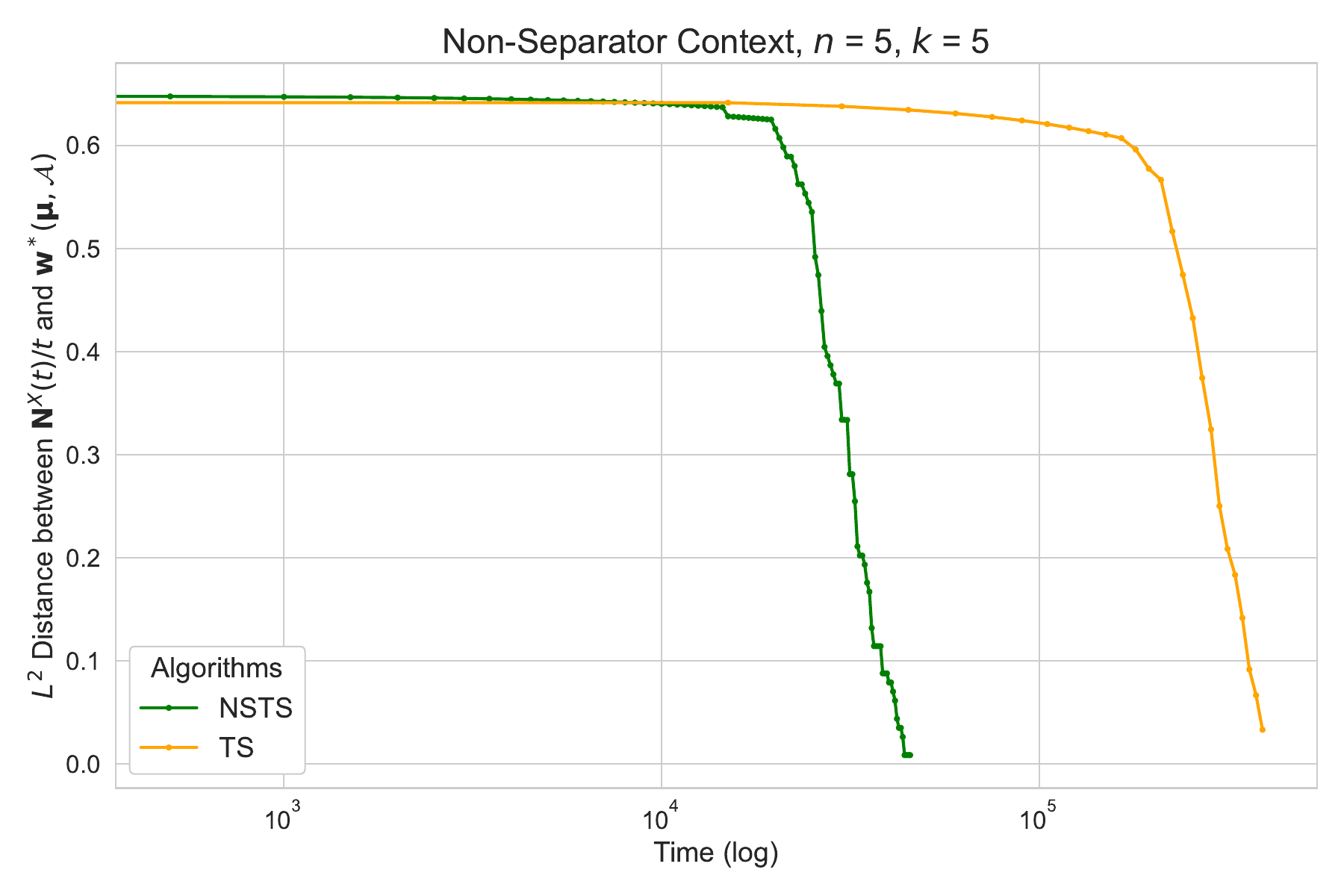} \\
        \includegraphics[width=0.35\textwidth]{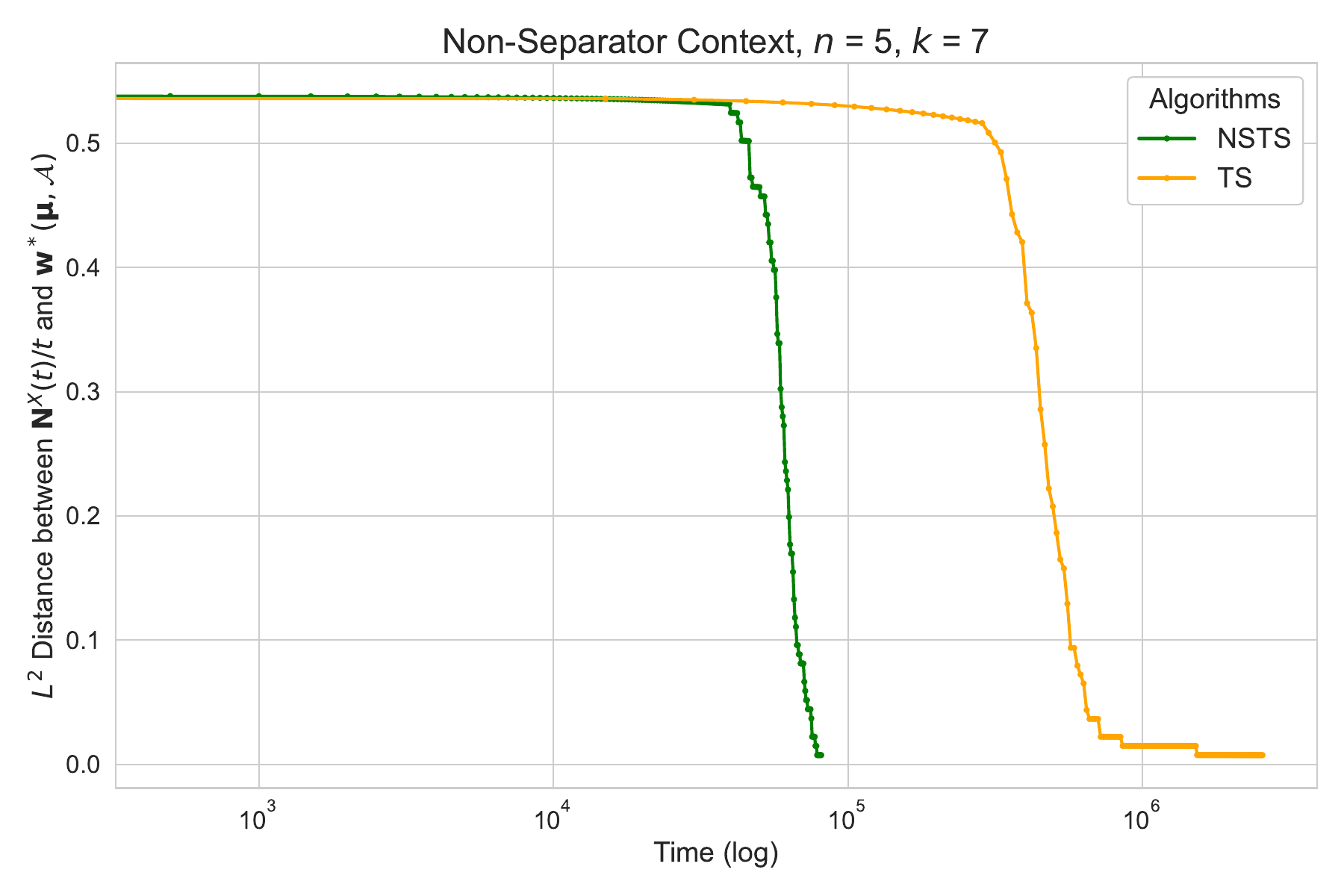} &
        \includegraphics[width=0.35\textwidth]{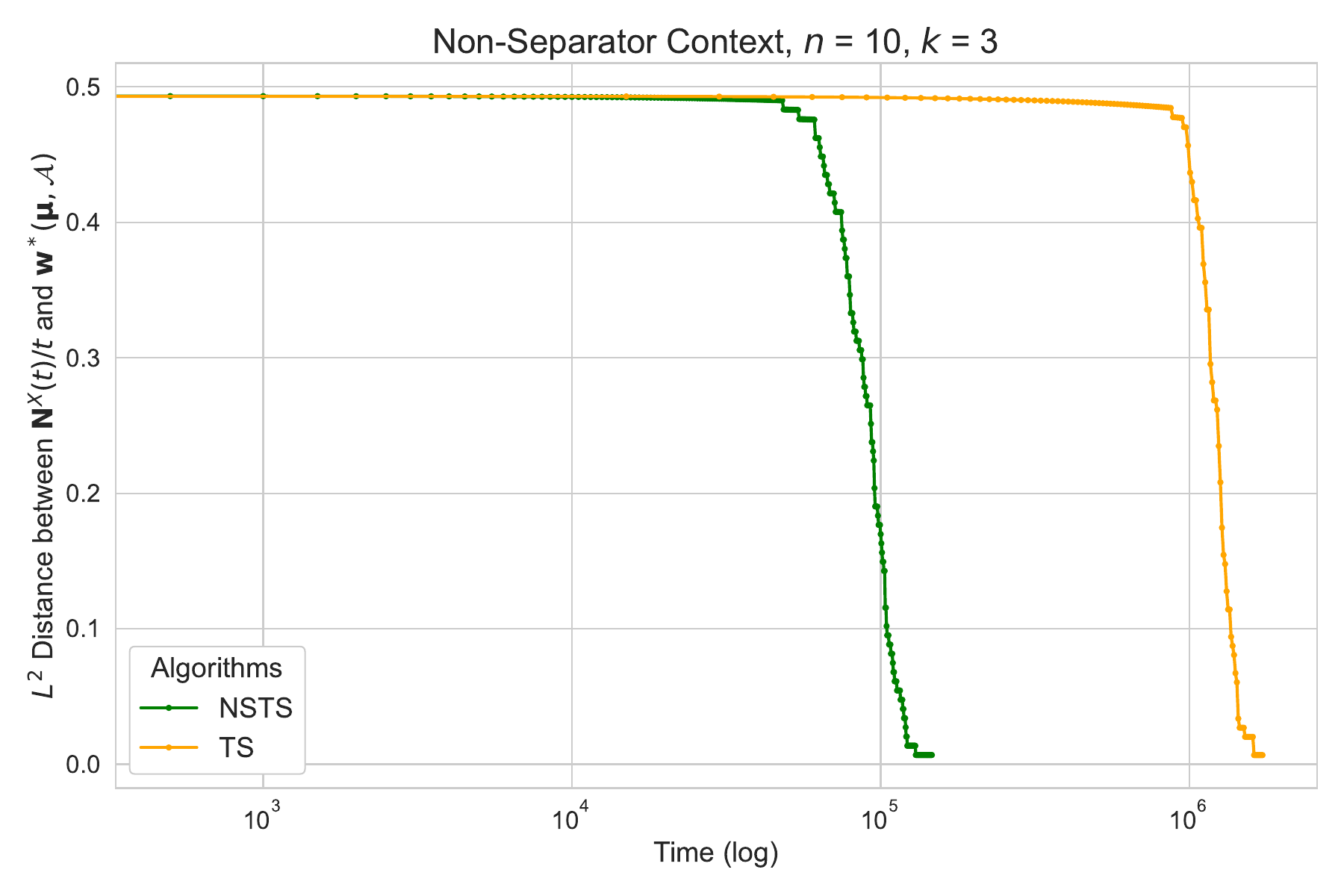} \\
        \includegraphics[width=0.35\textwidth]{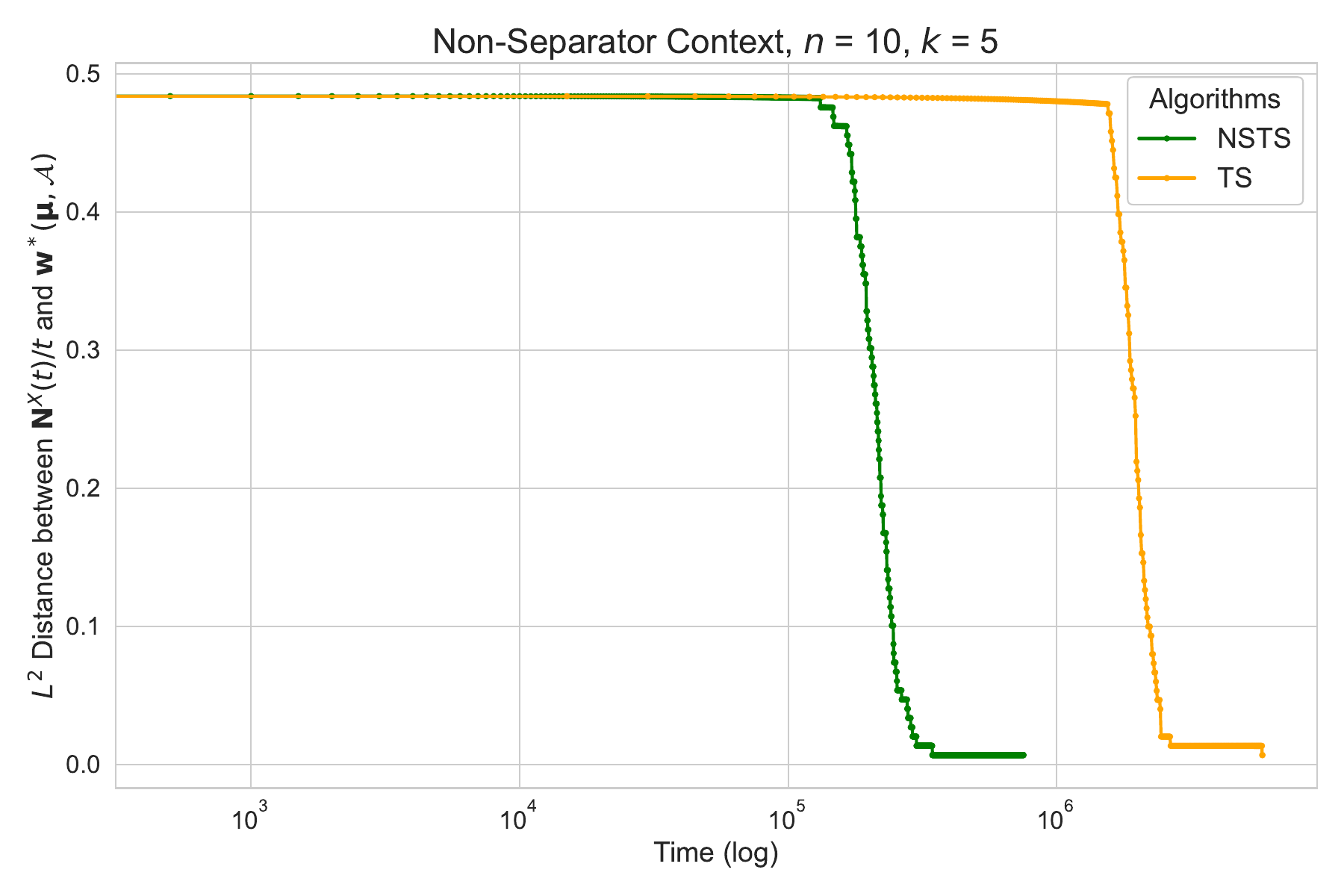} &
        \includegraphics[width=0.35\textwidth]{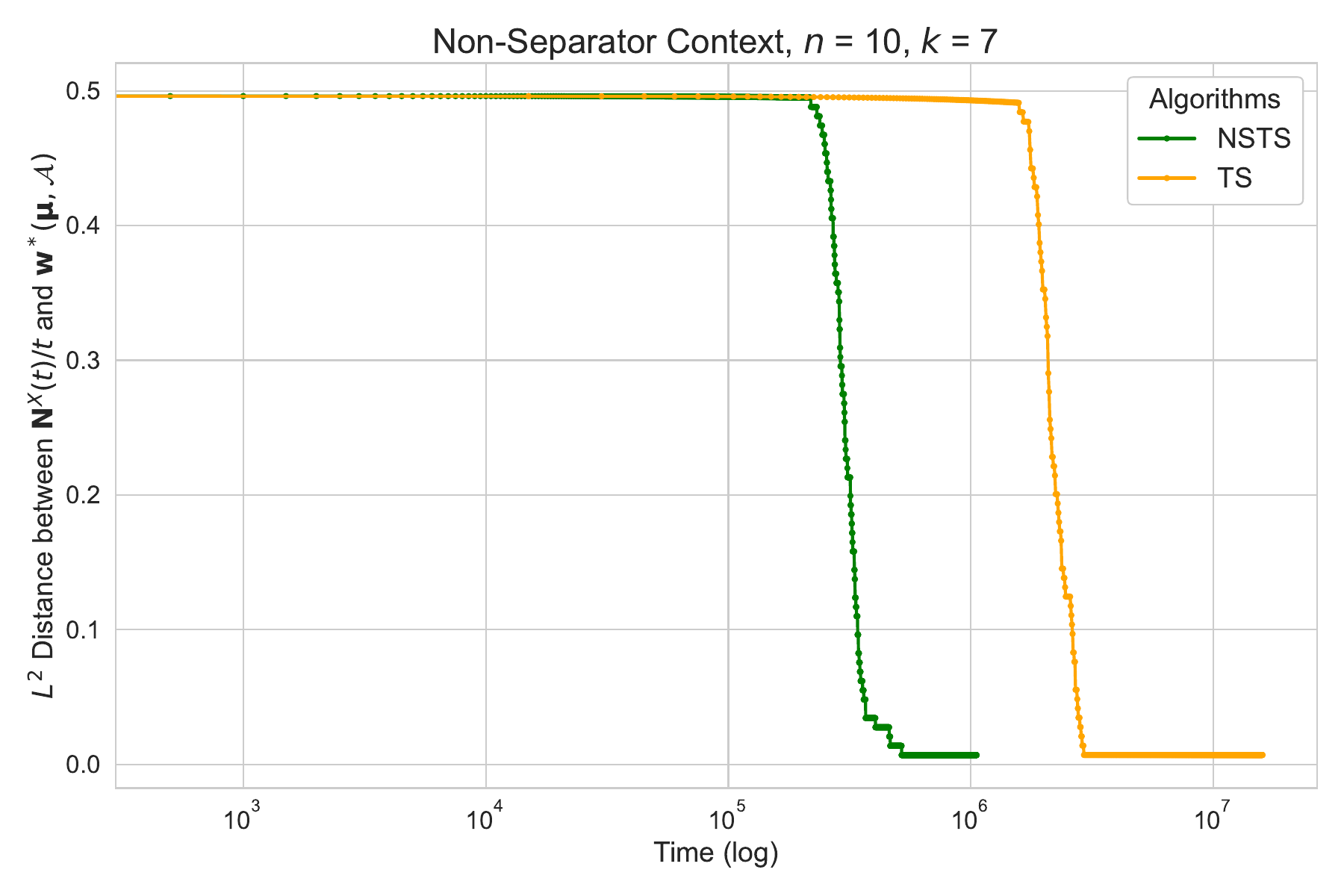} \\
        \includegraphics[width=0.35\textwidth]{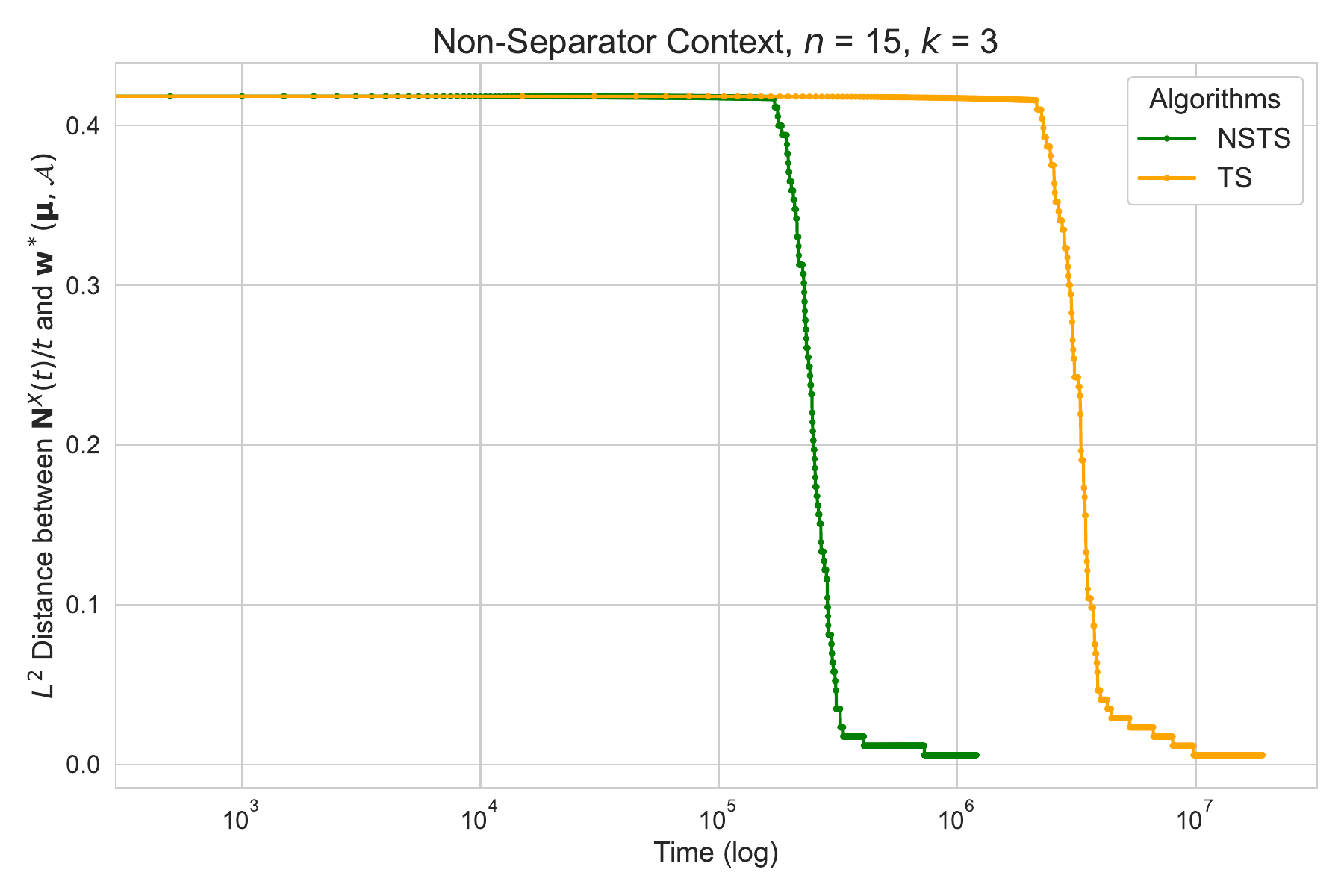} &
        \includegraphics[width=0.35\textwidth]{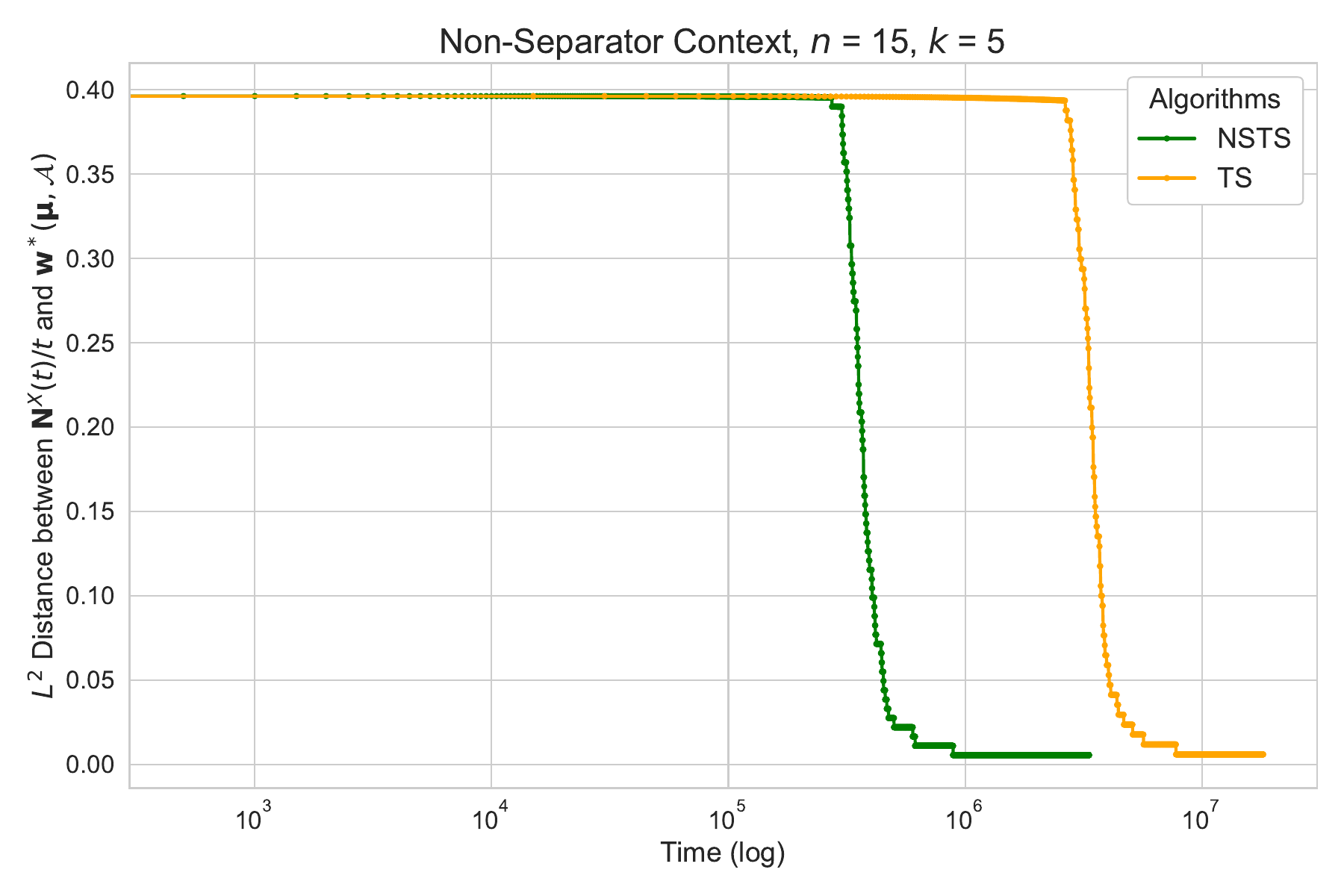} \\
        \includegraphics[width=0.35\textwidth]{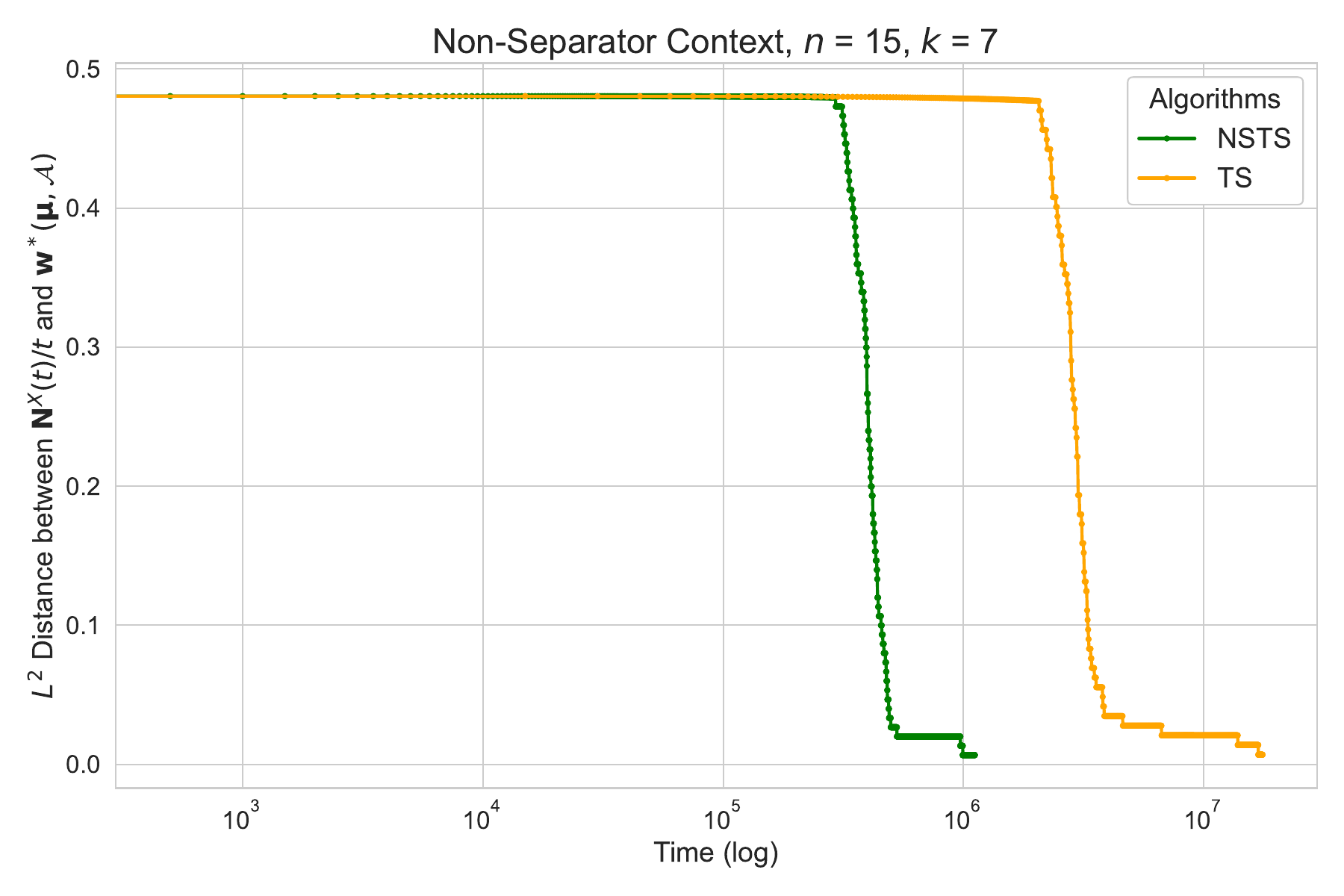} &
    \end{tabular}
    \caption{Comparison of the $L^2$ distance of the frequencies of pulled arms and the optimal frequency over time between two algorithms.}
    \label{fig: non-sep-dist}
\end{figure}




\begin{figure}
    \centering
    \begin{tabular}{cc}
        \includegraphics[width=0.35\textwidth]{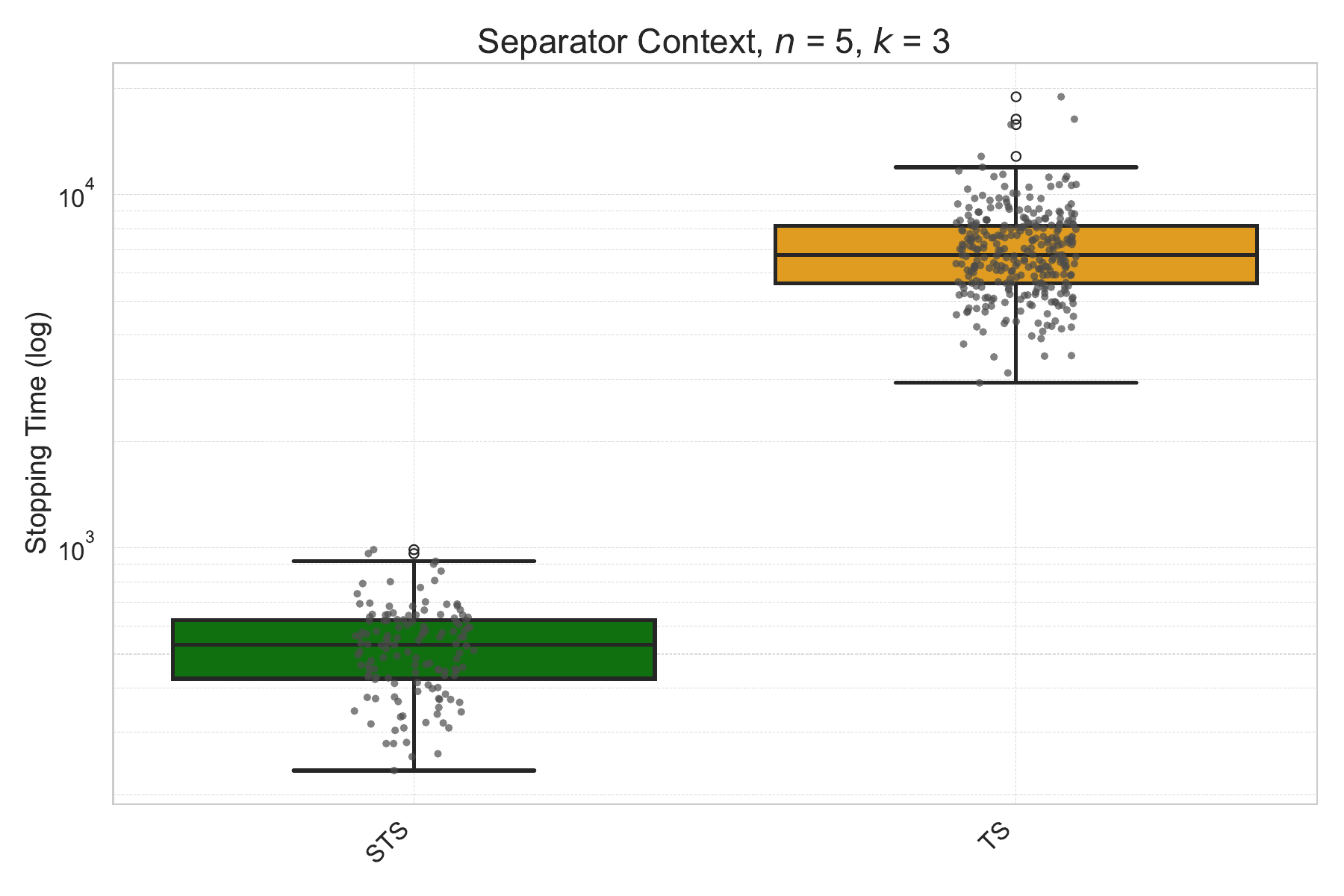} &
        \includegraphics[width=0.35\textwidth]{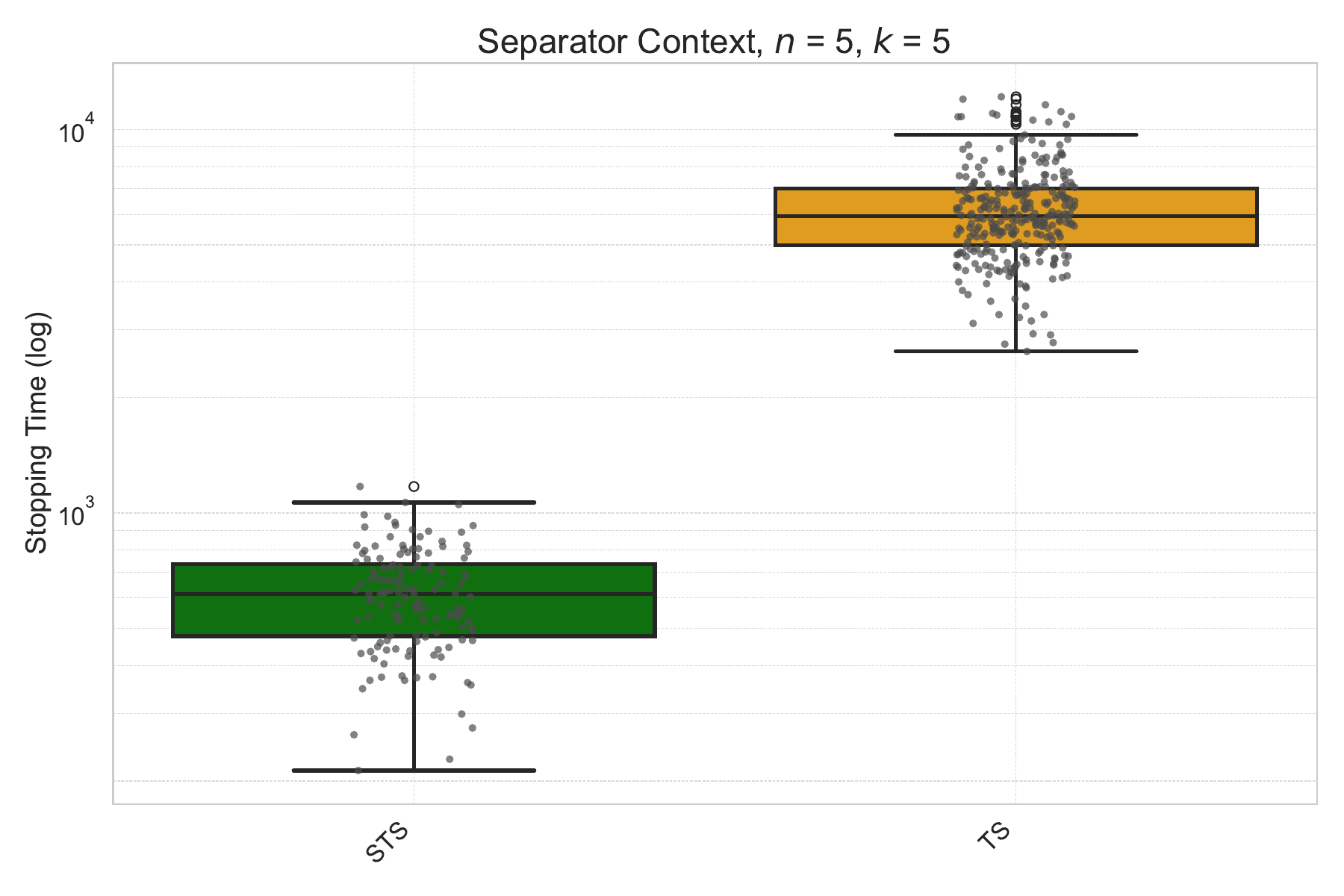} \\
        \includegraphics[width=0.35\textwidth]{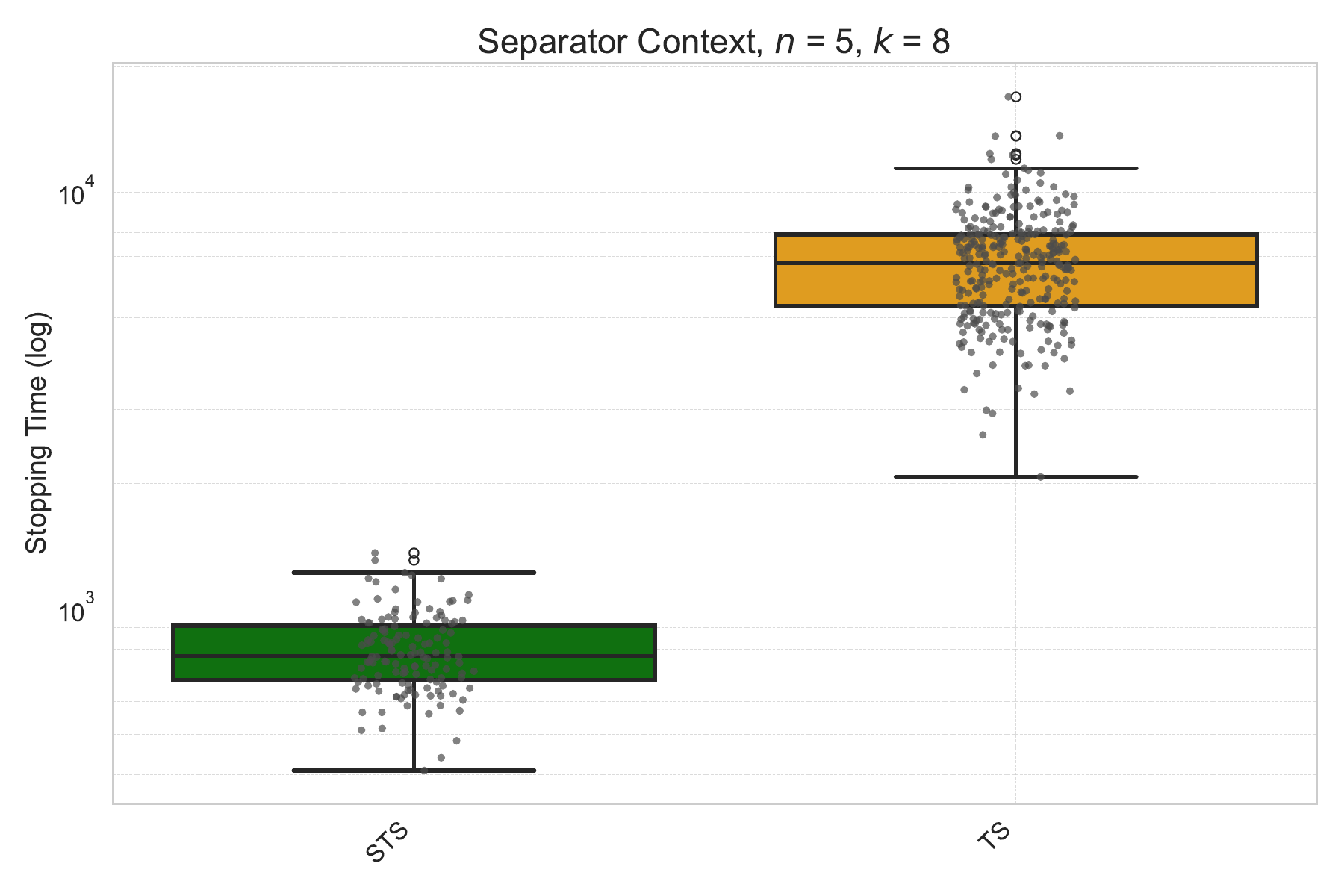} &

         \includegraphics[width=0.35\textwidth]{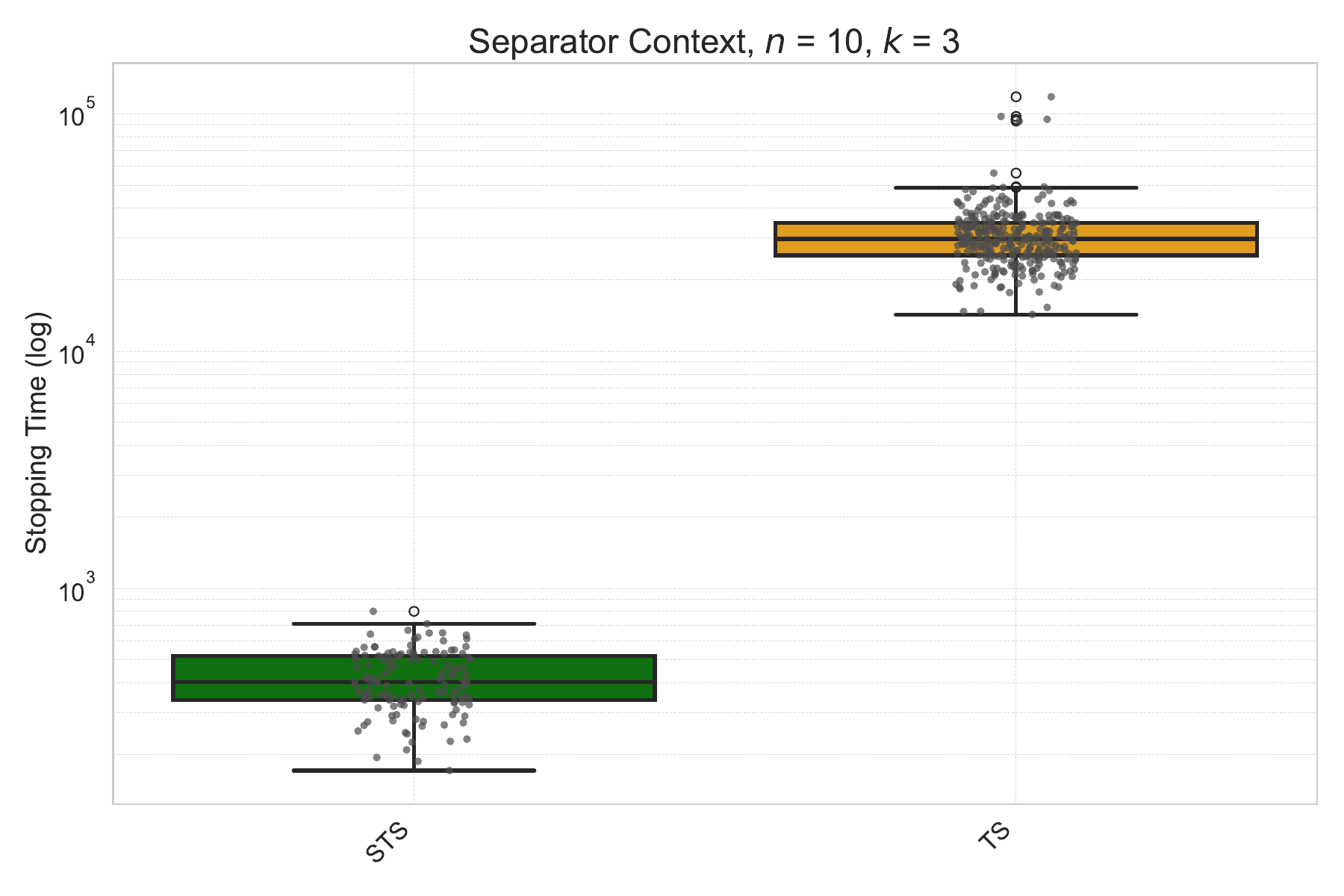} \\
        \includegraphics[width=0.35\textwidth]{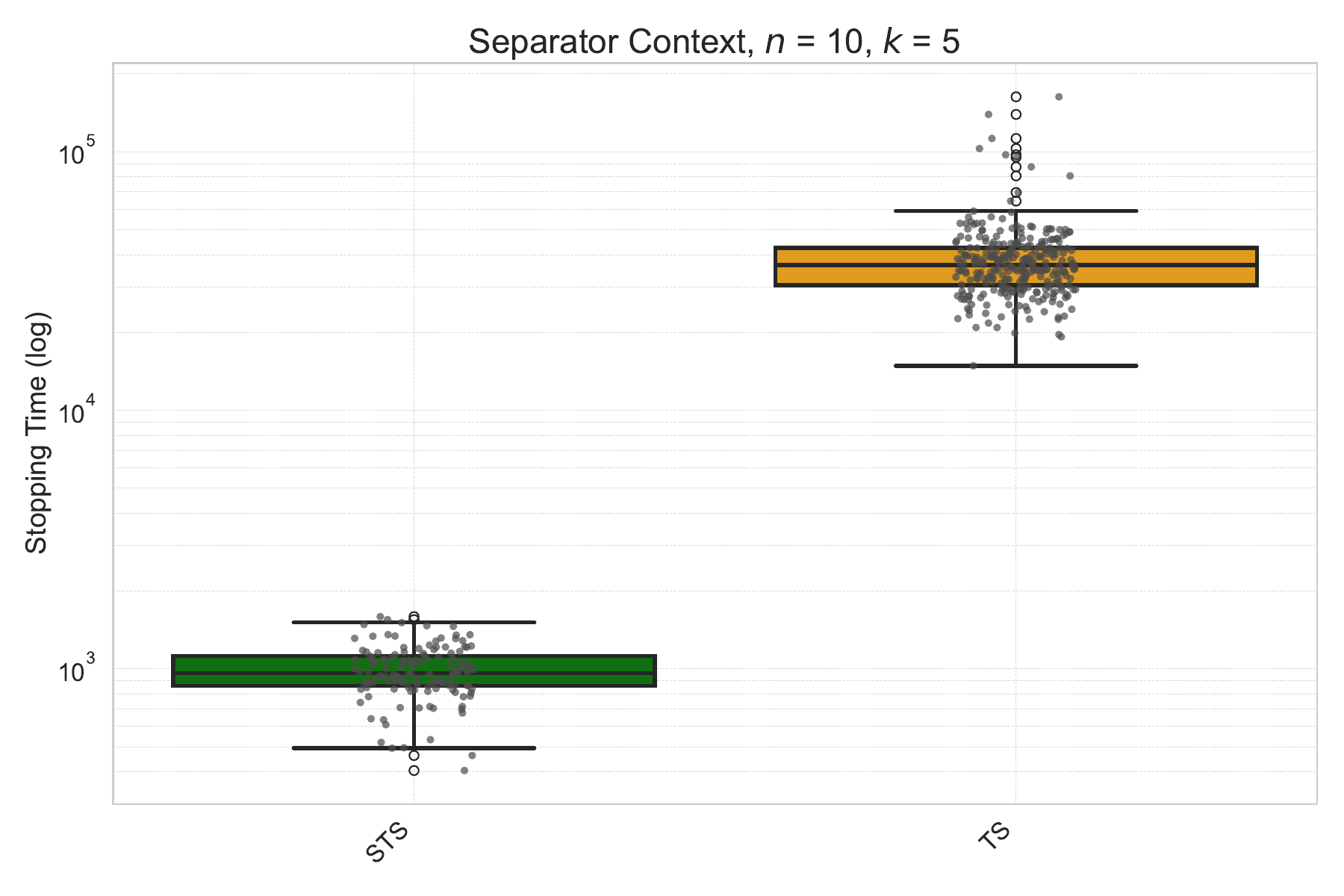} &
        \includegraphics[width=0.35\textwidth]{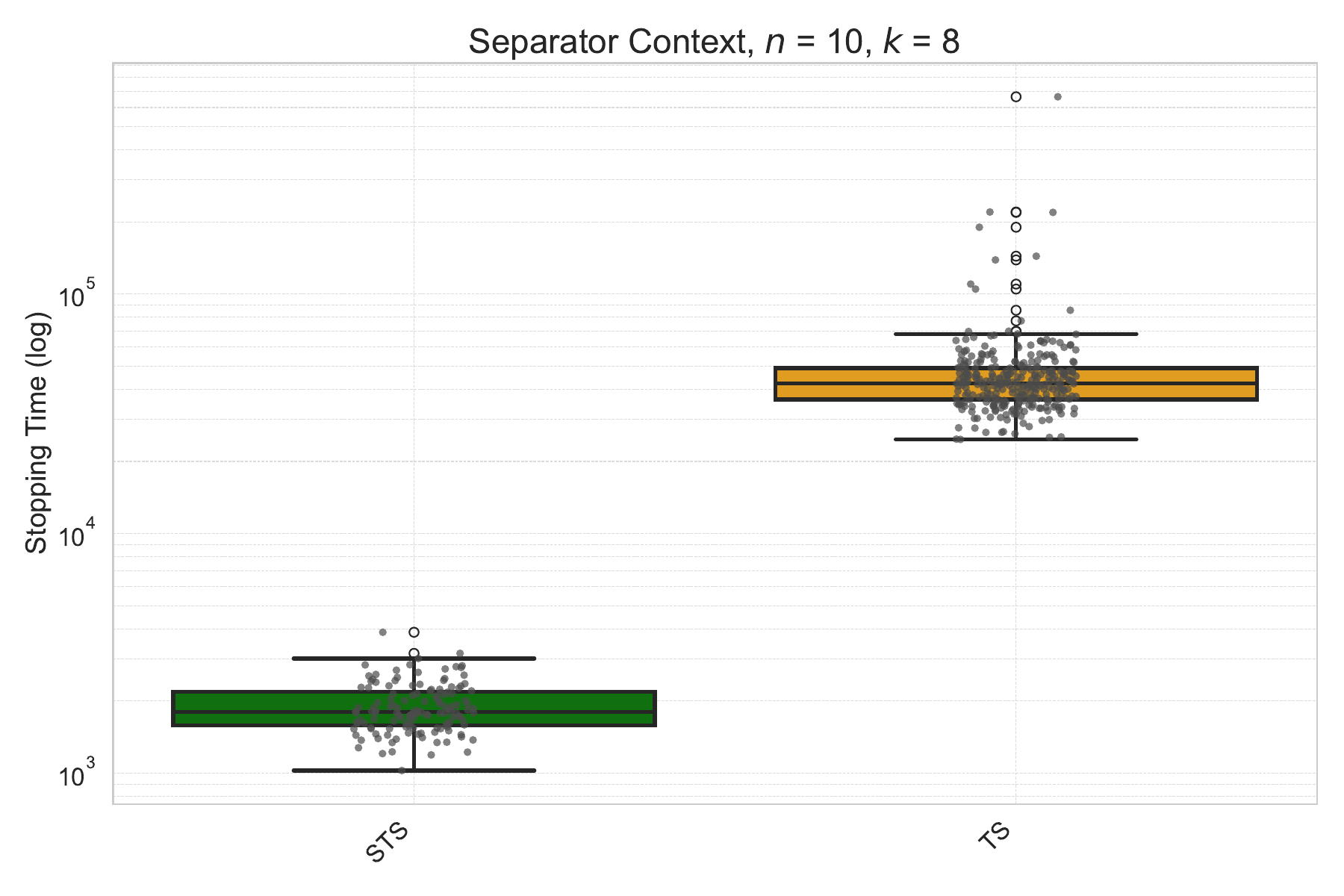} \\

         \includegraphics[width=0.35\textwidth]{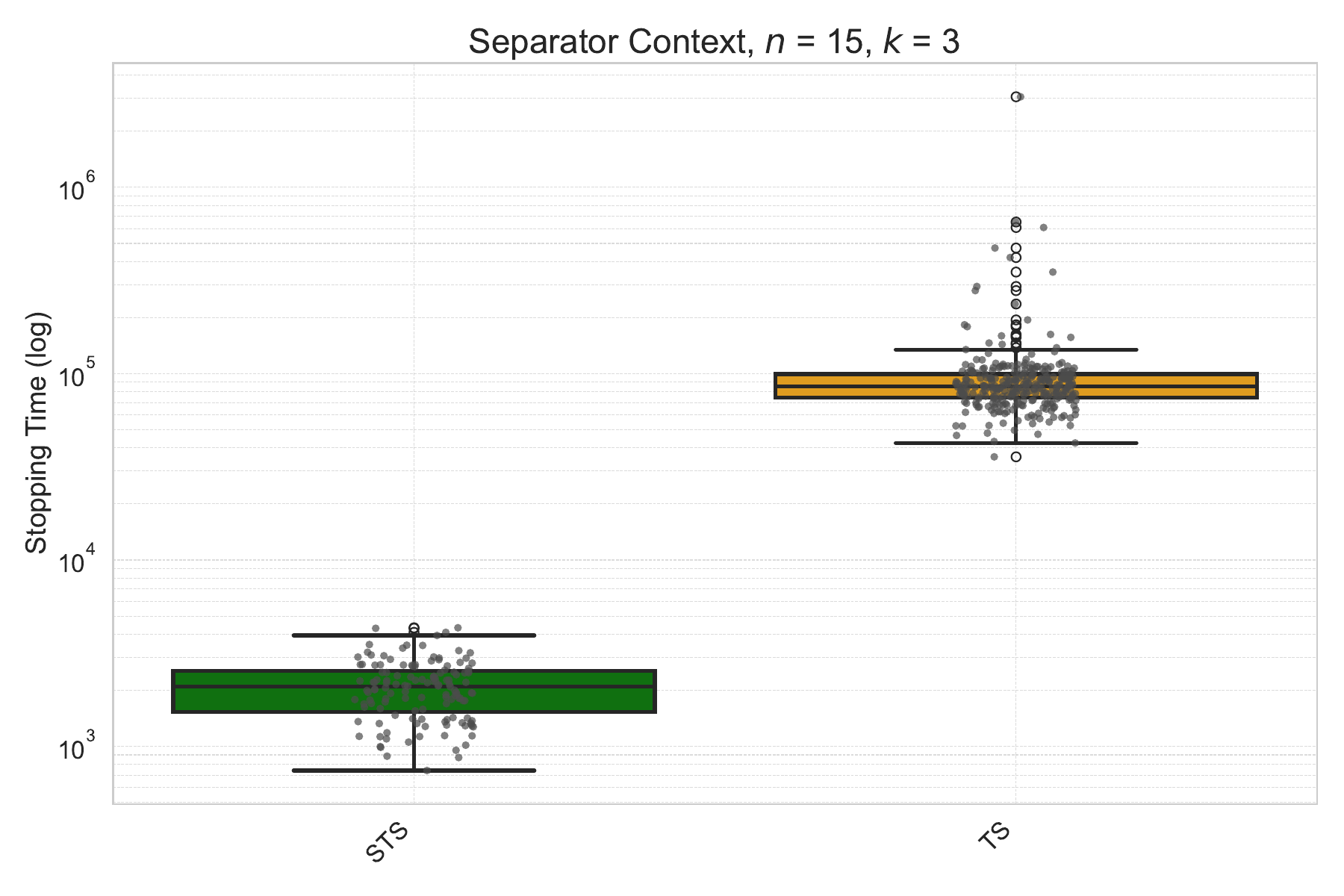} &
        \includegraphics[width=0.35\textwidth]{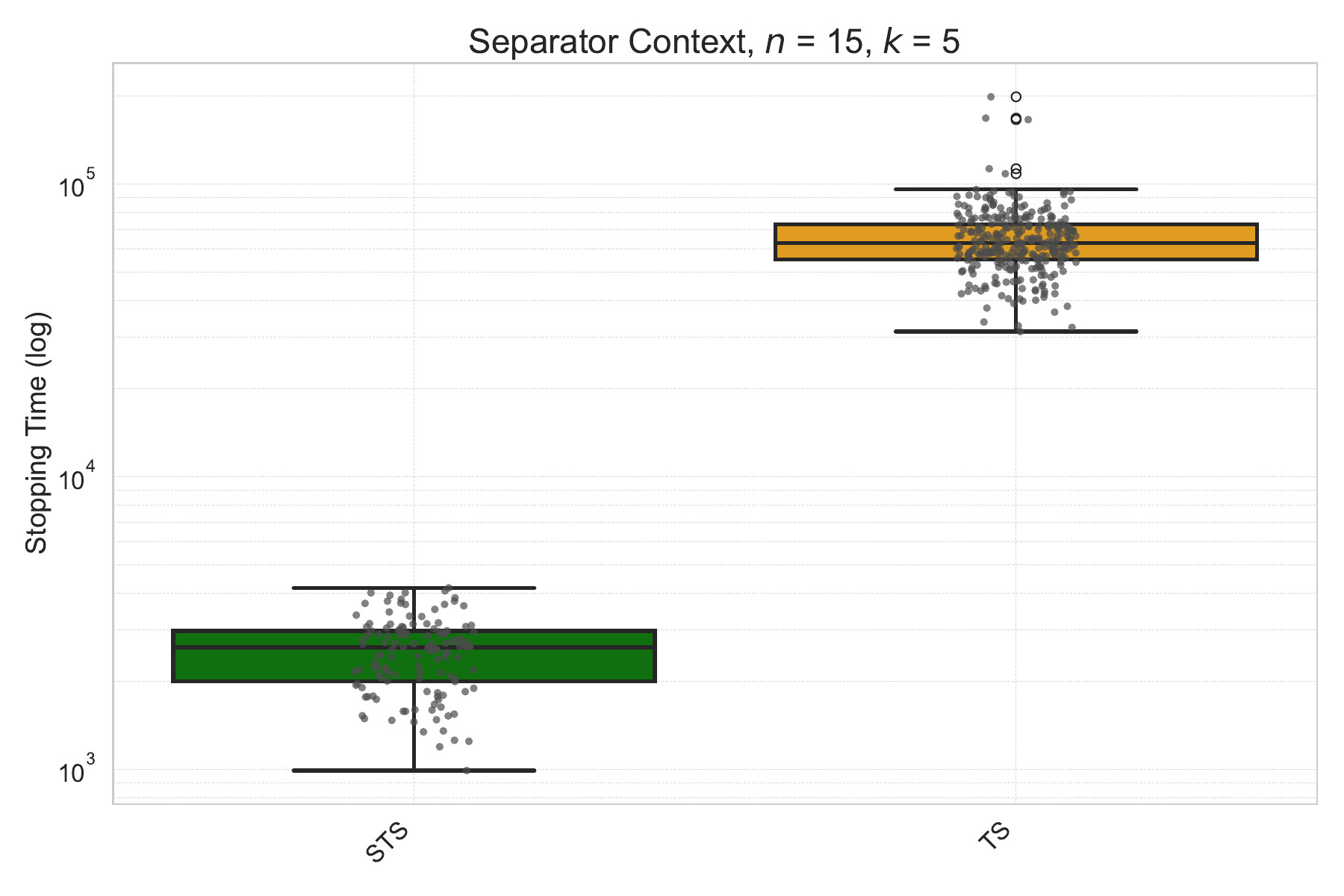} \\
        \includegraphics[width=0.35\textwidth]{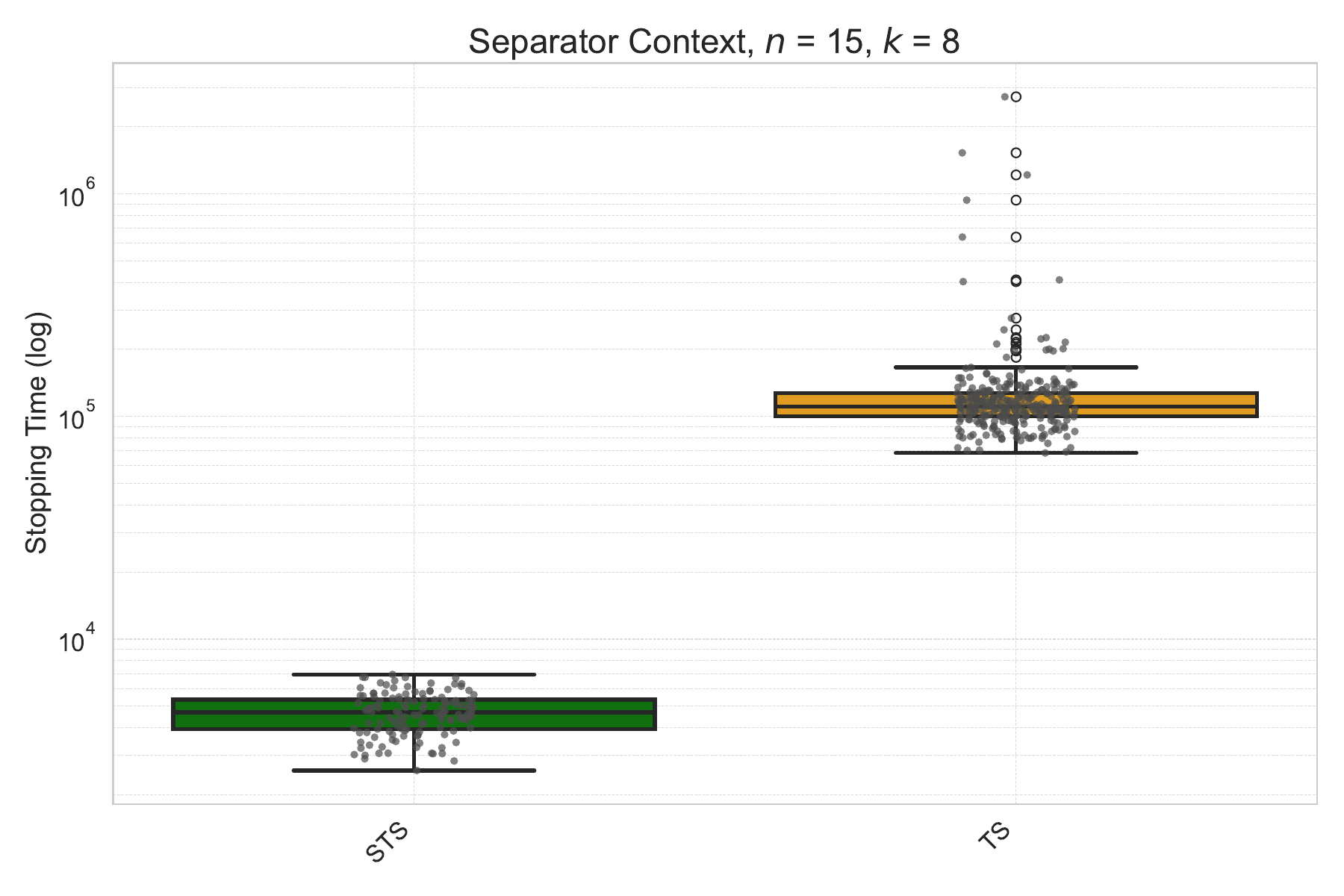} &
    \end{tabular}
    \caption{Comparison of the stopping times of different algorithms on randomly generated instances.}
    \label{fig: sep-box}
\end{figure}

\subsection{Separator Context}
    Here, we present additional results from running different algorithms on instances in the separator setting. For each $n \in \{5, 10, 15\}$ and $k \in \{3, 5, 8\}$, we randomly generated an instance with $n$ arms and $k$ context values. To avoid trivial instances, we repeatedly generated instances until finding one where $\Delta_i \in \left[\frac{1}{2n}, \frac{i+1}{2n}\right]$, with the first arm being the best arm. 

    As shown in the main text, the LTS algorithm does not perform competitively, and due to its long computation time, we do not report its results on these instances and only compare TS and Algorithm \ref{algo: sep}. Figure \ref{fig: sep-box} presents a box plot of the stopping times for these algorithms across all instances. The results demonstrate that ignoring information about the post-action context leads to significant sub-optimality. 

    Figure \ref{fig: sep-dist} shows the average $L^2$ distance of the context frequencies vector $\frac{\Nb^{Z}(t)}{t}$ and the optimal context frequencies $\wzstar{\bmu}$ over time for all previously introduced instances and for different algorithms.




\begin{figure}[t]
    \centering
    \begin{tabular}{cc}
        \includegraphics[width=0.35\textwidth]{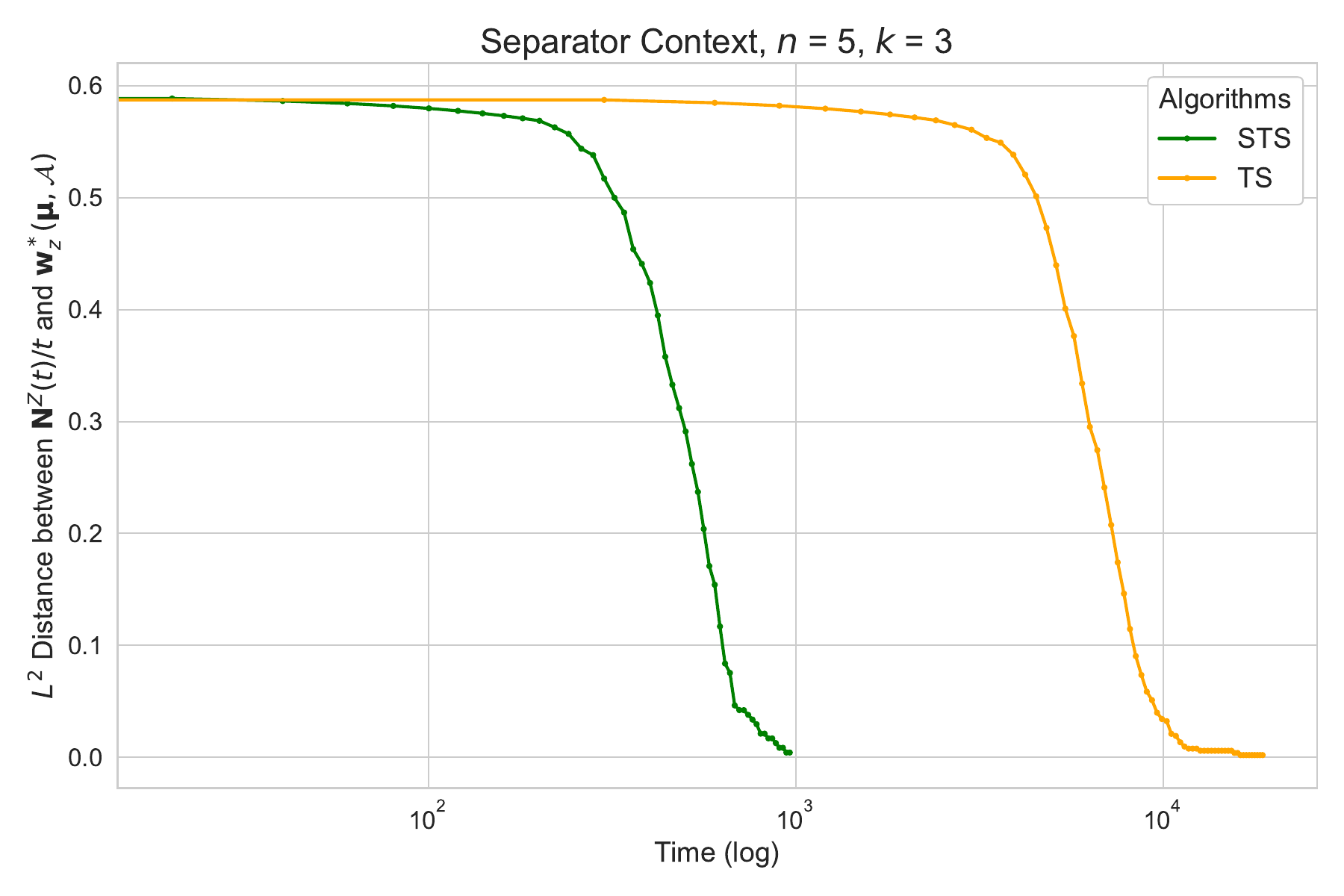} &
        \includegraphics[width=0.35\textwidth]{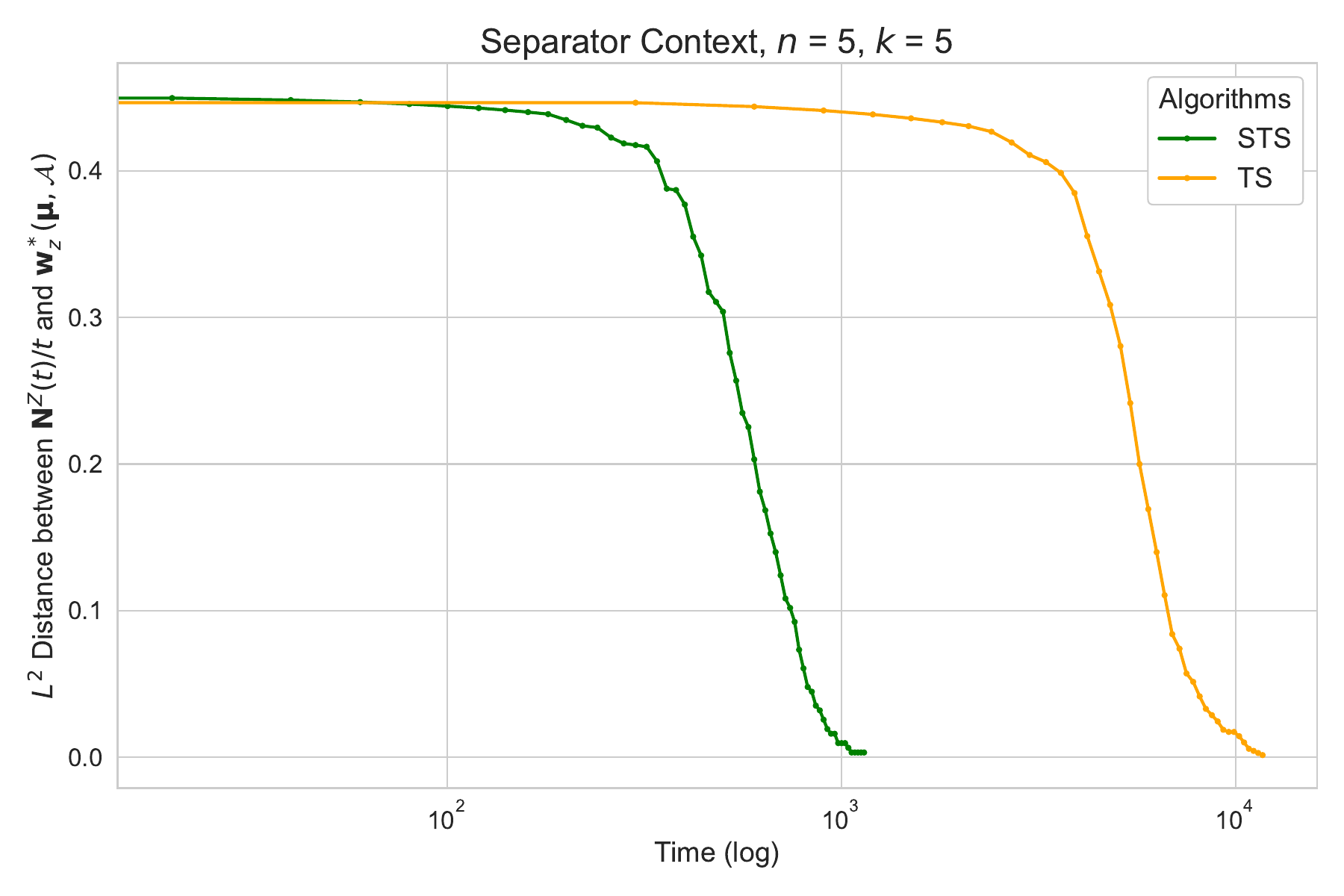} \\
        \includegraphics[width=0.35\textwidth]{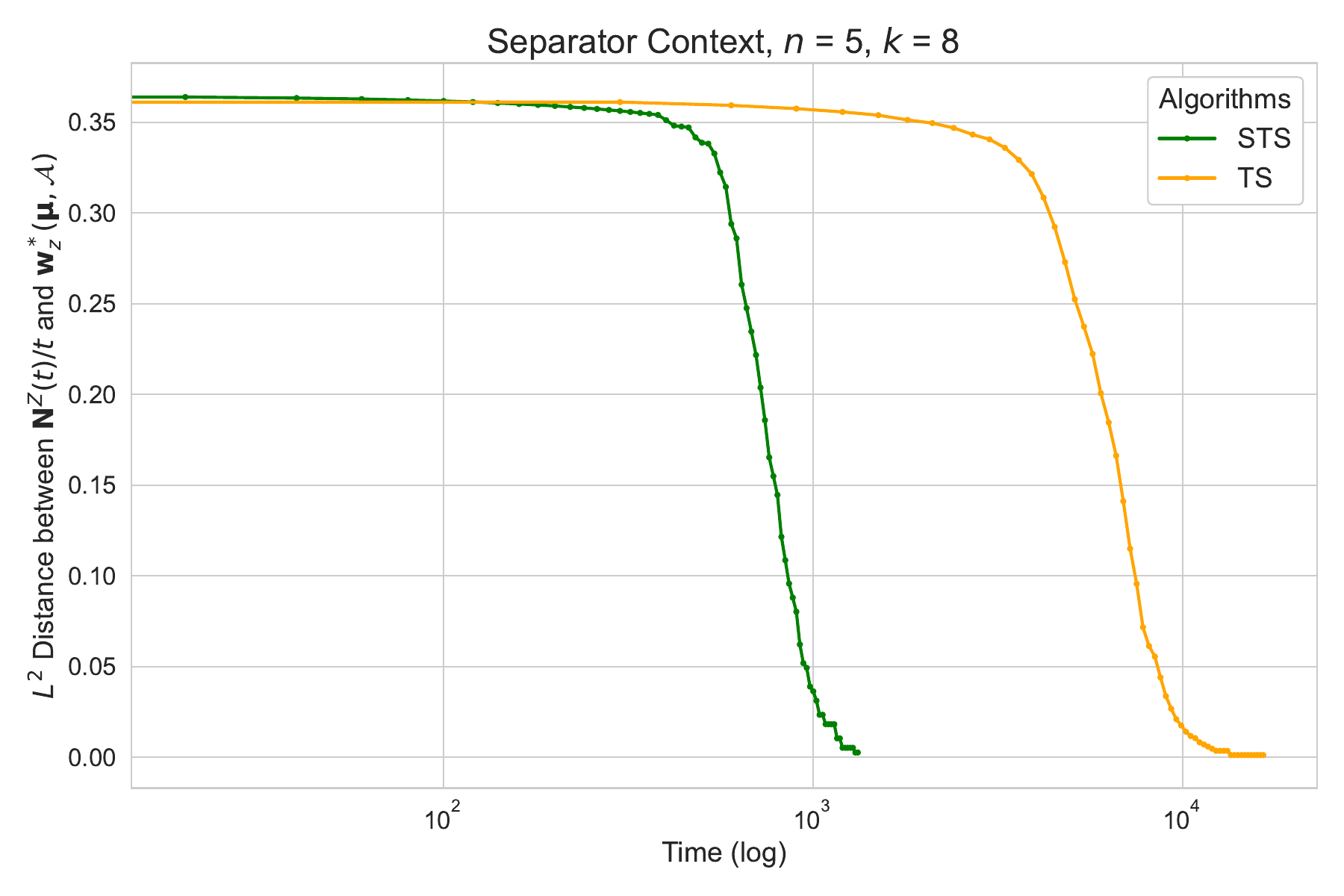} &

        \includegraphics[width=0.35\textwidth]{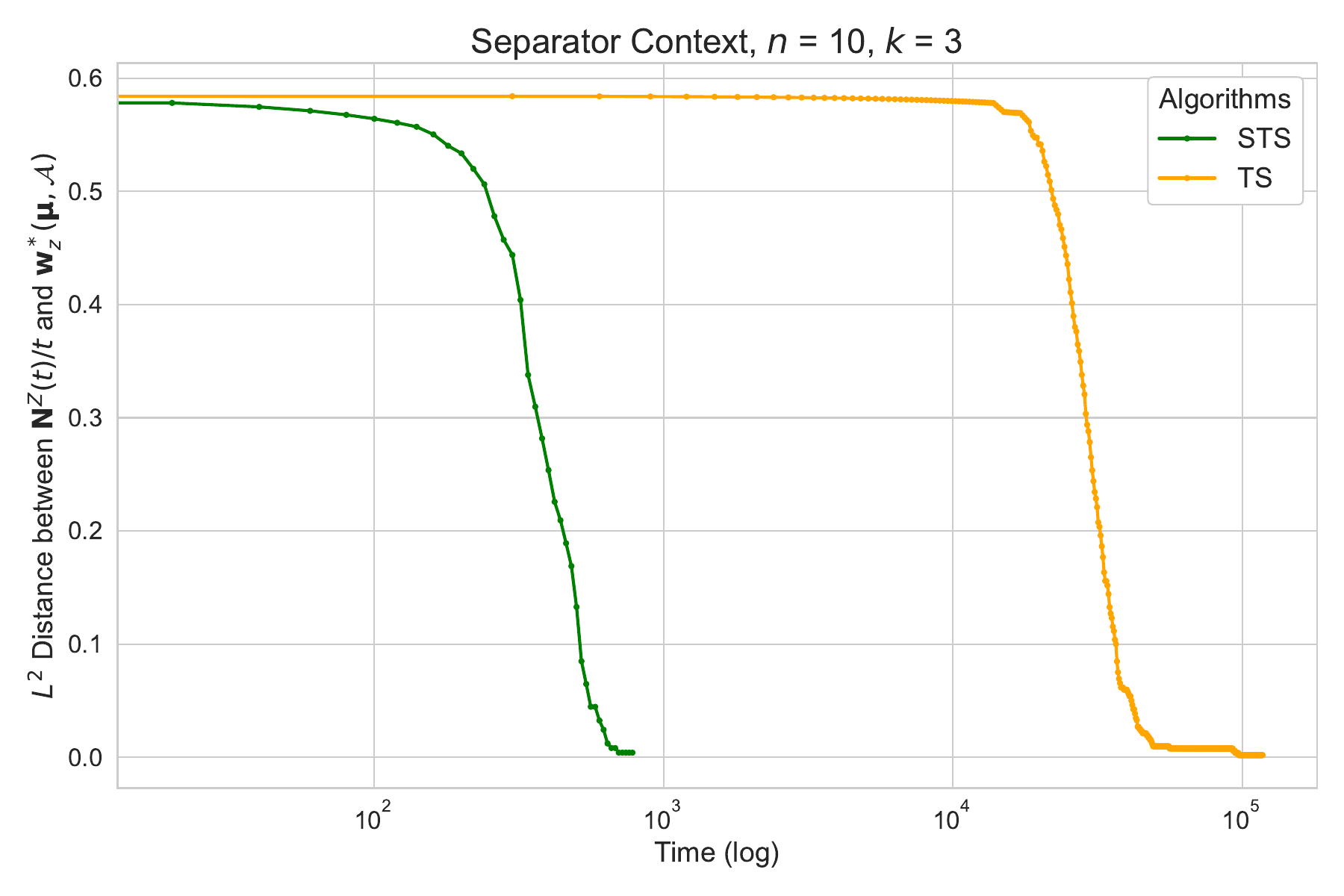} \\
        \includegraphics[width=0.35\textwidth]{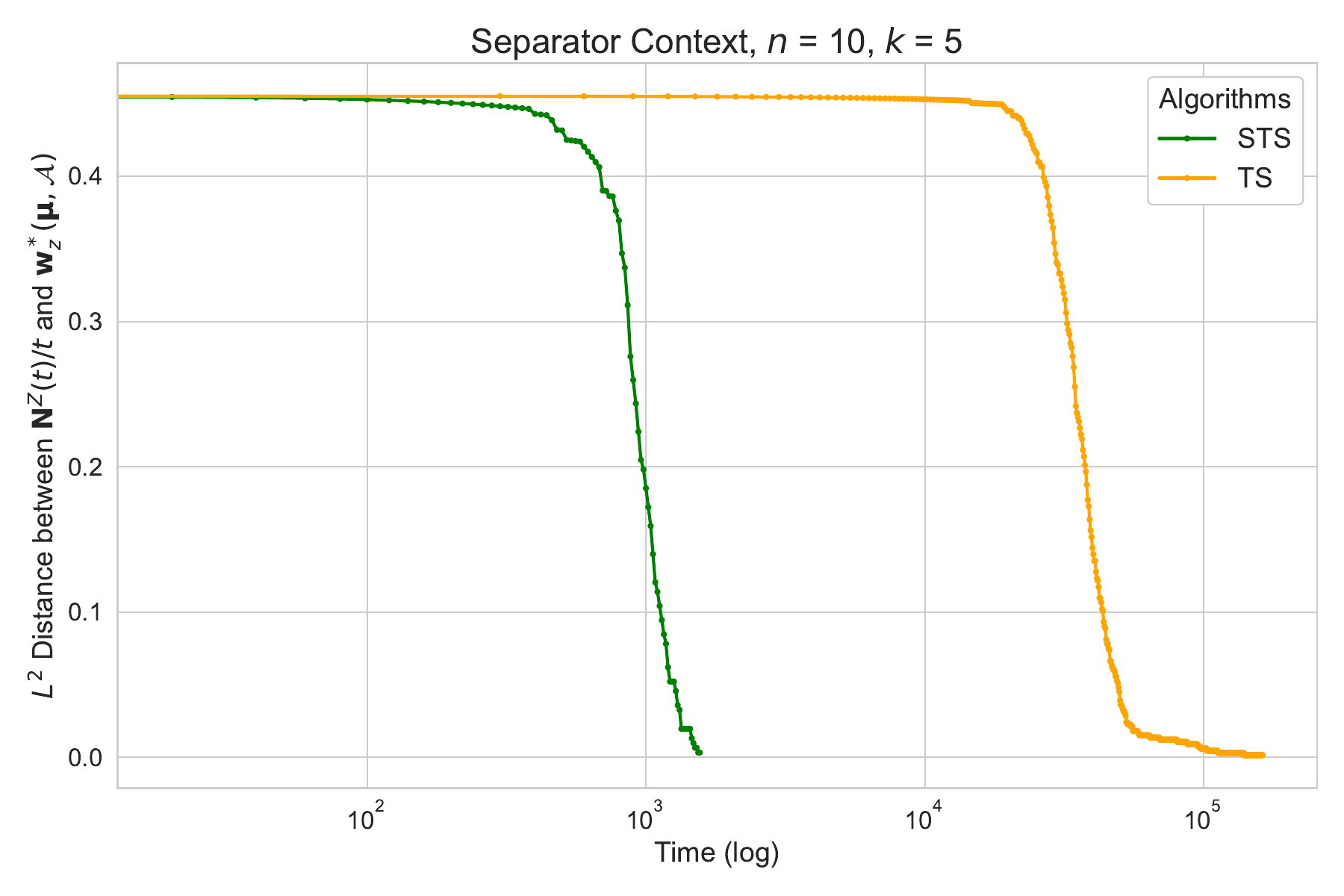} &
        \includegraphics[width=0.35\textwidth]{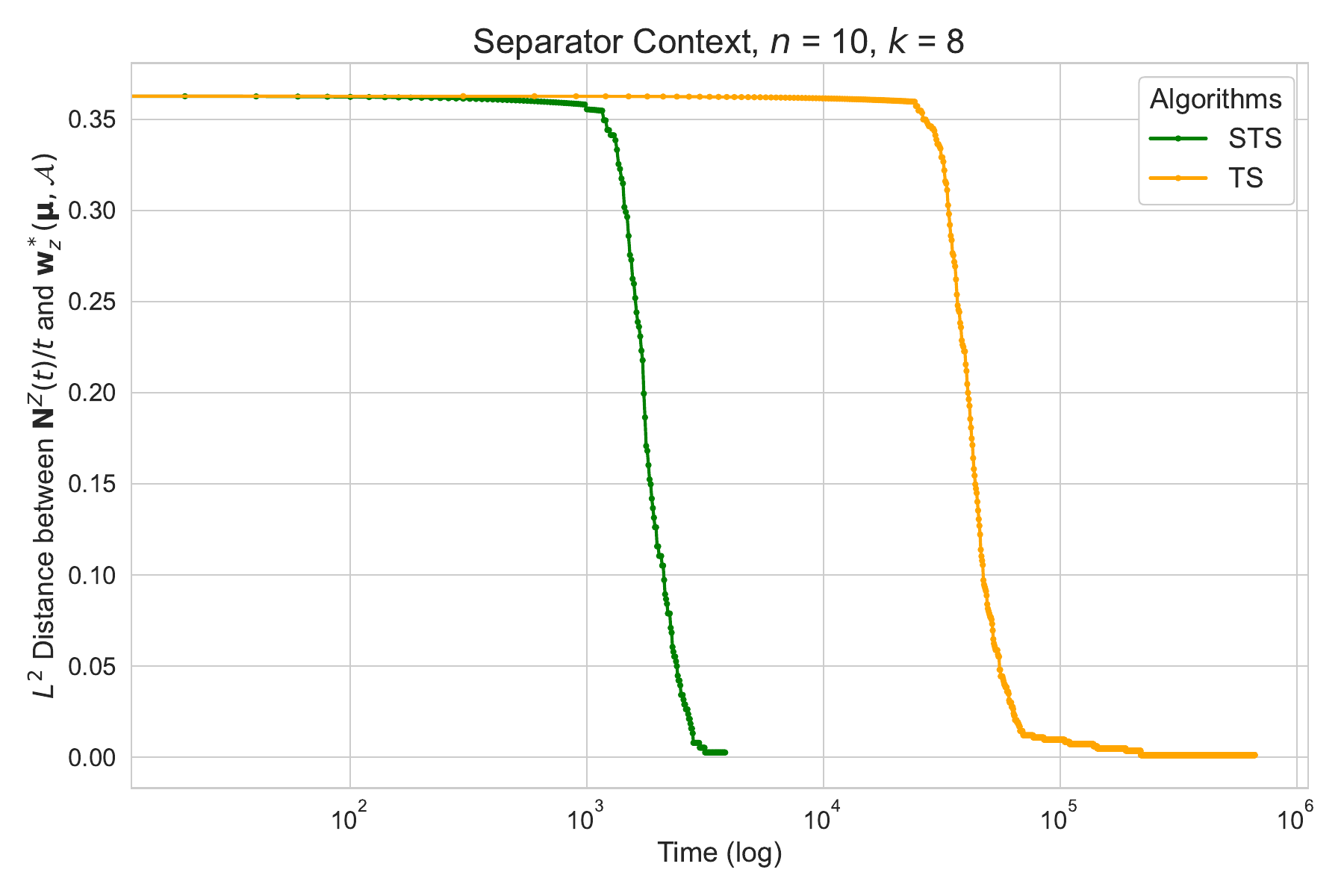} \\

        \includegraphics[width=0.35\textwidth]{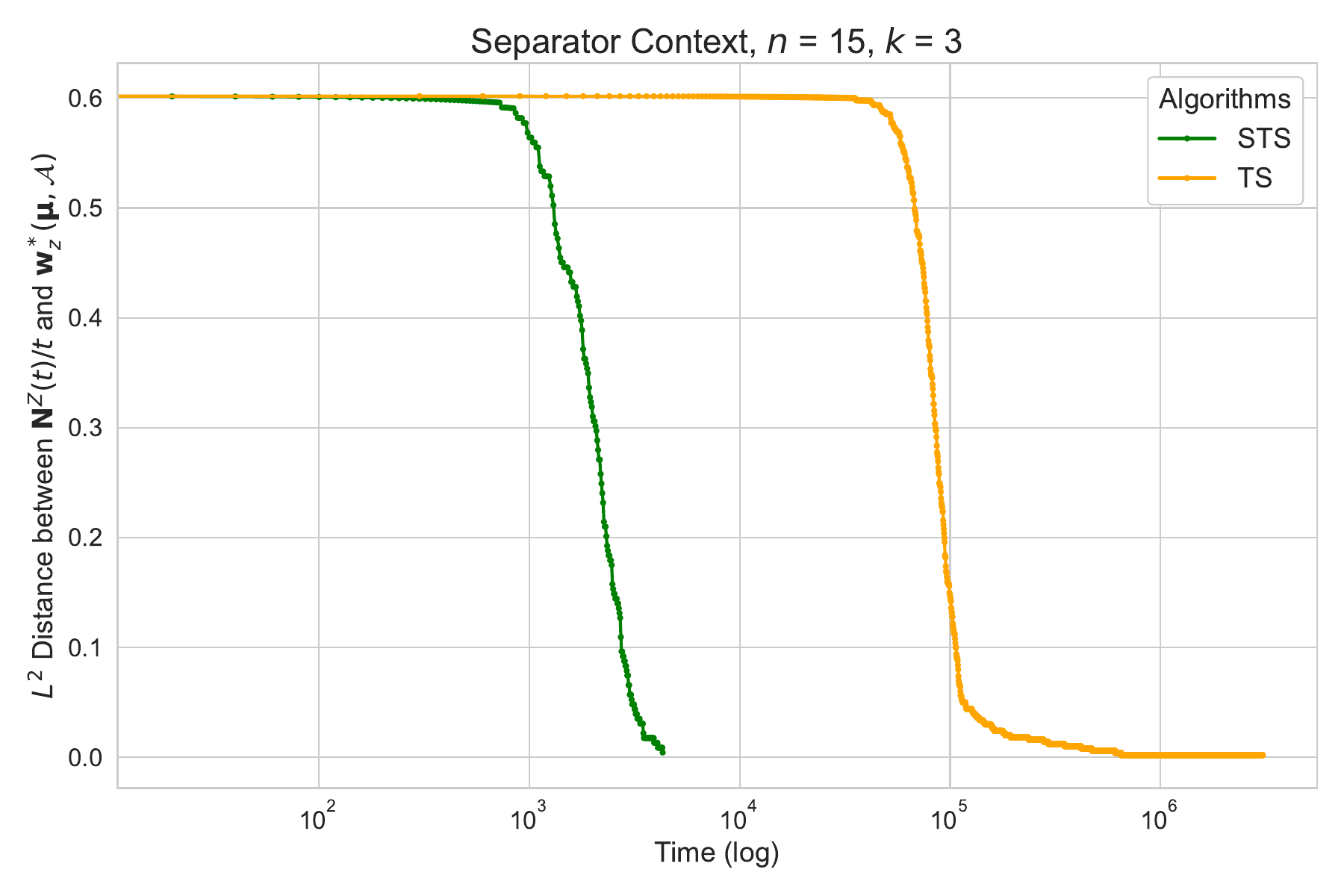} &
        \includegraphics[width=0.35\textwidth]{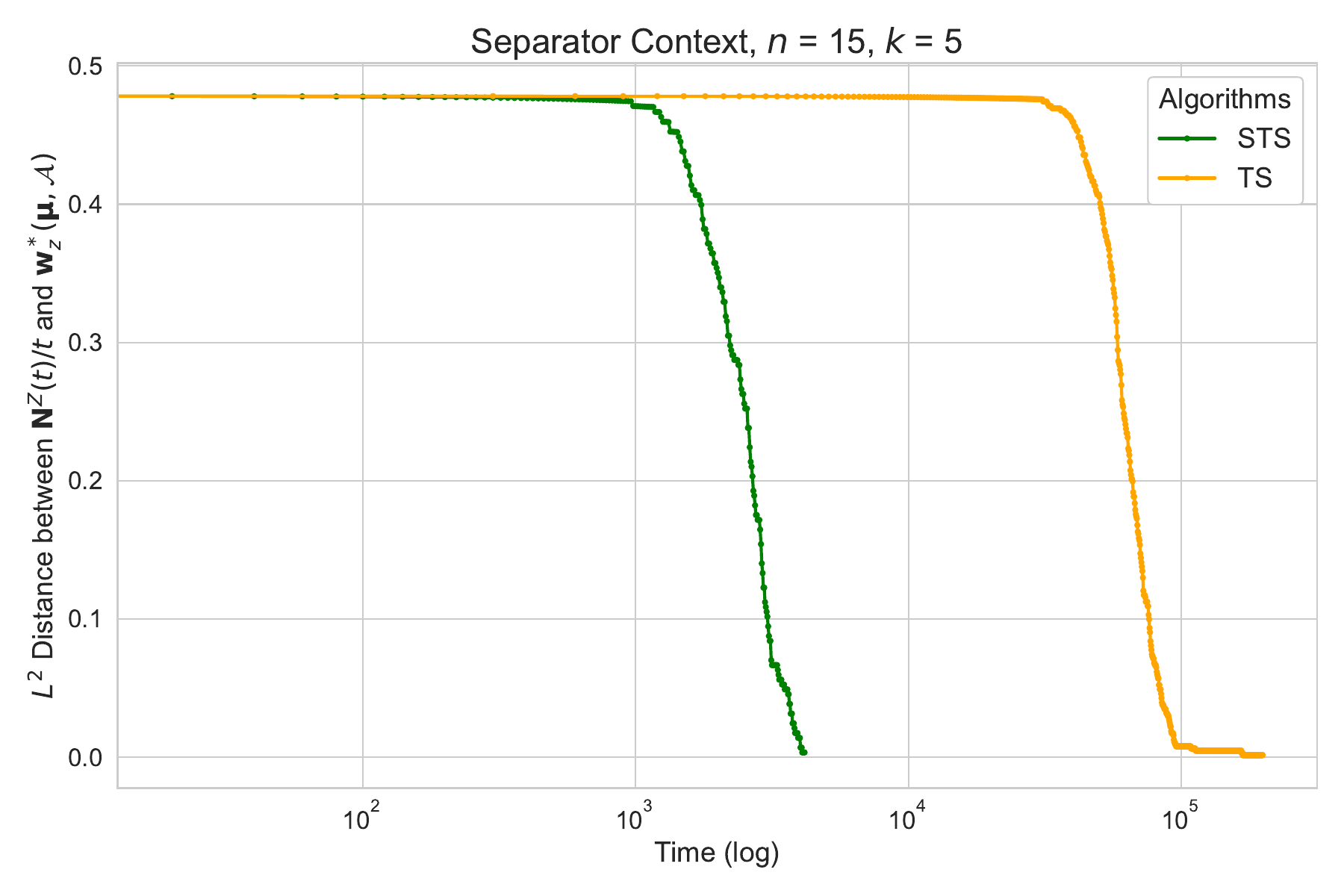} \\
        \includegraphics[width=0.35\textwidth]{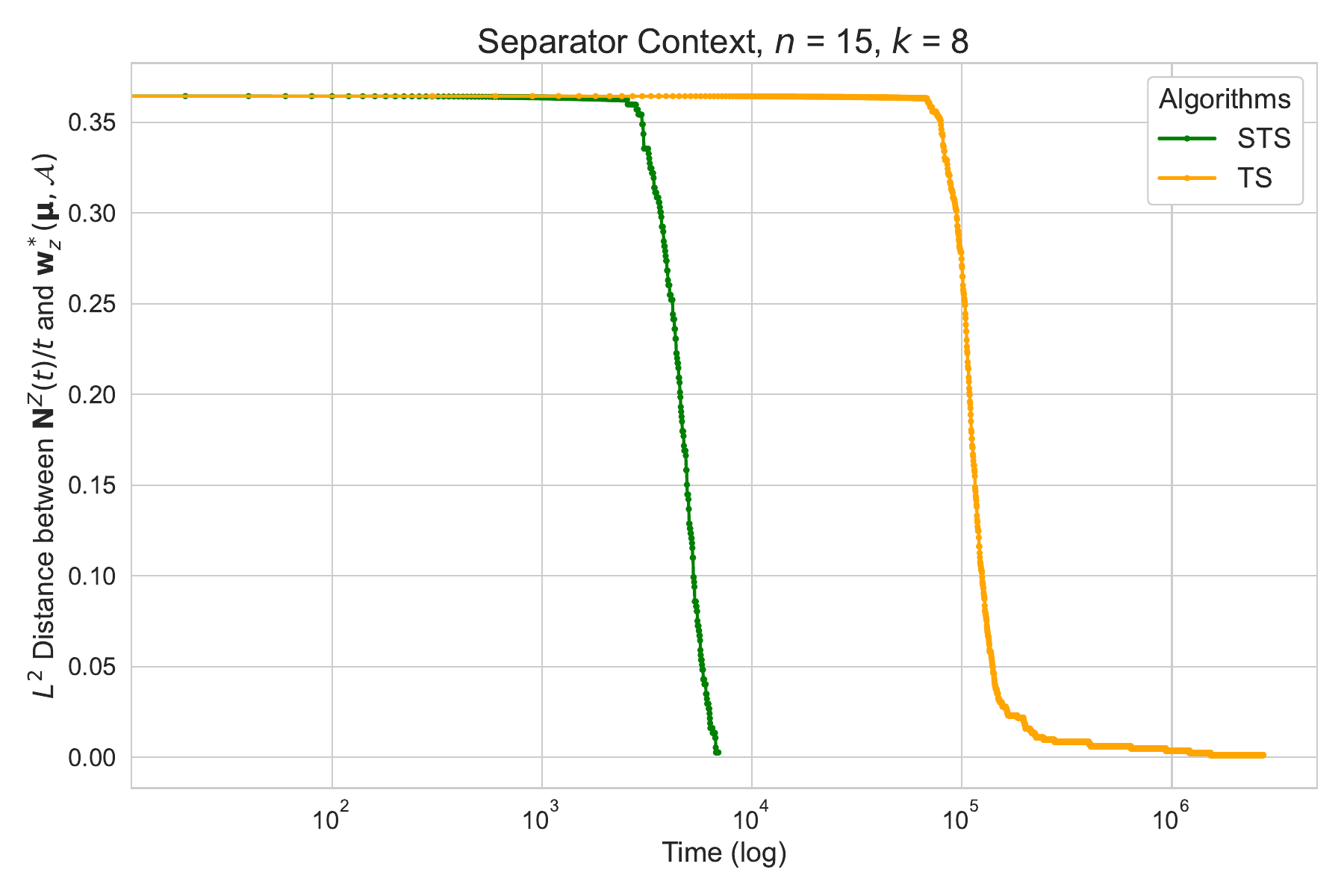} 
    \end{tabular}
    \caption{Comparison of the $L^2$ distance of the frequencies of observed contexts and the optimal frequency over time among different algorithms.}
    \label{fig: sep-dist}
\end{figure}



\end{document}